%% file: main.tex
\documentclass{article}

\PassOptionsToPackage{authoryear,square}{natbib}

 \usepackage[main, final]{neurips_2026}

\usepackage[utf8]{inputenc}
\usepackage[T1]{fontenc}
\usepackage{wrapfig}
\input{preambles}

\usepackage[colorlinks=true,linkcolor=blue,citecolor=blue,urlcolor=blue]{hyperref}
\usepackage[nameinlink,noabbrev]{cleveref}

\title{Asymptotically Optimal Best Arm Identification with Fixed-Budget under Differential Privacy}
\author{%
  Keqin Chen\\
  Institute of Data Science\\
  National University of Singapore\\
  \texttt{keqin.chen@u.nus.edu} \\
\And
  Jie Bian\thanks{Corresponding authors.} \\
  Department of Electrical and Computer Engineering \\
  National University of Singapore \\
  \texttt{jiebian@u.nus.edu}\\
\And
  Yulian Wu\footnotemark[1] \\
  Department of Mathematics \\
  National University of Singapore \\
 \texttt{yulian.wu@nus.edu.sg} \\
\And
  Vincent Y.~F.~Tan \\
  Department of Mathematics \\
  Department of Electrical and Computer Engineering \\
  National University of Singapore \\
  \texttt{vtan@nus.edu.sg} \\
}

\begin{document}

\maketitle

\begin{abstract}
\input{Sections/abstract}
\end{abstract}

\input{Sections/introduction}
\input{Sections/related_work}
\input{Sections/problem_setting}

\input{Sections/lower_bound}
\input{Sections/methodology}
\input{Sections/upper_bound}
\input{Sections/experiment}
\input{Sections/conclusion}

\newpage
\begin{ack}
Y.~Wu and V.~Y.~F.~Tan  are supported by a Singapore Ministry of Education AcRF Tier~2 grant (A-8004062-00-00). V.~Y.~F.~Tan is also supported by   a Singapore National Research Foundation Investigatorship (A-8004657-00-00). We thank the anonymous reviewers and area chair for their constructive feedback and thoughtful suggestions, which helped improve this paper.
\end{ack}

\nocite{arspy_docs,gilks1992ars,unuran_manual}
\bibliographystyle{plainnat}
\bibliography{reference}

\appendix

\input{Appendix/table_notations}

\input{Appendix/more_related_work}
\input{Appendix/proof_of_lipschitz}

\input{Appendix/lower_bound_proof}
\input{Appendix/laplace_tree}

\input{Appendix/privacy_proof}
\input{Appendix/upper_bound_proof}
\input{Appendix/comparison}

\input{Appendix/good_events_proof}
\input{Appendix/discussion_sampling}
\input{Appendix/additional_experiments}

\input{Appendix/future_work}

\newpage
\input{checklist.tex}

\end{document}

%% file: preambles.tex
\usepackage{amsfonts,amsmath,amsthm,amssymb}
\usepackage{graphicx}
\usepackage{algorithm}
\usepackage{algpseudocode}
\usepackage{enumerate}
\usepackage{xcolor}
\usepackage{booktabs}
\usepackage{nicefrac}
\usepackage{microtype}
\usepackage{url}
\usepackage{enumitem}

\DeclareMathOperator*{\argmax}{arg\,max}

\newtheorem{theorem}{Theorem}[section]

\newtheorem{lemma}[theorem]{Lemma}

\newtheorem{corollary}[theorem]{Corollary}
\theoremstyle{definition}
\newtheorem{definition}[theorem]{Definition}

\theoremstyle{remark}
\newtheorem{remark}[theorem]{Remark}
\theoremstyle{definition}
\newtheorem{sublemma}{Lemma}[theorem]

\newcommand{\TV}{\mathrm{TV}}
\newcommand{\KL}{\mathrm{KL}}
\newcommand{\kl}{\mathrm{kl}}
\newcommand{\Alt}{\mathrm{Alt}}
\newcommand{\rmd}{\mathrm{d}}
\newcommand{\coloneqq}{:=}

%% file: Sections/abstract.tex
Best arm identification under differential privacy is a pure-exploration problem in which both statistical efficiency and privacy protection must be achieved simultaneously. We study fixed-budget best arm identification for bandits under pure $\epsilon$-differential privacy, where the learner must recommend an arm after a prescribed sampling budget while protecting the full transcript. We prove that the optimal exponential decay rate of the error probability is upper bounded by an instance-dependent privacy-aware transportation exponent that differs from the analogous quantity used to characterize the stopping time in fixed-confidence analysis by Jourdan and Azize [2025]. 
Guided by this exponent, we propose \textsc{AO-Pri-BAI}, an adaptive algorithm that maintains private running estimates through Laplace-tree mechanisms and learns a sampling design through a min--max interaction between hard alternatives and arm allocations. We prove that \textsc{AO-Pri-BAI} satisfies pure \(\epsilon\)-differential privacy. We also establish that the exponent of the failure probability of \textsc{AO-Pri-BAI} matches the privacy-aware benchmark. Numerical studies show that even in the non-asymptotic setting, \textsc{AO-Pri-BAI} outperforms benchmark algorithms on various instances, complementing the theoretical analyses. 

%% file: Sections/introduction.tex
\section{Introduction}\label{sec:introduction}
Multi-armed bandits form a standard model for sequential decision-making under uncertainty, with early motivation from adaptive experimentation and clinical trials~\citep{thompson1933likelihood,lattimore2020banditalgorithms,orabona2019modern}.  In the pure-exploration setting, the learner is evaluated not by the cumulative reward collected during the experiment, but by the quality of the final recommendation.  Best-arm identification (BAI) is the canonical example: the learner adaptively samples arms and ultimately recommends the arm with the largest mean reward~\citep{evendar2006actionelimination,audibert2010bestarm}. This problem arises naturally in applications such as clinical trials, online experiments, recommendation systems, and user studies, where the final decision must be statistically reliable.

BAI is commonly studied under two performance criteria.  In the fixed-confidence regime, the learner targets a prescribed risk level and seeks to minimize sample complexity.  In the fixed-budget regime, which is our focus, the sampling horizon is fixed in advance, and performance is measured by the probability of recommending a suboptimal arm.  The two regimes have different information-theoretic behavior~\citep{kaufmann2016complexitybai}; in fixed-budget BAI, the central object is the exponential decay rate of the error probability as the budget grows~\citep{carpentier2016fixedbudgetlower}.

In many applications of adaptive experimentation, the observations collected during the experiment are sensitive.  Differential privacy (DP) provides a formal stability guarantee limiting the influence of any individual observation~\citep{dwork2014foundationsdp}.  Privacy is especially delicate in bandit problems: even if raw rewards are never released, the sequence of chosen arms can reveal information because future actions depend on past observations.  This transcript-level issue has been emphasized in private bandits and private BAI~\citep{shariff2018differentiallyprivatecontextuallinear,azize2022dpbandits,chen2024fixed,jourdan2025optimalbestarmidentification}.  We therefore require pure \(\epsilon\)-DP for the full adaptive transcript, including both the sampling trajectory and the final recommendation.

This paper studies fixed-budget BAI under pure \(\epsilon\)-DP.  Our starting point is that privacy changes the information geometry of the identification problem.  In non-private BAI, instance hardness is often expressed through a max--min comparison between a sampling design and the most confusing alternative instance~\citep{kaufmann2016complexitybai,garivier2016optimalbai,li2023optimalexplorationharderthompson}.  Under pure DP, distinguishability is constrained not only by stochastic uncertainty but also by the stability imposed by privacy.  Recent fixed-confidence work shows that this leads to privacy-aware transportation costs involving both KL divergence and total variation scaled by the privacy parameter~\citep{azize2023dpbaifixedconfidence,azize2024dpbai,jourdan2025optimalbestarmidentification}. Our fixed-budget exponent is related but not identical: the lower bound evaluates the transportation cost from the alternative instance to the true instance, reversing the true-to-alternative order used in the fixed-confidence characteristic time. Since the KL component is asymmetric, the resulting exponent is a distinct instance-dependent quantity.

Algorithmically, we propose \textsc{AO-Pri-BAI}, an adaptive procedure designed around this privacy-aware exponent.  The algorithm maintains private reward estimates using Laplace-tree continual release~\citep{dwork2010panprivate}, generates hard private alternatives from the current private estimate, and updates a sampling design using AdaHedge \citep{rooij2014adahedge}. The interaction can be viewed as a private min-max procedure: a min-learner searches for alternatives that remain hard to distinguish, while a max-learner allocates samples to improve distinguishability against those alternatives.

\paragraph{Contributions.}
Our contributions are as follows.  First, we formulate a fixed-budget BAI problem under pure \(\epsilon\)-DP and introduce a privacy-aware transportation exponent whose argument order differs from the fixed-confidence counterpart.  Second, we prove an information-theoretic lower bound showing that this exponent upper is the best possible asymptotic decay rate achievable by any consistent private policy with a limiting allocation.  Third, we propose \textsc{AO-Pri-BAI}, an adaptive algorithm that combines continual-release private statistics, privacy-aware challenger generation, and an AdaHedge sampling-design learner.  Fourth, we prove that \textsc{AO-Pri-BAI} is pure \(\epsilon\)-DP for the full adaptive transcript.  Finally, we establish an asymptotic error probability upper bound with exponent matching the optimal exponent proved in the lower bound, and empirically compare with a fixed-budget adaptation of DP-TT in \citet{jourdan2025optimalbestarmidentification}, and a uniform private sampling baseline.

%% file: Sections/related_work.tex
\section{Related Work}\label{sec:related_work}

This section briefly reviews the literature that are needed to position the paper; an extended discussion is deferred to Appendix~\ref{apx:more_related_work}.  The main point is that our problem combines three requirements that are usually treated separately: fixed-budget pure exploration, transcript-level differential privacy, and instance-dependent allocation against confusing alternatives.

\paragraph{Fixed-budget error exponents.}
Fixed-confidence and fixed-budget BAI should be viewed as distinct statistical regimes rather than two versions of the same objective.  In fixed confidence, the central quantity is sample complexity at a prescribed risk level; in fixed budget, the central quantity is the decay rate of the final error probability.  The complexity comparison of Kaufmann, Capp\'e, and Garivier~\citep{kaufmann2016complexitybai} motivates treating the fixed-budget exponent as a first-class object, while classical fixed-budget algorithms such as Successive Rejects and Sequential Halving provide benchmark non-private procedures~\citep{audibert2010bestarm,karnin2013almostoptimal}.  Our work asks for the corresponding exponent when the entire adaptive transcript must be private.

\paragraph{Privacy in sequential learning.}
Differential privacy provides a stability guarantee for algorithms on neighbouring datasets \citep{dwork2014foundationsdp}. In bandits, privacy is more subtle because the action sequence itself can leak information about prior rewards. Basu, Dimitrakakis and Tossou \citep{basu2019privacybandits} give an early treatment of privacy notions for multi-armed bandits and relate these notions through regret lower bounds. Subsequent regret-minimization work under global DP shows that privacy changes the effective hardness of stochastic and linear bandits, leading to privacy regimes and adaptive private statistics \citep{azize2022dpbandits, azize2025optimalregretbernoullibandits}. These works are regret-focused, but they motivate two ingredients used here: transcript-level privacy and private running estimates. We implement the latter using Laplace-tree continual release \citep{dwork2010panprivate, shariff2018differentiallyprivatecontextuallinear}. The closest work to ours is \citet{chen2024fixed}, which tackles fixed-budget private BAI of linear bandits. However, we differ from them fundamentally in the analysis techniques and theoretical guarantees (see discussions after Theorem \ref{thm:UpperBound}).

\paragraph{Private transportation geometry.}
Modern BAI analyses often express hardness through a competition between a sampling allocation and the most confusing alternative instance~\citep{garivier2016optimalbai,li2023optimalexplorationharderthompson,bian2025asymptotically}.  In the private setting, the cost of confusing two instances is no longer captured by the classical KL geometry alone.  Fixed-confidence DP-BAI results identify privacy-aware transportation costs that interpolate between KL divergence and total variation scaled by the privacy budget~\citep{azize2023dpbaifixedconfidence,azize2024dpbai,jourdan2025optimalbestarmidentification}.  Our contribution is to build a fixed-budget error analysis around this private transportation geometry and to design an adaptive allocation rule targeted at the resulting exponent.

%% file: Sections/problem_setting.tex
\section{Problem Setting}\label{sec:problem_setting}
We consider a fixed-budget Bernoulli best-arm identification problem with \(K\) arms. The true instance is specified by a mean vector
\[
\theta^{\ast}=(\mu_1,\ldots,\mu_K)\in\Theta,
\quad\mbox{where}\quad
\Theta:=\{\theta\in\mathbb R^K: \theta_i\in(0,1)\text{ for all }i\in[K]\}.
\]
The arms are indexed by \([K]=\{1,\ldots,K\}\). When arm \(a\) is pulled, the learner observes an independent reward drawn from \(\mathrm{Ber}(\mu_a)\). We assume that the best arm is unique and denote it by
\[
 a_{\theta^{\ast}}:=\argmax_{a\in[K]}\mu_a .
\]
The learner is given a fixed sampling budget \(T\). A policy \(\pi\) consists of a sequence of sampling rules $\{\lambda_t\}_{t\in[T]}$, which selects the arm \(A_t\) at each time \(t\in[T]\) based on the previous interaction history, and a recommendation rule $I$, which outputs a final arm \(A_{\mathrm{out}}\in[K]\) after \(T\) rounds.

\paragraph{Consistency.}
The interaction between an instance \(\theta^\ast\) and a policy \(\pi\) induces a probability measure on the transcript
$(A_1,\ldots,A_T,A_{\mathrm{out}})\in[K]^{T+1}$.  
We denote this measure by \(\mathbb M_{\theta^\ast,\pi}\). The error probability of \(\pi\) on \(\theta^\ast\) is
$ 
\mathbb M_{\theta^\ast,\pi}(A_{\mathrm{out}}\neq a_{\theta^\ast}).
$

\begin{definition}[Consistency]
A policy \(\pi\) is consistent if, for every \(\theta^\ast\in\Theta\),
\[
\lim_{T\to\infty}\mathbb M_{\theta^\ast,\pi}(A_{\mathrm{out}}\neq a_{\theta^\ast})=0.
\]
\end{definition}

\paragraph{Differential Privacy.}
Let $R \subseteq \{0,1\}^{K \times T}$ denote the set of reward matrices 
$r = \{r_{i,t}\}_{i \in [K],\, t \in [T] }$, where $r_{i,t}$ represents the   reward obtained the $t$-th time arm $i$ is played. A policy 
$\pi$ can be viewed as a map from $R$ to $[K]^{T+1}$, 
producing a sequence of actions together with a final recommendation.

Concretely, the interaction proceeds recursively as follows~\citep{chen2024fixed}:
\begin{align*}
    &\text{Arm at } t=1:\quad A_1 = \lambda_1 \quad (\text{no history})\\
    &\text{Arm at } t=2:\quad A_2 = \lambda_2\bigl(A_1,\, r_{A_1, N_{A_1,1}}\bigr)\\
    &\;\;\vdots\\
    &\text{Arm at } t=T:\quad A_T = \lambda_T\bigl(A_1, r_{A_1, N_{A_1,T-1}}, \dots, A_{T-1}, r_{A_{T-1}, N_{A_{T-1},T-1}}\bigr)\\
    &\text{Final recommendation:}\quad 
    A_{\mathrm{out}} = I\bigl(A_1, r_{A_1, N_{A_1,T}}, \dots, A_T, r_{A_T, N_{A_T,T}}\bigr)
\end{align*}
where $N_{a,t}$ denotes the number of times arm $a$ has been played up to 
(and including) time $t$, $\lambda_{t}$ is a function of $\{ (A_{s}, X_{s})\}_{s<t}$, where $X_{s}$ is a sample from Bernoulli distribution $\mathrm{Ber}(\mu_{A_{s}})$, and $I$ is a function of $\{(A_{t},X_{t})\}_{t\in[T]}$, which makes the final recommendation using all observed history.

\begin{definition}[Differential Privacy for Best Arm Identification \citep{chen2024fixed}]
    Consider two such reward tables $r,r'$. We say that they are neighboring if there exists at most one pair $(i,t)\in[K]\times[T]$ such that $r_{i,t} \neq r'_{i,t}$, while all other entries are equal. We say that $\pi$ is $\epsilon$-DP if for all neighboring $r,r'$, for all instances $\theta^{*}$, and for all $S\subseteq[K]^{T+1}$, we have
\[\mathbb{M}_{\theta^{*},\pi}(\pi(r)\in S) \leq e^{\epsilon}\mathbb{M}_{\theta^{*},\pi}(\pi(r')\in S).\]
\end{definition}

%% file: Sections/lower_bound.tex
\section{Lower Bound}\label{sec:lower_bound}
We first introduce the privacy-aware divergence that will govern the fixed-budget exponent. Inspired by the transportation-cost formulation of private Bernoulli BAI~\citep{jourdan2025optimalbestarmidentification}, for two instances \(\theta_1=(\mu_{1,1},\ldots,\mu_{1,K})\) and \(\theta_2=(\mu_{2,1},\ldots,\mu_{2,K})\). Let $\Delta_{[K]}$ denote the probability simplex over $[K]$. For an allocation \(\lambda\in\Delta_{[K]}\), define
\[
d_\epsilon(\theta_1,\theta_2,\lambda)
:=
\inf_{\theta'\in\Theta}
\mathbb E_{a\sim\lambda}
\!\big[
\epsilon\TV(\mathrm{Ber}(\mu_{1,a})\|\mathrm{Ber}(\mu'_a))
+
\KL(\mathrm{Ber}(\mu'_a)\|\mathrm{Ber}(\mu_{2,a}))
\big],
\]
where \(\theta'=(\mu'_1,\ldots,
\mu'_K)\) is an intermediate instance.

There are two properties of $d_{\epsilon}$, with formal statements and proofs deferred to Appendix \ref{apx:proof_lipschitz}, which support subsequent theoretical analysis and algorithmic implementation:
\begin{itemize}[leftmargin=*]
    \item $d_{\epsilon}(\theta_{1},\theta_{2},\lambda)$ admits a closed-form solution.
    \item $d_{\epsilon}(\theta_{1},\theta_{2},\lambda)$ is $\epsilon$-Lipschitz with respect to $\lambda$. That is, for any $\theta_1,\theta_2\in\Theta$ and $\lambda_{1},\lambda_{2}\in\Delta_{[K]}$, we have  $|d_{\epsilon}(\theta_{1},\theta_{2},\lambda_{1})-d_{\epsilon}(\theta_{1},\theta_{2},\lambda_{2})| \leq \epsilon \| \lambda_{1}-\lambda_{2}\|_{1}$.
\end{itemize}

We are now ready to state the lower bound. It shows that the best achievable fixed-budget exponent is upper bounded by a privacy-aware max-min value $\tau^{*}(\theta^{*})$. The proof can be found in Appendix~\ref{apx:lowerbound}.
 
\begin{theorem}[Lower Bound]\label{thm:LowerBound}
Suppose a policy \(\pi\) is consistent and \(\epsilon\)-DP. Assume further that there exists an allocation \(\lambda^*=(\lambda^*_1,\ldots,\lambda^*_K)\in\Delta_{[K]}\) such that, for every problem instance \(\theta\in\Theta\), the arm counts generated by \(\pi\) satisfy
$ 
\lim_{T\to\infty}\frac{N_{a,T}}{T}=\lambda_a^*
$  a.s.\ for all $a\in[K]$.
Then
\[
\limsup_{T\to\infty}
-\frac1T\log\mathbb M_{\theta^*,\pi}(A_{\mathrm{out}}\neq a_{\theta^*})
\le \tau^*(\theta^*):=
\sup_{\lambda\in\Delta_{[K]}}
\inf_{\theta'\in\Alt(\theta^*)}
 d_\epsilon(\theta',\theta^*,\lambda),
\]
and
$ 
\Alt(\theta^*)=
\{\theta=(\mu_1,\ldots,\mu_K)\in\Theta:
 a_{\theta^*}\notin\argmax_{a\in[K]}\mu_a\}.
$
\end{theorem}

\begin{remark}
The lower bound in \citet{li2023optimalexplorationharderthompson} (for the no privacy scenario) imposes a more restrictive assumption. To wit, in their Theorem 2.1, they assume that algorithms are non-adaptive, in the sense that an algorithm always keeps a fixed sampling design $\lambda$ independent of problem instance $\theta^{*}$ it is trying to solve. Our Theorem \ref{thm:LowerBound}  relaxes this condition by allowing the algorithm to be adaptive in sampling with an instance-independent limiting design.
\end{remark}

\paragraph{Discussions of and Novelties in Theorem \ref{thm:LowerBound}}
This divergence measures the shortest two-step path from \(\theta_1\) to \(\theta_2\): the first step is charged in total variation, scaled by the privacy budget, while the second step is charged in KL divergence. If we relax the privacy constraint by letting $\epsilon\rightarrow\infty$, then $d_{\epsilon}(\theta_{1},\theta_{2},\lambda)$ reduces to $\mathbb{E}_{a\sim\lambda}[\KL(\mathrm{Ber}(\mu_{1,a})\|\mathrm{Ber(\mu_{2,a})})]$. We have then recovered the non-private transportation cost that is purely $\KL$-based geometry; cf.\  Theorem 2.1 of \citet{li2023optimalexplorationharderthompson}.

We  emphasize that although the definition of \(\tau^*(\theta^*)\) is structurally inspired by the characteristic time in Theorem~2 of \citet{jourdan2025optimalbestarmidentification}---both involve a supremum over designs \(\lambda\in\Delta_{[K]}\) and an infimum over alternatives \(\theta'\in\Alt(\theta^*)\)---our quantity is mathematically distinct. In particular, the characteristic time in \citet{jourdan2025optimalbestarmidentification} is based on $d_\epsilon(\theta^*,\theta',\lambda),$ whereas our exponent is governed by $d_\epsilon(\theta',\theta^*,\lambda),$ i.e., the order of the first two arguments is reversed. This distinction is fundamental rather than cosmetic: since KL divergence (for Bernoullis) is asymmetric, reversing the arguments changes the optimization geometry and leads to a genuinely different characterization. Consequently, our result is not a straightforward adaptation of the fixed-confidence setting, but instead identifies a new privacy-aware transportation quantity tailored to fixed-budget BAI.

At the technical level, the proof of Theorem~\ref{thm:LowerBound} develops a new coupled-environment argument specialized to the fixed-budget regime. While the construction is partly motivated by the fixed-confidence analysis of \citet{jourdan2025optimalbestarmidentification}, several key ingredients are fundamentally different. In fixed-confidence BAI, the transportation cost quantifies the difficulty of moving from the true instance toward a confusing alternative. In contrast, our fixed-budget lower bound controls the error probability under the true instance by relating it to a wrong-best-arm alternative, which necessarily reverses the direction of the change-of-measure argument. As a consequence, the maximal coupling is constructed between the alternative instance and an intermediate instance, whereas the likelihood-ratio comparison is performed between the intermediate and true instances. This asymmetric construction naturally generates the total-variation and KL terms in the alternative-to-true order appearing in the definition of \(\tau^*(\theta^*)\). To the best of our knowledge, this alternative-to-true transportation geometry and its associated coupling argument are new in the literature on private best-arm identification.

%% file: Sections/methodology.tex
\section{\textsc{AO-Pri-BAI}: A Methodology Matching the Lower Bound}\label{sec:methodology}

In Section \ref{sec:lower_bound}, we established an instance-dependent lower bound on the optimal error exponent for fixed-budget private best-arm identification. We now turn to the algorithmic question: can this exponent be achieved by a computationally tractable private policy? To answer this question, we now introduce \textsc{AO-Pri-BAI} (\underline{A}symptotically \underline{O}ptimal \underline{Pri}vate \underline{B}est \underline{A}rm \underline{I}dentification), a fixed-budget algorithm for Bernoulli best-arm identification under an $\epsilon$-differential privacy constraint. The algorithm pseudocode is stated as Algorithm \ref{algo:FB-DP-BAI}. It is designed around the instance-dependent exponent $\tau^\ast(\theta^{*})$ introduced in Theorem~\ref{thm:LowerBound}. This exponent has a natural two-player game-theoretic interpretation, consisting of a max-player and a min-player. The max-player selects a sampling design over the arms, while the min-player selects an alternative instance with a different optimal arm that is hardest to distinguish from the true instance under that design. Thus, the value $\tau^\ast(\theta^{*})$ corresponds to the best worst-case distinguishability that can be achieved by an allocation rule. Since we do not know the true instance $\theta^{*}$, it rules out the possibility of directly solving the saddle-point optimization problem as in the definition of $\tau^{*}(\theta^{*})$. Our algorithm \textsc{AO-Pri-BAI} overcomes this difficulty by learning in a sampling-based online approach. 

The goal of \textsc{AO-Pri-BAI} is to learn  the optimal allocation achieving $\tau^{*}(\theta^{*})$ adaptively. At a high level, the algorithm maintains two interacting online components, the min-learner and the max-learner, which simulate the corresponding players. The min-learner generates a plausible alternative instance that challenges the current private estimate of the best arm. The max-learner then uses this alternative to generate a stochastic loss, which advises the update of sampling design, increasing the sampling mass on arms that are informative for separating the current instance from hard alternatives. This online interaction is intended to drive the empirical allocation toward the saddle-point design associated with $\tau^\ast(\theta^{*})$. We emphasize the simplicity of the online interaction process, as we do not need to solve any optimization problems. The process is realized only through sampling, plus an argmax oracle used for the final recommendation. 

\paragraph{Initialization.}
At the beginning of the experiment, the algorithm has no evidence favoring one arm over another. We therefore initialize the private mean vector at the uninformative value
$m_{1,0}=\cdots=m_{K,0}=\frac12$
and initialize the sampling design uniformly over the arms. The AdaHedge learner is also initialized in its standard form as in \citet{rooij2014adahedge}, with zero cumulative losses, zero cumulative mixability gap, and infinite initial learning rate.

\paragraph{Private running statistics.}
Since the sampling rule is adaptive, the sequence of selected arms may itself reveal information about the observed rewards. It is therefore not sufficient to privatize only the final recommendation; the entire transcript generated by the policy must be protected. To achieve this, \textsc{AO-Pri-BAI} accesses rewards only through private running estimates. For each arm, we maintain a Laplace tree over its local pull counts. Whenever a reward is observed, it is inserted into the next available leaf of the corresponding arm-wise tree. Internal nodes store subtree sums, and an independent Laplace perturbation is added to each completed node. We refer readers who are less familiar with tree-based algorithms to Appendix \ref{apx:trees} for pseudocode of how to initialize, maintain and update the Laplace trees, as well as how to retrieve data from them. This tree structure allows the algorithm to recover a private cumulative reward for any arm by summing only $O(\log T)$ noisy node values, rather than directly reusing all previously observed rewards. Because a single reward affects at most one node at each tree level, calibrated Laplace noise at the tree nodes yields an $\epsilon$-DP release of the private reward summaries. The private empirical mean $m_{a,t}$ is then computed from the noisy prefix sum of arm $a$ and is used as the only reward-dependent statistic in subsequent decisions. The privacy proof in Theorem~\ref{thm:DP} formalizes this post-processing argument.

\paragraph{Min-learner: sampling hard alternatives.}
At round $t$, the algorithm first identifies the current private empirical leader
$\hat a_t \in \arg\max_{a\in[K]} m_{a,t-1}$.
The min-learner then samples an alternative parameter vector from a Gibbs distribution $q_{t}$ supported on instances in which $\hat a_t$ is no longer optimal, defined by
\[q_{t}(\theta) \propto \mathbf{1}_{ \{ \theta \in \mathrm{Alt}(\hat{\theta}_{t}) \}}\exp\bigg(-\beta_{t}\sum_{a=1}^{K}\lambda_{t,a}d_{\epsilon}^{\mathrm{Ber}}(\theta_{a},m_{a,t-1})\bigg).\]
The density of this distribution is tilted by the current sampling design and by the arm-wise privacy-aware divergence $d_\epsilon^{\mathrm{Ber}}$. Consequently, alternatives that are difficult to distinguish from the current private estimate receive an (exponentially) larger probability mass. There is a hyperparameter $\beta_{t}$, for which we only require  $\beta_{t}>0$ and  $\beta_{t} \rightarrow \infty$ when $t\rightarrow\infty$. This step plays the role of generating a stochastic best response for the min-player.  Intuitively, alternatives with smaller transportation costs are sampled more frequently, so by our design, the min-learner naturally concentrates on the most informative challengers without explicitly solving an optimization problem at each iteration. The alternative instance is then used to produce a stochastic loss for the max-learner. For implementation, the restriction to the alternative set can be enforced by rejection sampling from the corresponding product-form distribution, due to its empirical simplicity. 

\paragraph{Max-learner: updating the sampling design.}
Given the sampled alternative instance, the max-learner updates the sampling design using AdaHedge. The loss assigned to arm $a$ is
\[\ell_{t,a}=1-\frac{1}{\epsilon}d_\epsilon^{\mathrm{Ber}}(\mu'_{a,t-1},m_{a,t-1}),\]
where $\mu'_{a,t-1}$ is the sampled alternative mean of arm $a$. Since $d_\epsilon^{\mathrm{Ber}}\in[0,\epsilon]$, the resulting loss lies in $[0,1]$, which is the natural scale for Hedge-type algorithms.

The AdaHedge update uses the cumulative losses to compute the next sampling design, while the learning rate is chosen adaptively from the cumulative mixability gap. This obviates specifying a fixed learning-rate schedule in advance. Intuitively, arms that repeatedly provide stronger evidence against hard alternatives accumulate smaller losses and receive larger sampling mass in future rounds.
\paragraph{Forced exploration.}
To ensure that every arm continues to receive samples, the design produced by AdaHedge is mixed with a fixed exploration distribution. Specifically, the played distribution is\[
\hat\lambda_t=(1-\gamma_t)\lambda_t+\gamma_t\lambda',\]
where $\lambda'$ is the uniform distribution and $\gamma_t=t^{-\alpha}$ for $\alpha\in(0,\frac{1}{2})$. The exploration component guarantees that all arm pull counts diverge, which is essential for private mean consistency and for eventual identification of the true leader. At the same time, because $\gamma_t\to0$, the exploration component vanishes asymptotically and does not affect the limiting exponent.

\paragraph{Final recommendation.}
After the budget is exhausted, the algorithm constructs a final sampling rule using the private mean estimates and the realized arm counts. A final parameter vector is sampled from another Gibbs distribution $q_{\mathrm{Final}}$ defined as
\[q_{\mathrm{Final}}(\theta)\propto\exp\!\left(-\sum_{a=1}^KN_{a,T}d_\epsilon^{\mathrm{Ber}}(\theta_a,m_{a,T})\right),\]
and the recommended arm is the arm with the largest sampled coordinate. This final randomized recommendation is consistent with the error analysis in Theorem~\ref{thm:UpperBound}: the probability of recommending a suboptimal arm is exactly the probability mass assigned to the alternative set.

\begin{algorithm}
\caption{\textsc{AO-Pri-BAI}$(\epsilon,T)$}\label{algo:FB-DP-BAI}
\begin{algorithmic}
    \Require Finite set of arms $\mathcal{X}=[K]$; Budget $T$; Distribution parameter $\beta_{t}$; Shrinking rate of exploration $\alpha \in (0,\frac{1}{2})$; Initial estimator $\hat{\theta}_{1}=(m_{1,0},\dots,m_{K,0})=(\frac{1}{2},\dots,\frac{1}{2})$ AdaHedge initial learning rate $\eta_{1}=\infty$; Initial cumulative loss $L^{\mathrm{Cumu}}_{0}=(L_{0,1},\dots,L_{0,K})=(0,0,\dots,0)$; Mixability gap $\Delta_{0}=0$; Privacy constraint parameter $\epsilon$

    \State Initialize: $\lambda'=\frac{1}{K}\mathbf{1}$; $\lambda_{1}=\frac{1}{K}\mathbf{1}$; Call \textsc{Init}($K,T,\epsilon$).

    \For{$t=1,2,\dots,T$}
        \State $\gamma_{t}=t^{-\alpha}$
        \State Update $\lambda_{t,a} \propto e^{-\eta_{t}L_{t-1,a}^{\mathrm{Cumu}}}$ for all $a\in[K]$
        \State Compute $\hat{a}_{t}=\arg\max_{a\in[K]} m_{a,t-1}$, break ties arbitrarily
        \State Sample $\theta'_{t}=(\mu'_{1,t-1},\dots,\mu'_{K,t-1}) \sim q_{t}$
        \State $\hat{\lambda}_{t}=(1-\gamma_{t})\lambda_{t}+\gamma_{t}\lambda'$
        \State Sample $A_{t}\sim\hat{\lambda}_{t}$ and observe reward $r_{t}\sim\mathrm{Ber}(\mu_{i})$, where $i$ is the index of $A_{t}$.
        
        \textsc{Update}($\{\mathcal{T}_{a}\}_{a=1}^{K}, \{N_{a,t-1}\}_{a=1}^{K}, A_{t}, r_{t}, t$)
        \State $\{\tilde{S}_{a,t}\}_{a=1}^{K}, \{\tilde{S}'_{a,t}\}_{a=1}^{K}, \{m_{a,t}\}_{a=1}^{K} \gets \textsc{ComputeMeans}(\{\mathcal{T}_{a}\}_{a=1}^{K}, t, \{N_{a,t}\}_{a=1}^{K})$
    
        \State Set $\hat{\theta}_{t+1}=(m_{1,t},\dots,m_{K,t})$
        \State Receive loss vector $l_{t}$, where $l_{t,a}=1-\frac{1}{\epsilon}d_{\epsilon}^{\mathrm{Ber}}(\mu'_{a,t-1}, m_{a,t-1})$
        \State Compute Hedge loss $h_{t}=\sum_{a=1}^{K}\lambda_{t,a}l_{t,a}$
        \State Compute mixed loss $m_{t}=-\frac{1}{\eta_{t}}\log\sum_{a=1}^{K}\lambda_{t,a}e^{-\eta_{t}l_{t,a}}$
        \State Compute mixability gap $\delta_{t}=h_{t}-m_{t}; \Delta_{t}=\Delta_{t-1}+\delta_{t}$
        \State Update cumulative loss $L^{\mathrm{Cumu}}_{t}=L^{\mathrm{Cumu}}_{t-1}+l_{t}$
        \State Update AdaHedge learning rate $\eta_{t+1}=\frac{\log K}{\Delta_{t}}$
    \EndFor
    \State Sample $\hat{\theta}_{T+1}\sim q_{\mathrm{Final}}$
    \State Return $A_{\mathrm{out}}=\arg\max_{a\in[K]}\hat{\theta}_{T+1,a}$
    
\end{algorithmic}
    
\end{algorithm}

We state and prove (in Appendix \ref{apx:DP}) that \textsc{AO-Pri-BAI} is indeed $\epsilon$-DP.

\begin{theorem}[Differential Privacy]\label{thm:DP}
    For every $T$, \textsc{AO-Pri-BAI}$(\epsilon,T)$  satisfies the $\epsilon$-DP constraint.
\end{theorem}

\paragraph{Novelty of our algorithm.} Compared to non-private BAI algorithms proposed by \citet{li2023optimalexplorationharderthompson}, \textsc{AO-Pri-BAI} consists of the following innovations:
\begin{itemize}[leftmargin=*]
    \item In the non-private case, one can compute running statistics of empirical estimations by ordinary least square estimators (OLS), while doing so in private case requires privatizing each reward observation. To preserve privatized running statistics using minimal noise, we incorporate Laplace trees into the algorithm structure. The tree mechanism allows adding noise of $O(\log T)$ scale into node values, and to maintain privatized statistics we only need $O(\log T)$ such values. Together with the concentration property of private estimators proved in Lemma~\ref{lem:GE1} and forced exploration in Lemma~\ref{lem:GE3}, we demonstrate that the addition of such Laplacian noises does not degrade our theoretical guarantee, in the sense that \textsc{AO-Pri-BAI} still asymptotically achieves the lower bound (Theorem \ref{thm:LowerBound}).
    \item In private BAI, the algorithm only observes private estimators perturbed by (Laplacian) noise and clipping. Therefore, unlike non-private algorithms like PEPS in \citet{li2023optimalexplorationharderthompson} which works with tractable posterior distributions, it is challenging to do so with intractable posteriors, which is the case when privacy is enforced. We overcome this difficulty by directly defining Gibbs distributions $\{q_{t}\}_{t=1}^{T}$ and $q_{\mathrm{Final}}$ for sampling instances (for either min-learner or final recommendation). The Gibbs energy is chosen to be the same transportation cost that appears in the lower bound, so the sampling rule is designed to simulate the min-max game defined by the fixed-budget DP exponent rather than around the raw-data likelihood. Despite not being exact Bayesian posteriors, they are effective in learning the optimal design to achieve $\tau^{*}(\theta^{*})$ asymptotically, as demonstrated in Lemma \ref{lem:GE6}.

\end{itemize}


%% file: Sections/upper_bound.tex
\section{Upper Bound}\label{sec:upper_bound}
We now state the asymptotic upper bound on our error probability as Theorem \ref{thm:UpperBound}. It implies that our algorithm \textsc{AO-Pri-BAI} has error probability decaying exponentially fast with a rate at least as $\tau^{*}(\theta^{*})$, the same exponent as defined in the lower bound (Theorem \ref{thm:LowerBound}). Since the exponents are matching, we establish asymptotic optimality of \textsc{AO-Pri-BAI}. The proof can be found in  Appendix~\ref{apx:upperbound}.

\begin{theorem}[Upper Bound]\label{thm:UpperBound}
    With probability 1,
    \[\liminf_{T\rightarrow\infty}-\frac{1}{T}\log\mathbb{P}_{\theta\sim q_{\mathrm{Final}}}(A_{\mathrm{out}}\neq a_{\theta^{*}})\geq\tau^{*}(\theta^{*}).\]
\end{theorem}

\paragraph{Comparison with existing private BAI guarantees.}
Our setting is closely related to the private fixed-budget linear BAI
problem studied by \citet{chen2024fixed}, but our guarantees differ
in both form and interpretation.
Their phase-based algorithm, \textsc{DP-BAI}, uses a MAX-DET principle.
Their Theorem~4.2 gives a non-asymptotic upper bound on the
unconditional error probability of the form
$\exp\bigl(-\Omega(T/(H(\theta^{*})\log d))\bigr)$,
where $d$ is the arm-vector dimension, $T$ is the sampling budget, and
\[
H(\theta^{*})
=
\max_{2\leq i\leq(d^{2}\wedge K)}
\frac{i}{\Delta_{(i)}^{2}}
+
\frac{1}{\epsilon}
\max_{2\leq i\leq(d^{2}\wedge K)}
\frac{i}{\Delta_{(i)}}.
\]
Here, arms are ordered by decreasing mean, and
$\Delta_{(i)}=\mu_{1}-\mu_{i}$ is the $(i-1)$-th smallest optimality gap.
In contrast, Theorem~\ref{thm:UpperBound} gives an asymptotic guarantee
for the conditional error probability of \textsc{AO-Pri-BAI},
with an explicit exponent $\tau^{*}(\theta^{*})$ and no unspecified
multiplicative constant in the exponent.

\paragraph{Different measures of instance difficulty.}
Both $H(\theta^{*})$ and $\tau^{*}(\theta^{*})$ are instance-dependent,
but they capture different aspects of the instance.
The former depends on the optimality gaps, whereas the latter is
the value of a privacy-aware transportation game between a sampling
design and alternative instances in $\Alt(\theta^{*})$.
Appendix~\ref{apx:difference} gives numerical examples in which
two instances have the same value of $H(\theta^{*})$ but different
values of $\tau^{*}(\theta^{*})$, and vice versa.
Thus, neither quantity determines the other.
Our analysis provides a distinct perspective on instance difficulty
through the transportation-game exponent.

\paragraph{Conditional and unconditional error guarantees.}
Our lower and upper bounds concern different notions of error.
Theorem~\ref{thm:LowerBound} bounds the unconditional frequentist
error probability
$\mathbb{M}_{\theta^{*},\pi}
\bigl(A_{\mathrm{out}}\neq a_{\theta^{*}}\bigr)$,
which averages over the observations and the internal randomness
of the policy $\pi$.
In contrast, Theorem~\ref{thm:UpperBound} controls
$\mathbb{P}_{\theta\sim q_{\mathrm{Final}}}
\bigl(A_{\mathrm{out}}\neq a_{\theta^{*}}\bigr)
=q_{\mathrm{Final}}\bigl(\mathrm{Alt}(\theta^{*})\bigr)$,
conditional on the history used to construct the final Gibbs
distribution.
Its guarantee holds along almost every sample path and does not,
by itself, imply the same exponential rate for the unconditional
error probability.
Consequently, the two bounds do not establish a matching
characterization of the unconditional error exponent.
More broadly, \citet{chao2022open} asked whether fixed-budget BAI
admits an instance-dependent complexity characterization:
a benchmark for the optimal error exponent that a single policy
can attain simultaneously across all instances.
\citet{goldberger2026fundamentallimitationsfixedbudgetbestarm}
recently answered this question negatively in general.

\paragraph{Extensions beyond Bernoulli bandits.}
A natural question is whether our analysis extends beyond Bernoulli
rewards.
Appendix~\ref{apx:future_work} outlines a possible extension based on
the quantization technique of \citet{pmlr-v117-riou20a}.

\paragraph{Proof sketch.} We now provide a proof sketch of Theorem \ref{thm:UpperBound}. The proof relies on arguing that six particular good events will happen with high probability. We elaborate on these good events one by one. Their formal statements and proofs are deferred to Appendix \ref{apx:GE}.

Good Event 1, stated in Lemma \ref{lem:GE1}, is used to prove a concentration property of the private estimators $m_{a,t}$ for all arms $a\in[K]$ and across all $t\in[T]$. Together with a lower bound on the number of arm pulls $N_{a,t}$ asserted by Good Event 3 (Lemma \ref{lem:GE3}), it ensures that the private mean estimators converge to the corresponding true mean, which plays a crucial role in the error probability analysis.

Good Event 2, stated in Lemma \ref{lem:GE2}, demonstrates the boundedness property of the private empirical estimator, the alternative instance sampled each round, as well as the final recommendation. This is to ensure that the algorithm operates and terminates without producing any private estimator or final recommendation having an arm mean that is smaller than $0$ or greater than $1$.

Good Event 3, stated in  Lemma \ref{lem:GE3}, shows that the forced exploration, implemented by a mixture of uniform design each round, ensures that each arm can get $\Omega(t^{1-\alpha})$ arm pulls and gather enough exploration by spending $\mathrm{polylog}(T)$ arm pulls in total.

Good Event 4, stated in Lemma \ref{lem:GE4}, proves that with the support of Good Event 1 (Lemma \ref{lem:GE1}) and Good Event 3 (Lemma \ref{lem:GE3}), after a burn-in period of order $\mathrm{polylog}(T)$, the private empirical best arm agrees with the true best arm. This result carries importance into the discussion of the subsequent good events, where we show that the max-learner and the min-learner in \textsc{AO-Pri-BAI} eventually solve this game.

Good Event 5, stated in Lemma \ref{lem:GE5} and with the associated Corollary \ref{col:GE5}, is to establish that after taking rounds of order $\mathrm{polylog}(T)$, the logarithm of the ratio between $q_{\mathrm{Final}}(\theta)$ and $q_{\mathrm{Final}}(\theta^{*})$ is asymptotically equal to the an empirical approximation of the hardness term $\tau^{*}(\theta^{*})$. This result assists in the proof of the upper bound because to prove the upper bound is exactly to control the probability of sampling an instance where the best arm differs from the true one, which is exactly $q_{\mathrm{Final}}(\mathrm{Alt}(\theta^{*}))$. Since $q_{\mathrm{Final}}(\mathrm{Alt}(\theta^{*}))$ is the integral of density $q_{\mathrm{Final}}(\theta)$ over $\mathrm{Alt}(\theta^{*})$, having control over the density allows us to control the integral as well as to later relate it to the optimal error exponent $\tau^{*}(\theta^{*})$.

The final Good Event 6 stated in Lemma \ref{lem:GE6} shows that the final empirical design $\bar{e}_{T}=\{\bar{e}_{a,T}\}_{a=1}^{T}$, is asymptotically achieving the hardness term $\tau^{*}(\theta^{*})$ by bounding their difference with an error term $\varepsilon_{6,T}(\delta)$, which we define by a telescoping sum and then Lemma \ref{lem:GE6control}, \ref{sublem:GE6control1}, \ref{sublem:GE6control2}, \ref{sublem:GE6control3}, \ref{sublem:GE6control4}, \ref{sublem:GE6control5}, \ref{sublem:GE6control6} argue that each of the summand, and thus $\varepsilon_{6,T}(\delta)$, is of order $o(1)$.

The final argument on the upper bound, stated in Theorem \ref{thm:UpperBound}, is then established by assembling the good events proved. With the choice of $\delta_T=T^{-3}$, since the failure probabilities of the primitive good events are summable, the
Borel--Cantelli lemma implies that, with probability one, all six good events
hold for all sufficiently large \(T\).  Hence, the rest of the argument may be
carried out pathwise on a combination of them.

By construction of the final recommendation rule,
\[
\mathbb P_{\theta\sim q_{\mathrm{Final}}}
\bigl(a_{\mathrm{out}}\neq a_{\theta^\ast}\bigr)
=
q_{\mathrm{Final}}\bigl(\Alt(\theta^\ast)\bigr).
\]
Thus, it remains to upper bound the probability mass of the alternative set.

Good Event~5 gives the large-deviation approximation
\[
q_{\mathrm{Final}}\bigl(\Alt(\theta^\ast)\bigr)
\doteq
\exp\big(
-T\ G^\ast(\bar e_T)
\big)\quad\mbox{where} \quad G^\ast(\bar e_T):=\inf_{\theta\in\Alt(\theta^\ast)}
\sum_{a=1}^K
\bar e_{a,T}
d_\epsilon^{\mathrm{Ber}}(\theta_a,\mu_a),
\]
and where $\bar e_{a,T}:=\frac{N_{a,T}}{T}$ and \(a_T\doteq b_T\) denotes equality up to sub-exponential factors.  

Good Event~6 then identifies this empirical-design exponent with the
saddle-point value:
\[
0\le \tau^\ast(\theta^{*})-G^\ast(\bar e_T)\le o(1),
\]
and substituting this into the large-deviation approximation gives
\[
q_{\mathrm{Final}}\bigl(\Alt(\theta^\ast)\bigr)
\le
\exp\bigl(-T(\tau^\ast(\theta^{*})-o(1))\bigr).
\]

Finally, since the error probability of the final recommendation is
exactly \(q_{\mathrm{Final}}(\Alt(\theta^\ast))\), we obtain the bound in Theorem \ref{thm:UpperBound}.

%% file: Sections/experiment.tex
\section{Experiments}\label{sec:experiments}
The empirical performance of \textsc{AO-Pri-BAI} is firstly compared to two baselines: The first   is \textsc{Uniform}, which follows the algorithmic structure of \textsc{AO-Pri-BAI} but with AdaHedge removed, i.e., \textsc{Uniform} does not update arm design and always pulls arms uniformly at random. We treat it as an ablation of \textsc{AO-Pri-BAI} to test the empirical effectiveness of the two online learners. The second is \textsc{DP-TT} from \citet{jourdan2025optimalbestarmidentification}. It is a fixed-confidence algorithm; we modify it to make it stop after pulling $T$ arms and recommend the latest empirical best arm. We call this fixed-budget adaptation \textsc{DP-TT-FB}. Note that both baselines are $\epsilon$-DP, since \textsc{Uniform} has access to observations through private Laplace forests, same as \textsc{AO-Pri-BAI}, and \textsc{DP-TT-FB} has access to observations through Geometric Private Estimator (GPE), whose privacy is guaranteed by Lemma 5 of \citet{jourdan2025optimalbestarmidentification}. Throughout all empirical experiments in this section, we choose $\epsilon=0.1$, $\alpha=\frac{1}{4}$ and $\beta_{t}=2\sqrt{t+1}$. For each algorithm, we conduct $500$ independent trials to plot empirical accuracy (with $\pm1$ standard error) against budget $T$. More experiments are included in Appendix \ref{apx:more_experiments}.

\begin{figure}[!htbp]
    \centering
    \begin{minipage}[t]{0.48\textwidth}
        \vspace{0pt}
        \centering
        \includegraphics[width=1.0\linewidth]{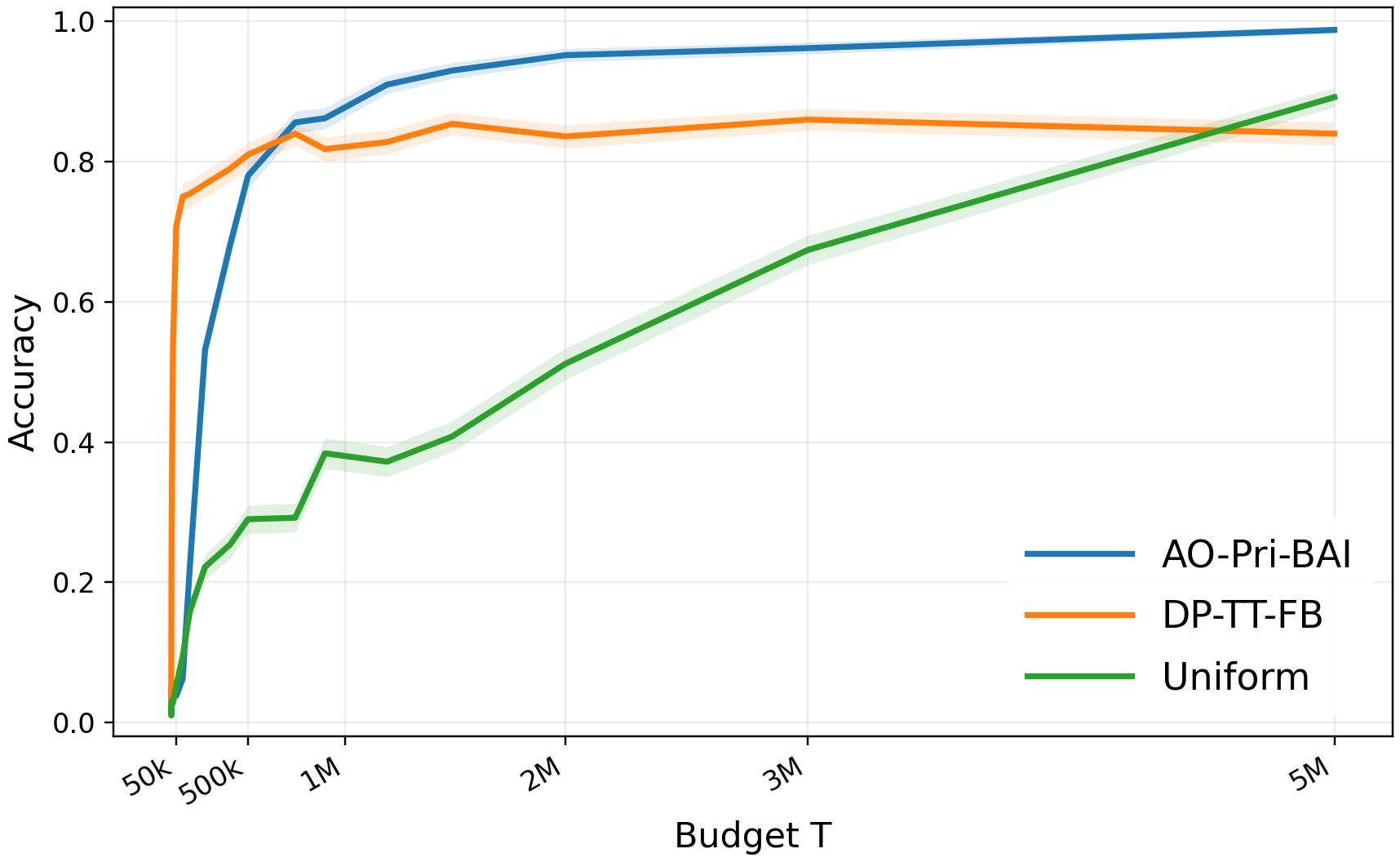}
        \caption{Comparison of \textsc{AO-Pri-BAI} to \textsc{Uniform} and \textsc{DP-TT-FB} on the 50-arm instance.}
        \label{fig:50arms}
    \end{minipage}\hfill
    \begin{minipage}[t]{0.48\textwidth}
        \vspace{0pt}
        \centering
        \includegraphics[width=1.0\linewidth]{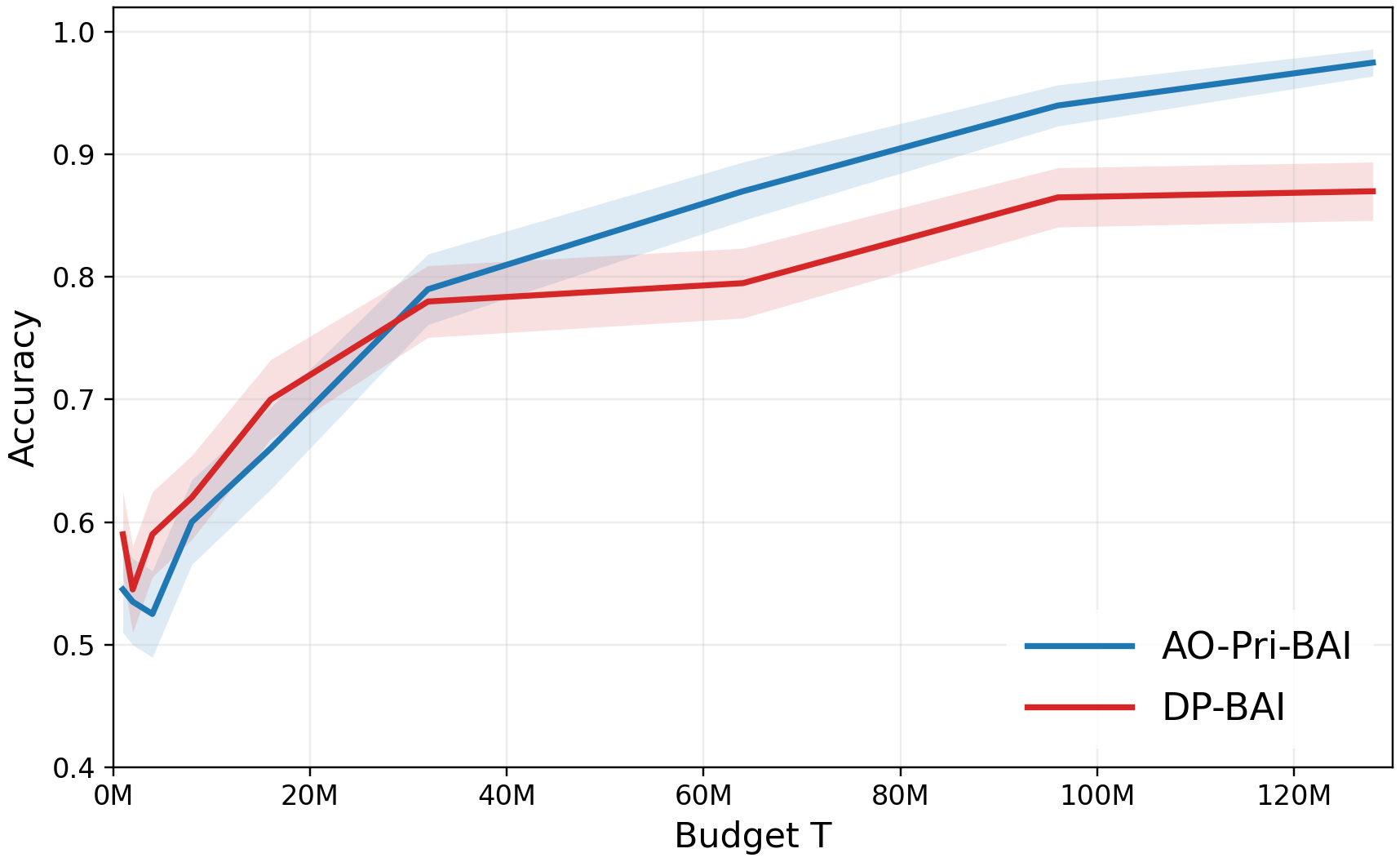}
        \caption{Comparison of \textsc{AO-Pri-BAI} to \textsc{DP-BAI} on the 17-arm instance.}
        \label{fig:vschen}
    \end{minipage}
\end{figure}
We consider a synthetic instance with $50$ arms. The best arm has a mean reward of $0.95$. There are $4$ arms with a mean reward of $0.925$, and another $45$ arms with a mean reward of $0.50$.
Figure~\ref{fig:50arms} reports the experimental results: \textsc{DP-TT-FB} rises quickly but plateaus around the mid-\(0.8\) range as GPE updates become slow in the long run, since it requires doubling arm pulls for any arm to update its private empirical mean.  \textsc{AO-Pri-BAI} catches up with  \textsc{DP-TT-FB} and then clearly dominates in the moderate and large-budget regimes.    \textsc{Uniform} improves slowly because it continues wasting samples on easy distractors. In contrast, \textsc{AO-Pri-BAI} continues improving and approaches perfect accuracy, showing that the max-min learner interaction converts private challenger samples into a more efficient sampling design.  Additional 20-arm and 100-arm experiments in Appendix~\ref{apx:more_experiments} show the same qualitative pattern: our method has a burn-in phase but achieves the strongest large-budget accuracy, consistent with the {\em asymptotic optimality}  guarantee of  \textsc{AO-Pri-BAI}.

We also compare with \textsc{DP-BAI} of \citet{chen2024fixed}, a private elimination-based algorithm, on a challenging $K=17$-arm instance with means $0.6$, $0.59975$, and $0.01$ for the remaining 15 arms. Figure~\ref{fig:vschen} shows that AO-Pri-BAI consistently achieves higher accuracy once the sampling budget is moderately large. Whereas DP-BAI progressively eliminates arms, AO-Pri-BAI maintains a distribution over alternative instances and refines its sampling allocation through max-min optimization. Elimination can be effective when gaps are large or privacy constraints are weak. However, with a small gap and stringent privacy, Laplace noise can cause premature elimination of the best arm. Avoiding such irreversible decisions may explain AO-Pri-BAI's greater robustness on this instance.

%% file: Sections/conclusion.tex
\section{Conclusion}\label{sec:conclusion}
We studied fixed-budget Bernoulli best-arm identification under pure \(\epsilon\)-differential privacy. The main conceptual finding is that privacy fundamentally alters the information geometry of the problem: the classical KL-based transportation cost must be replaced by a privacy-aware divergence that accounts for both stochastic distinguishability and transcript stability. This leads to an instance-dependent exponent \(\tau^\ast(\theta^{*})\) for private fixed-budget BAI. On the algorithmic side, we proposed \textsc{AO-Pri-BAI}, which combines private continual reward estimation, privacy-aware challenger sampling, and an AdaHedge sampling-design learner. We proved that the algorithm is pure \(\epsilon\)-DP for the full adaptive transcript and established an asymptotic error probability upper bound with exponent \(\tau^\ast(\theta^{*})\). Together with the lower bound, this provides an asymptotically optimal fixed-budget theory for private Bernoulli BAI. We conducted some empirical experiments demonstrating the effectiveness of our approach. Several interesting directions remain for future work, which we discuss in Appendix~\ref{apx:future_work}.

%% file: Appendix/table_notations.tex
\newpage
\begin{center}
\rule{\textwidth}{0.8pt}\\[0.8em]
{\Large\bfseries
Supplementary Material for the NeurIPS 2026 Submission:}\\[0.5em]
{\Large\bfseries
Asymptotically Optimal Best Arm Identification with \\Fixed-Budget Under Differential Privacy}\\[0.8em]
\rule{\textwidth}{0.8pt}
\end{center}

\section{Notations}
We list some notations frequently used in the pseudocode of Algorithm $1$, as well as subsequent proofs.

\begin{enumerate}
    \item $S_{a,t}$: Non-private cumulative reward of arm $a$ by round $t$. Defined by $S_{a,t}=\sum_{i=1}^{t}r_{i}\mathbf{1}_{A_{t}=a}$.
    \item $\tilde{S}_{a,t}$: Private cumulative reward of arm $a$ by round $t$. Computed by \textsc{ComputeMeans}$(\{\mathcal{T}_{a}\}_{a=1}^{K}, t, \{N_{a,t}\}_{a=1}^{K})$.
    \item $\tilde{\mu}_{a,t}$: Private empirical mean of arm $a$ at round $t$. Defined by $\tilde{\mu}_{a,t}=\frac{\tilde{S}_{a,t}}{N_{a,t}+1}$.
    \item $\tilde{S}'_{a,t}$: Clipped private cumulative reward of arm $a$ by round $t$. Computed by 
    $\tilde{S}'_{a,t}=[\tilde{S}_{a,t}]_{0}^{N_{a,t}}$.
    \item $\tilde{\mu}'_{a,t}$: Clipped private empirical mean of arm $a$ at round $t$. Defined by $\tilde{\mu}'_{a,t}=\frac{\tilde{S}'_{a,t}}{N_{a,t}+1}$.
    \item $m_{a,t}$: Estimator of true mean $\mu_{a}$ at round $t$. Computed by $m_{a,t}=\frac{1+\tilde{S}'_{a,t}}{2+N_{a,t}}$.
\end{enumerate}

%% file: Appendix/more_related_work.tex
\section{Extended Discussion of Related Work}\label{apx:more_related_work}

This appendix provides a more detailed discussion of the literature summarized in Section~\ref{sec:related_work}.  The goal is to situate the paper within four overlapping lines of work: classical best-arm identification, game-theoretic hardness terms, differential privacy in bandits, and differentially private best-arm identification.

\paragraph{Best-arm identification: fixed-confidence and fixed-budget regimes.}
Best-arm identification is a central topic in pure exploration bandits.  In the fixed-confidence setting, the learner is given a target error probability and aims to minimize the number of samples required to identify the best arm.  Early elimination-style approaches were developed by \citet{evendar2006actionelimination}, and later works provided sharper complexity characterizations and stopping rules.  In particular,   \citet{garivier2016optimalbai} proposed Track-and-Stop and characterized the asymptotic fixed-confidence sample complexity through an instance-dependent characteristic time.

The fixed-budget regime studies a different objective: the sampling budget is fixed in advance, and the goal is to minimize the probability of returning a suboptimal arm.  Classical methods in this regime include Successive Rejects~\citep{audibert2010bestarm} and Sequential Halving~\citep{karnin2013almostoptimal}.   \citet{gabillon2012bestarmunified} proposed a unified view of fixed-budget and fixed-confidence BAI.  \citet{kaufmann2016complexitybai} are especially relevant to our work because they explicitly analyze complexity notions in both regimes and show that fixed-budget and fixed-confidence objectives are not interchangeable.   \citet{carpentier2016fixedbudgetlower} further clarified the fixed-budget regime by deriving tight lower bounds for the error probability.

Posterior-sampling and design-based methods form another important line of work.   \citet{russo2016bayesianbai} introduced Top-Two sampling rules for BAI, adapting posterior sampling to pure exploration by forcing competition between a leader and a challenger.  In structured settings, \citet{soare2014linearbai} connected linear BAI to optimal experimental design, while \citet{yang2022fixedbudgetlinearbai} developed a minimax-optimal fixed-budget method for linear BAI.  These works motivate the view that fixed-budget BAI is fundamentally about learning an allocation that separates the true instance from its most confusing alternatives.

\paragraph{Game-theoretic hardness terms and optimal allocations.}
A recurring theme in modern BAI is that the hardness of an instance can be expressed through a game between two players.  The learner chooses a sampling allocation, while an adversary chooses an alternative instance whose best arm differs from the true one and that is hardest to distinguish under the chosen allocation.  This perspective is implicit in the complexity analysis of \citet{kaufmann2016complexitybai} and in the fixed-confidence characteristic time of \citet{garivier2016optimalbai}.

Several recent works make this game-theoretic structure explicit.  \citet{degenne2020gamification} formulate pure exploration for linear bandits through a game-based lens.  \citet{li2023optimalexplorationharderthompson} use posterior-style sampling and a sampling-design learner to approach the optimal fixed-allocation exponent.  \citet{bian2025asymptotically} develop a related min-learner/max-learner approach for fixed-budget best feasible arm identification, using AdaHedge to update the sampling distribution.  AdaHedge was introduced by  \citet{rooij2014adahedge} as a data-adaptive version of Hedge.

Our work follows this broad max--min tradition, but with a privacy-aware notion of distinguishability.  In non-private BAI, hardness is usually measured by KL divergence, Gaussian quadratic forms, or Chernoff-type quantities.  Under differential privacy, the learner's behavior must also be stable under changes in individual observations.  This motivates replacing the classical transportation geometry by a private one that accounts for both statistical and privacy constraints.

\paragraph{Differential privacy in bandits.}
Differential privacy provides a formal stability guarantee for algorithms operating on sensitive data; see~\citet{dwork2014foundationsdp} for a standard reference.  In bandits, privacy is challenging because the adaptive action sequence itself may reveal information about previous observations.  Continual-release mechanisms, including binary-tree mechanisms for private running sums~\citep{dwork2010panprivate}, provide a standard way to maintain private sequential statistics.

Most early work on private bandits focused on regret minimization.   \citet{jain2012dponlinelearning} studied differentially private online learning.    \citet{mishra2015privatemab} studied private stochastic multi-armed bandits, and \citet{tossou2016privatemab} proposed differentially private UCB-type algorithms.  Privacy has also been studied in contextual and linear bandits, including the work of \citet{shariff2018differentiallyprivatecontextuallinear} and  \citet{agarwal2017priceofdp}.  More recently, \citet{azize2022dpbandits} gave a refined analysis of stochastic and linear bandits under global DP, identifying privacy-dependent regimes in regret minimization.

Although these works mainly address regret minimization, they provide key ideas for private pure exploration: private running statistics, post-processing arguments for adaptive policies, and the recognition that privacy can change the effective information available to the learner.

\paragraph{Differential privacy in best-arm identification.}
Differentially private BAI has received increasing attention only recently.   \citet{nikolakakis2021quantileprivate} studied quantile BAI and proposed a private successive-elimination scheme.   \citet{rio2023privatecommunications} studied multi-agent BAI with private communications.

In the fixed-confidence regime,  \citet{azize2023dpbaifixedconfidence} derived privacy-dependent sample-complexity lower bounds for DP-BAI.  A later work by the same authors~\citep{azize2024dpbai} further developed private BAI under local and global privacy models.  \citet{jourdan2025optimalbestarmidentification} sharpened the Bernoulli global-DP theory by introducing a refined privacy-aware divergence, a private geometric batching estimator, and a Top-Two sampling rule based on private transportation costs.  Their work shows that the classical KL transportation cost is no longer sufficient in private Bernoulli BAI.

The closest work to ours in the fixed-budget private regime is \citet{chen2024fixed}.  They study fixed-budget DP-BAI in linear bandits with bounded rewards, propose the DP-BAI algorithm based on a MAX-DET principle, prove pure DP, and derive exponentially decaying upper and lower bounds.  Their work fills an important gap between non-private fixed-budget BAI and private fixed-confidence BAI.  Our work is closest to~\citet{chen2024fixed} in the setting, because we also study fixed-budget private BAI.  At the same time, it is closer to~\citet{azize2023dpbaifixedconfidence,azize2024dpbai,jourdan2025optimalbestarmidentification} in the sense that both use concepts from information geometry. Indeed, our analysis is built around a privacy-aware transportation cost for Bernoulli arms.

%% file: Appendix/proof_of_lipschitz.tex
\section{Properties of the Privacy-Aware Divergence}\label{apx:proof_lipschitz}
There are some properties of $d_{\epsilon}$, which we have defined in Section \ref{sec:lower_bound}, that are of importance to subsequent theoretical analysis as well as algorithm implementation. In this section, we state those properties explicitly and provide justifications whenever necessary.

\subsection{Closed-form expression and solution}
Recall the definition
\[
d_\epsilon(\theta_1,\theta_2,\lambda)
:=
\inf_{\theta'\in\Theta}
\mathbb E_{a\sim\lambda}
\!\big[
\epsilon\TV(\mathrm{Ber}(\mu_{1,a})\|\mathrm{Ber}(\mu'_a))
+
\KL(\mathrm{Ber}(\mu'_a)\|\mathrm{Ber}(\mu_{2,a}))
\big],
\]

For Bernoulli rewards, the expression separates arm by arm. Since
\[
\TV(\mathrm{Ber}(p)\|\mathrm{Ber}(q))=|p-q|,
\]
we have
\begin{align*}
 d_\epsilon(\theta_1,\theta_2,\lambda)
&=\inf_{\theta'\in\Theta}\sum_{a=1}^K\lambda_a
\Big[
\epsilon|\mu_{1,a}-\mu'_a|
+\mu'_a\log\frac{\mu'_a}{\mu_{2,a}}
+(1-\mu'_a)\log\frac{1-\mu'_a}{1-\mu_{2,a}}
\Big]\\
&=\sum_{a=1}^K\lambda_a
\inf_{u\in(0,1)}
\Big[
\epsilon|\mu_{1,a}-u|
+u\log\frac{u}{\mu_{2,a}}
+(1-u)\log\frac{1-u}{1-\mu_{2,a}}
\Big].
\end{align*}
It is therefore convenient to define the one-dimensional Bernoulli divergence
\[
d_\epsilon^{\mathrm{Ber}}(p,q)
:=
\inf_{u\in(0,1)}
\Big\{
\epsilon\TV(\mathrm{Ber}(p)\|\mathrm{Ber}(u))
+
\KL(\mathrm{Ber}(u)\|\mathrm{Ber}(q))
\Big\}.
\]
With this notation,
\[
d_\epsilon(\theta_1,\theta_2,
\lambda)
=
\sum_{a=1}^K\lambda_a d_\epsilon^{\mathrm{Ber}}(\theta_{1,a},\theta_{2,a}).
\]

The Bernoulli divergence admits a closed-form expression. Let
\[
g_\epsilon^{-}(x)=\frac{x e^\epsilon}{1+x(e^\epsilon-1)},
\qquad
 g_\epsilon^{+}(x)=1-g_\epsilon^{-}(1-x)=\frac{x}{e^\epsilon(1-x)+x}.
\]
Then, for \(p,q\in(0,1)\), \citep{jourdan2025optimalbestarmidentification}
\[
d_\epsilon^{\mathrm{Ber}}(p,q)=
\begin{cases}
    -\log\!\bigl(1-q(1-e^{-\epsilon})\bigr)-\epsilon p, & p\in(0,g_\epsilon^{+}(q)),\\[1mm]
    \kl(p,q), & p\in[g_\epsilon^{+}(q),g_\epsilon^{-}(q)],\\[1mm]
    -\log\!\bigl(1+q(e^\epsilon-1)\bigr)+\epsilon p, & p\in(g_\epsilon^{-}(q),1).
\end{cases}
\]
This formula will be used both in the algorithmic implementation and in the concentration analysis.

\subsection{Lipschitz continuity}
\begin{theorem}\label{thm:Lipschitz-dber}
The map \(q\mapsto d_\epsilon^{\mathrm{Ber}}(p,q)\) is \(C_\mu\)-Lipschitz on \([3\xi_\mu,1-3\xi_\mu]\), uniformly over \(p\in[0,1]\), where
\[
C_\mu=\frac{1}{3\xi_\mu(1-3\xi_\mu)},
\qquad
\xi_\mu=\frac14\min_{a\in[K]}\{\mu_a,1-\mu_a\}.
\]
\end{theorem}

\begin{proof}[Proof of Theorem \ref{thm:Lipschitz-dber}]
    
Fix $p\in[0,1]$ and $q,q'\in[3\xi_{\mu}, 1-3\xi_{\mu}]$, for $u\in(0,1)$,
    \[|(\epsilon|p-u|+\kl(u,q))-(\epsilon|p-u|+\kl(u,q'))|=|\kl(u,q)-\kl(u,q')|.\]

    Notice that
    \[\frac{\partial}{\partial q}\kl(u,q)=\frac{q-u}{q(1-q)}.\]

    Since $q(1-q) \geq 3\xi_{\mu}(1-3\xi_{\mu})$ and $|q-u| \leq 1$, we have
    \[ \bigg|\frac{\partial}{\partial q}\kl(u,q)\bigg| \leq \frac{1}{3\xi_{\mu}(1-3\xi_{\mu})}=C_{\mu}.\]

    By the mean-value theorem, 
    \[|\kl(u,q)-\kl(u,q')| \leq C_{\mu}|q-q'|.\]

    If we define $f(p,q,u)=\epsilon|p-u|+\kl(u,q)$, then it shows that
    \[|f_(p,q,u)-f(p,q',u)| \leq C_{\mu}|q-q'|,\]
    then for every fixed $u\in(0,1)$, it implies $f$ is $C_{\mu}$-Lipschitz with respect to second argument.

    For any $u\in(0,1)$, we have
    \[f(p,q,u) \leq f(p,q',u) + C_{\mu}|q-q'|.\]

    Take infimum over $u$ on both sides, we have
    \[d_{\epsilon}^{\mathrm{Ber}}(p,q)=\inf_{u\in(0,1)}f(p,q,u) \leq \inf_{u\in(0,1)}f(p,q',u)+C_{\mu}|q-q'|=d_{\epsilon}^{\mathrm{Ber}}(p,q')+C_{\mu}|q-q'|.\]

    We swap $q$ and $q'$ to obtain
    \[d_{\epsilon}^{\mathrm{Ber}}(p,q')\inf_{u\in(0,1)}f(p,q',u) \leq \inf_{u\in(0,1)}f(p,q,u)+C_{\mu}|q'-q|=d_{\epsilon}^{\mathrm{Ber}}(p,q)+C_{\mu}|q'-q|.\]

    and combining the two inequalities leads to
    \[|d_{\epsilon}^{\mathrm{Ber}}(p,q)-d_{\epsilon}^{\mathrm{Ber}}(p,q')| \leq C_{\mu}|q-q'|.\]

\end{proof}

\begin{theorem}\label{thm:Lipschitz-dep}
For fixed \(\theta_1,\theta_2\), the map \(\lambda\mapsto d_\epsilon(\theta_1,\theta_2,\lambda)\) is \(\epsilon\)-Lipschitz in $\lambda$ with respect to the \(\ell_1\)-norm. That is, for $\lambda_{1},\lambda_{2}\in\Delta_{[K]}, |d_{\epsilon}(\theta_{1},\theta_{2},\lambda_{1})-d_{\epsilon}(\theta_{1},\theta_{2},\lambda_{2})| \leq \epsilon \| \lambda_{1}-\lambda_{2}\|_{1}$.
\end{theorem}

\begin{proof}[Proof of Theorem~\ref{thm:Lipschitz-dep}]
The claim follows immediately from the facts that \(d_\epsilon^{\mathrm{Ber}}(p,q)\in[0,\epsilon]\) for all \(p\in[0,1]\), \(q\in(0,1)\), and that \(d_\epsilon(\theta_1,\theta_2,\lambda)\) is linear in \(\lambda\).
\end{proof}

%% file: Appendix/lower_bound_proof.tex
\section{Proof of Lower Bound} \label{apx:lowerbound}

We now formally state the proof of Theorem~\ref{thm:LowerBound}.
\begin{proof}[Proof of Theorem \ref{thm:LowerBound}]
We fix a policy $\pi$ consisting of a sampling rule $\{\lambda_t\}_{t=1}^T$ and a recommendation rule $I$. The goal of the proof is to upper bound
\[
\KL \bigl(\mathbb{M}_{\theta',\pi} \,\|\, \mathbb{M}_{\theta^{*},\pi}\bigr).
\]
Following the approach of Theorem~2 in \citet{jourdan2025optimalbestarmidentification}, let $\theta' \in \Alt(\theta^{*})$ be arbitrary, and let $\theta'' \in \Theta$ be an auxiliary (``intermediate'') bandit instance.

For each arm $a \in [K]$, define
\[
c_a \in\argmax_{\gamma\in\Gamma(P_{\theta'}^a,P_{\theta''}^a)} ,\mathbb{P}_{r,r''\sim\gamma}(r'=r''),\]
where  $\Gamma(P_{\theta'}^a,P_{\theta''}^a)$ is the set of all couplings between $P_{\theta'}^a$ and $P_{\theta''}^a$,
as the maximal coupling between $P_{\theta'}^a$ and $P_{\theta''}^a$, such that:
\begin{enumerate}
    \item The marginals are $P_{\theta'}^a$ and $P_{\theta''}^a$ respectively, i.e.,
    \[\sum_{r'}c_{a}(r', r'')=P_{\theta''}^{a}(r''),\]
    \[\sum_{r''}c_{a}(r',r'')=P_{\theta'}^{a}(r').\]
    \item Among all couplings between $P_{\theta'}^a$ and $P_{\theta''}^a$, it maximizes the probability $\mathbb{P}(r'=r'')$. 
    As a consequence, 
    \[\mathbb{P}_{r,r''\sim c_{a}}(r'\neq r'')=TV(P_{\theta'}^a || P_{\theta''}^a).\]
\end{enumerate}

We construct a coupled environment $\gamma$ associated with $(\theta', \theta'')$ over a time horizon $T$, which generates an augmented history
\[
\bigl\{(A_t, r'_t, r''_t)\bigr\}_{t=1}^T \;\text{and}\; A_{\mathrm{out}}.
\]
The interaction proceeds iteratively as follows:
\begin{enumerate}
    \item At each time $t$, the arm $A_t$ is drawn according to the sampling rule $\lambda_t$, based only on the history
    \[
    (A_1, r'_1, A_2, r'_2, \dots, A_{t-1}, r'_{t-1}).
    \]
    In particular, the policy does not depend on $\{r''_i\}_{i=1}^{t-1}$ when selecting arms.
    
    \item Conditionally on $A_t$, the reward pair $(r'_t, r''_t)$ is sampled from the maximal coupling $c_{A_t}$, and is conditionally independent of the past history
    \[
    \{(A_i, r'_i, r''_i)\}_{i=1}^{t-1}.
    \]

    \item The recommended arm $A_{\mathrm{out}}$ is drawn according to the recommendation rule $I$, based only on the same history up to time $T$
    \[
    (A_1, r'_1, A_2, r'_2, \dots, A_{T}, r'_{T}),
    \]
    and is independent of $\{r''_i\}_{i=1}^{T-1}$.
\end{enumerate}

    This interaction between $\pi$ and $\gamma$ mirrors the standard interaction between the policy $\pi$ and the instance $\theta'$, with the addition that, at each round $t$, an auxiliary reward $r''_t$ is sampled via the maximal coupling of $P_{\theta'}^{A_t}$ and $P_{\theta''}^{A_t}$.

The distribution of the augmented history induced by $(\gamma,\pi)$ is given by
\[
p_{\gamma,\pi}(A_{1:T}, r'_{1:T}, r''_{1:T}, A_{\mathrm{out}})
= \Biggl[ \prod_{t=1}^T 
\lambda_t\bigl(A_t \mid A_{1:t-1}, r'_{1:t-1}\bigr)
\, c_{A_t}(r'_t, r''_t) \Biggr]
I\bigl(A_{\mathrm{out}} \mid A_{1:T}, r'_{1:T}\bigr).
\]

For notational convenience, define
\[
\mathbf{A} = (A_1,\dots,A_T), \quad
\mathbf{r}' = (r'_1,\dots,r'_T), \quad
\mathbf{r}'' = (r''_1,\dots,r''_T),
\]
\[
c_{\mathbf{A}}(\mathbf{r}', \mathbf{r}'')
:= \prod_{t=1}^T c_{A_t}(r'_t, r''_t),
\]
and
\[
\lambda(\mathbf{A}, A_{\mathrm{out}} \mid \mathbf{r}')
:= \Biggl[ \prod_{t=1}^T 
\lambda_t\bigl(A_t \mid A_{1:t-1}, r'_{1:t-1}\bigr) \Biggr]
I\bigl(A_{\mathrm{out}} \mid A_{1:T}, r'_{1:T}\bigr).
\]
Then
\[
p_{\gamma,\pi}(\mathbf{A}, \mathbf{r}', \mathbf{r}'', A_{\mathrm{out}})
= \lambda(\mathbf{A}, A_{\mathrm{out}} \mid \mathbf{r}')
\, c_{\mathbf{A}}(\mathbf{r}', \mathbf{r}'').
\]

By construction of the coupling, the distribution $\mathbb{M}_{\theta',\pi}$ is the marginal of $p_{\gamma,\pi}$ after integrating out $(\mathbf{r}', \mathbf{r}'')$, i.e.,
\[
\mathbb{M}_{\theta',\pi}(\mathbf{A}, A_{\mathrm{out}})
= \int p_{\gamma,\pi}(\mathbf{A}, \mathbf{r}', \mathbf{r}'', A_{\mathrm{out}})
\, \rmd\mathbf{r}' \, \rmd\mathbf{r}''.
\]

Next, define the auxiliary joint distribution
\[
q_{\gamma,\pi}(\mathbf{A}, \mathbf{r}', \mathbf{r}'', A_{\mathrm{out}})
= \lambda(\mathbf{A}, A_{\mathrm{out}} \mid \mathbf{r}'')
\frac{p_{\mathbf{A}}(\mathbf{r}'')}{p''_{\mathbf{A}}(\mathbf{r}'')}
\, c_{\mathbf{A}}(\mathbf{r}', \mathbf{r}''),
\]
where
\[
p_{\mathbf{A}}(\mathbf{r}'')
:= \prod_{t=1}^T P_{\theta^*}^{A_t}(r''_t),
\qquad
p''_{\mathbf{A}}(\mathbf{r}'')
:= \prod_{t=1}^T P_{\theta''}^{A_t}(r''_t).
\]

\paragraph{Validity of $q_{\gamma,\pi}$.}
We verify that $q_{\gamma,\pi}$ is a probability distribution:
\begin{align*}
&\sum_{\mathbf{A}, A_{\mathrm{out}}}
\int q_{\gamma,\pi}(\mathbf{A}, \mathbf{r}', \mathbf{r}'', A_{\mathrm{out}})
\, d\mathbf{r}' \, d\mathbf{r}'' \\
&= \sum_{\mathbf{A}, A_{\mathrm{out}}}
\int \lambda(\mathbf{A}, A_{\mathrm{out}} \mid \mathbf{r}'')
\frac{p_{\mathbf{A}}(\mathbf{r}'')}{p''_{\mathbf{A}}(\mathbf{r}'')}
\, c_{\mathbf{A}}(\mathbf{r}', \mathbf{r}'')
\, \rmd\mathbf{r}' \, \rmd\mathbf{r}'' \\
&=\sum_{\mathbf{A}, A_{\mathrm{out}}}
\int\left[\int\lambda(\mathbf{A}, A_{\mathrm{out}} \mid \mathbf{r}'')
\frac{p_{\mathbf{A}}(\mathbf{r}'')}{p''_{\mathbf{A}}(\mathbf{r}'')}
\, c_{\mathbf{A}}(\mathbf{r}', \mathbf{r}'') 
\, \rmd\mathbf{r}' \,\right] \rmd\mathbf{r}''\\
&=\sum_{\mathbf{A}, A_{\mathrm{out}}}
\int \lambda(\mathbf{A}, A_{\mathrm{out}} \mid \mathbf{r}'')
\frac{p_{\mathbf{A}}(\mathbf{r}'')}{p''_{\mathbf{A}}(\mathbf{r}'')}
\left(\int c_{\mathbf{A}}(\mathbf{r}', \mathbf{r}'') \, \rmd\mathbf{r'} \right)\, \rmd \mathbf{r}''\\
&=\sum_{\mathbf{A}, A_{\mathrm{out}}} 
\int \lambda(\mathbf{A}, A_{\mathrm{out}} \mid \mathbf{r}'')
\frac{p_{\mathbf{A}}(\mathbf{r}'')}{p''_{\mathbf{A}}(\mathbf{r}'')}\,
p''_{\mathbf{A}}(\mathbf{r}'')\, \rmd \mathbf{r}''\\
&= \sum_{\mathbf{A}, A_{\mathrm{out}}}
\int \lambda(\mathbf{A}, A_{\mathrm{out}} \mid \mathbf{r}'')
\, p_{\mathbf{A}}(\mathbf{r}'')
\, \rmd\mathbf{r}'' \\
&= \int p_{\mathbf{A}}(\mathbf{r}'') \, \rmd\mathbf{r}'' = 1.
\end{align*}

Moreover, integrating out $(\mathbf{r}', \mathbf{r}'')$ yields
\[
\int q_{\gamma,\pi}(\mathbf{A}, \mathbf{r}', \mathbf{r}'', A_{\mathrm{out}})
\, \rmd\mathbf{r}' \, \rmd\mathbf{r}''
= \mathbb{M}_{\theta^*,\pi}(\mathbf{A}, A_{\mathrm{out}}).
\]

\paragraph{KL contraction.}
By the data-processing inequality,
\[
\KL\bigl(\mathbb{M}_{\theta',\pi} \,\|\, \mathbb{M}_{\theta^*,\pi}\bigr)
\leq \KL\bigl(p_{\gamma,\pi} \,\|\, q_{\gamma,\pi}\bigr).
\]

\paragraph{KL upper bound.}
We now bound the right-hand side:
\begin{align*}
\KL(p_{\gamma,\pi} \| q_{\gamma,\pi})
&= \mathbb{E}_{p_{\gamma,\pi}}
\left[
\log \frac{
\lambda(\mathbf{A}, A_{\mathrm{out}} \mid \mathbf{r}')
}{
\lambda(\mathbf{A}, A_{\mathrm{out}} \mid \mathbf{r}'')
}
+ \log \frac{p''_{\mathbf{A}}(\mathbf{r}'')}{p_{\mathbf{A}}(\mathbf{r}'')}
\right] \tag{a} \\
&\leq \mathbb{E}_{p_{\gamma,\pi}}
\left[
\epsilon \, d_{\mathrm{Ham}}(\mathbf{r}', \mathbf{r}'')
+ \log \frac{p''_{\mathbf{A}}(\mathbf{r}'')}{p_{\mathbf{A}}(\mathbf{r}'')}
\right] \tag{b} \\
&= \sum_{t=1}^T \mathbb{E}_{p_{\gamma,\pi}}
\left[
\epsilon \, \mathbf{1}\{r'_t \neq r''_t\}
+ \log \frac{P_{\theta''}^{A_t}(r''_t)}{P_{\theta^*}^{A_t}(r''_t)}
\right] \tag{c} \\
&= \sum_{t=1}^T \mathbb{E}_{p_{\gamma,\pi}}
\left[
\mathbb{E}\bigg[
\epsilon \, \mathbf{1}\{r'_t \neq r''_t\}
+ \log \frac{P_{\theta''}^{A_t}(r''_t)}{P_{\theta^*}^{A_t}(r''_t)}
\,\Big|\, A_t
\bigg]
\right] \tag{d} \\
&= \sum_{t=1}^T \mathbb{E}_{p_{\gamma,\pi}}
\Bigl[
\epsilon \, \mathrm{TV}(P_{\theta'}^{A_t}, P_{\theta''}^{A_t})
+ \KL(P_{\theta''}^{A_t} \,\|\, P_{\theta^*}^{A_t})
\Bigr] \tag{e} \\
&= \sum_{a=1}^K \mathbb{E}_{\theta',\pi}[N_{a,T}]
\Bigl[
\epsilon \, \mathrm{TV}(P_{\theta'}^{a}, P_{\theta''}^{a})
+ \KL(P_{\theta''}^{a} \,\|\, P_{\theta^*}^{a})
\Bigr]. \tag{f}
\end{align*}

\paragraph{Justifications.}
\begin{enumerate}[label=(\alph*)]
    \item Definition of KL divergence: 
    \begin{align*}
\KL(p_{\gamma,\pi}||q_{\gamma,\pi})&=\mathbb{E}_{p_{\gamma,\pi}}\left[\log\frac{\lambda(\mathbf{A}, A_{\mathrm{out}} \mid \mathbf{r}')
\, c_{\mathbf{A}}(\mathbf{r}', \mathbf{r}'')}{\lambda(\mathbf{A}, A_{\mathrm{out}} \mid \mathbf{r}'')
\frac{p_{\mathbf{A}}(\mathbf{r}'')}{p''_{\mathbf{A}}(\mathbf{r}'')}
\, c_{\mathbf{A}}(\mathbf{r}', \mathbf{r}'')}
    \right]\\
    &=\mathbb{E}_{p_{\gamma,\pi}}
\left[
\log \frac{
\lambda(\mathbf{A}, A_{\mathrm{out}} \mid \mathbf{r}')
}{
\lambda(\mathbf{A}, A_{\mathrm{out}} \mid \mathbf{r}'')
}
+ \log \frac{p''_{\mathbf{A}}(\mathbf{r}'')}{p_{\mathbf{A}}(\mathbf{r}'')}
\right].
    \end{align*}
    \item $\epsilon$-differential privacy (group privacy): 
    \[
    \lambda(\cdot \mid \mathbf{r}')
    \leq \exp\bigl(\epsilon \, d_{\mathrm{Ham}}(\mathbf{r}', \mathbf{r}'')\bigr)
    \lambda(\cdot \mid \mathbf{r}'').
    \]
    \item Decomposition over time.
    \item Law of total expectation.
    \item Under $A_t=a$, $(r'_t, r''_t) \sim c_a$ implies
    \[
    \mathbb{P}(r'_t \neq r''_t) = \mathrm{TV}(P_{\theta'}^a, P_{\theta''}^a).
    \]
    \item Regrouping by arm pulls, where $N_{a,T}$ denotes the number of pulls of arm $a$.
\end{enumerate}

Define the empirical allocation
\[
\lambda' \coloneqq \left\{ \frac{\mathbb{E}_{\theta',\pi}[N_{a,T}]}{T} \right\}_{a=1}^K \in \Delta_{[K]}.
\]
Since the previous bound holds for an arbitrary intermediate instance $\theta'' \in \Theta$, we optimize over $\theta''$ to obtain
\begin{align*}
\KL\bigl(\mathbb{M}_{\theta',\pi} \,\|\, \mathbb{M}_{\theta^*,\pi}\bigr)
&\leq \KL(p_{\gamma,\pi} \,\|\, q_{\gamma,\pi}) \\
&\leq \inf_{\theta'' \in \Theta}
T \sum_{a=1}^K \lambda'_a
\Bigl[
\epsilon \, \mathrm{TV}\bigl(P_{\theta'}^a, P_{\theta''}^a\bigr)
+ \KL\bigl(P_{\theta''}^a \,\|\, P_{\theta^*}^a\bigr)
\Bigr] \\
&= T \inf_{\theta'' \in \Theta}
\mathbb{E}_{a \sim \lambda'}
\Bigl[
\epsilon \, \mathrm{TV}\bigl(P_{\theta'}^a, P_{\theta''}^a\bigr)
+ \KL\bigl(P_{\theta''}^a \,\|\, P_{\theta^*}^a\bigr)
\Bigr] \\
&= T \, d_{\epsilon}(\theta', \theta^*, \lambda').
\end{align*}

Let $a_{\theta'}$ denote the optimal arm under $\theta'$, and define the event
\[
E \coloneqq \{ A_{\mathrm{out}} = a_{\theta'} \}.
\]
By data processing,
\begin{align*}
\kl\bigl(\mathbb{M}_{\theta',\pi}(E), \mathbb{M}_{\theta^*,\pi}(E)\bigr)
&\leq \KL\bigl(\mathbb{M}_{\theta',\pi} \,\|\, \mathbb{M}_{\theta^*,\pi}\bigr) \\
&\leq T \, d_{\epsilon}(\theta', \theta^*, \lambda').
\end{align*}

Let
\[
q \coloneqq \mathbb{M}_{\theta',\pi}(E),
\qquad
p \coloneqq \mathbb{M}_{\theta^*,\pi}(E).
\]
Then
\[
\kl(q,p)
= q \log \frac{q}{p} + (1-q) \log \frac{1-q}{1-p}.
\]
Expanding and lower bounding,
\begin{align*}
\kl(q,p)
&= q \log \frac{1}{p} + q \log q + (1-q)\log(1-q) + (1-q)\log \frac{1}{1-p} \\
&\geq q \log \frac{1}{p} + \bigl[q \log q + (1-q)\log(1-q)\bigr].
\end{align*}
Hence,
\[
\log \frac{1}{p}
\leq \frac{T d_{\epsilon}(\theta', \theta^*, \lambda')
- \bigl[q \log q + (1-q)\log(1-q)\bigr]}{q}.
\]

Since $\theta' \in \mathrm{Alt}(\theta^{*})$, we have $a_{\theta'} \neq a_{\theta^{*}}$ and
\[
\mathbb{M}_{\theta^*,\pi}(A_{\mathrm{out}} \neq a_{\theta^*})
\geq \mathbb{M}_{\theta^*,\pi}(E) = p,
\]
and thus
\[
-\log \mathbb{M}_{\theta^*,\pi}(A_{\mathrm{out}} \neq a_{\theta^*})
\leq -\log p.
\]
Combining the bounds yields
\[
-\log \mathbb{M}_{\theta^*,\pi}(A_{\mathrm{out}} \neq a_{\theta^*})
\leq
\frac{T d_{\epsilon}(\theta', \theta^*, \lambda')
- \bigl[q \log q + (1-q)\log(1-q)\bigr]}{q}.
\]

\paragraph{Asymptotic analysis.}
By consistency of $\pi$, we have $q \to 1$ as $T \to \infty$, and hence
\[
q \log q + (1-q)\log(1-q) \to 0.
\]
Moreover, under the assumed regularity conditions (e.g., dominated convergence),
\[
\lambda' \to \lambda^*,
\]
and by continuity of $d_{\epsilon}$ in $\lambda$,
\[
\lim_{\lambda'\to\lambda^*} d_{\epsilon}(\theta', \theta^*, \lambda') =d_{\epsilon}(\theta', \theta^*, \lambda^*).
\]
 Dividing the log-probability by $T$ and taking $\limsup$ yields
\[
\limsup_{T \to \infty}
-\frac{1}{T}
\log \mathbb{M}_{\theta^*,\pi}(A_{\mathrm{out}} \neq a_{\theta^*})
\leq d_{\epsilon}(\theta', \theta^*, \lambda^*).
\]
Finally, since the bound works for arbitrary $\theta'\in\mathrm{Alt}(\theta^{*})$, we can take the infimum over $\theta'\in\mathrm{Alt}(\theta^{*})$ to obtain
\[
\limsup_{T \to \infty}
-\frac{1}{T}
\log \mathbb{M}_{\theta^*,\pi}(A_{\mathrm{out}} \neq a_{\theta^*})
\leq \inf_{\theta' \in \Alt(\theta^*)} d_{\epsilon}(\theta', \theta^*, \lambda^*).
\]
The right-hand side still depends on $\lambda^{*}$, a particular member of $\Delta_{[K]}$. To remove its dependence, we can take the supremum over all allocations to obtain the final desired bound: 
\[
\limsup_{T \to \infty}
-\frac{1}{T}
\log \mathbb{M}_{\theta^*,\pi}(A_{\mathrm{out}} \neq a_{\theta^*})
\leq
\sup_{\lambda \in \Delta_{[K]}}
\inf_{\theta' \in \Alt(\theta^*)}
d_{\epsilon}(\theta', \theta^*, \lambda)
= \tau^*(\theta^{*}).
\]

\end{proof}

%% file: Appendix/laplace_tree.tex
\section{Storing and Privatizing Reward Data via Laplace Tree}\label{apx:trees}
We now discuss in detail the subroutine of using a binary Laplace tree~\citep{shariff2018differentiallyprivatecontextuallinear} to store privatized cumulative statistics. There are two cumulative statistics that are of particular interest:
\begin{enumerate}
    \item $N_{a,t}$: the number of pulls allocated to arm $a$ by time $t$.
    \item $S_{a,t}$: the cumulative reward of arm $a$ by time $t$.
\end{enumerate}

Using $N_{a,t}$ and $S_{a,t}$, we could compute an estimation for the true mean $\mu_{a}$ for each arm $a$, such as the empirical mean.
To ensure $\epsilon$-DP, we want the reward used to compute the estimators to be privatized, and yet not to query too many privatized statistics to over-split our privacy budget. We leverage binary tree as the data structure.  For each arm $a\in[K]$, we define binary tree $\mathcal{T}_{a}$ as follows:
\begin{enumerate}
    \item Initialize an empty binary tree.
    \item Whenever a reward is observed, record it in the least unoccupied leaf.
    \item Whenever both children of an internal node are filled with a value, record their sum in the internal node.
    \item Add i.i.d Laplacian noise $\omega\sim\text{Lap}(\frac{m_{T}}{\epsilon})$ to each node value, where $m_{T}=\lceil \log_{2}T\rceil+1$.
\end{enumerate}

Whenever we need to compute empirical means for all arms at any round $t$, we do the following:
\begin{enumerate}
    \item Access counts $N_{i,t}$ for all $i\in[K]$.
    \item Compute prefix sum $[1,N_{i,t}]$ using $\mathcal{T}_{i}$ for all $i\in[K]$.
\end{enumerate}

We formalize the process to manage the Laplace trees into several algorithms, which will be relevant when we present our main algorithm.

\subsection{Initialization of Laplace trees}
We first state how to initialize those trees in Algorithm \ref{algo:init_trees} and \ref{algo:build_trees}.
For each arm $a \in [K]$, we maintain a binary tree $\mathcal T_a$ indexed by
local pull counts $1,2,\dots,T$. Let
\[
L_{T} = 2^{\lceil \log_2 T \rceil},
\qquad
m_{T} = \lceil \log_2 T \rceil + 1,
\qquad
b = \frac{m_{T}}{\epsilon} .
\]
Each node $u$ stores a tuple:
\[
(\ell_u, r_u, \mathrm{raw}_u, Z_u, \widetilde{\mathrm{raw}}_u, \mathrm{complete}_u),
\]
where $[\ell_u,r_u]$ is its interval, $\mathrm{raw}_u$ is the raw subtree sum,
$Z_u \sim \mathrm{Lap}(b)$ is sampled once at initialization,
$\widetilde{\mathrm{raw}}_u = \mathrm{raw}_u + Z_u$ is the noisy node value
once the interval is complete, and $\mathrm{complete}_u$ indicates whether all
leaves in the interval have been observed. Since rewards for each arm are
inserted in local pull-count order $1,2,3,\dots$, an interval $[\ell_u,r_u]$
becomes complete exactly when the local pull count reaches $r_u$.

\begin{algorithm}[H]
\caption{Initialize arm-wise Laplace trees (\textsc{Init})}\label{algo:init_trees}
\begin{algorithmic}[1]
\Require Number of arms $K$, horizon $T$, privacy budget $\epsilon$
\State $L_{T} \gets 2^{\lceil \log_2 T \rceil}$
\State $m_{T} \gets \lceil \log_2 T \rceil + 1$
\State $b \gets \frac{m_{T}}{\epsilon}$
\For{$a=1,2,\dots,K$}
    \State $\mathcal T_a \gets \Call{BuildTree}{1, L_{T}, b}$
    \State $N_{a,0} \gets 0$ $\forall a \in [K]$ 
\EndFor
\State \Return $\{\mathcal T_a\}_{a=1}^K$, $\{N_{a,0}\}_{a=1}^{K}$, $(L_{T},m_{T},b)$
\end{algorithmic}
\end{algorithm}

\begin{algorithm}[H]
\caption{\textsc{BuildTree}$(\ell,r,b)$}\label{algo:build_trees}
\begin{algorithmic}[1]
\Require Interval $[\ell,r]$, Laplace scale $b$
\State Create a node $u$
\State $u.\ell \gets \ell$
\State $u.r \gets r$
\State $u.\mathrm{raw} \gets 0$
\State Sample $u.Z \sim \mathrm{Lap}(b)$ 
\State $u.\widetilde{\mathrm{raw}} \gets 0$
\State $u.\mathrm{complete} \gets \mathrm{false}$
\If{$\ell = r$}
    \State $u.\mathrm{left} \gets \mathrm{null}$
    \State $u.\mathrm{right} \gets \mathrm{null}$
\Else
    \State $\mathrm{mid} \gets \lfloor (\ell+r)/2 \rfloor$
    \State $u.\mathrm{left} \gets \Call{BuildTree}{\ell,\mathrm{mid},b}$
    \State $u.\mathrm{right} \gets \Call{BuildTree}{\mathrm{mid}+1,r,b}$
\EndIf
\State \Return $u$
\end{algorithmic}
\end{algorithm}

\subsection{Update of Laplace trees}
At round $t$, an arm $a$ will be pulled, the corresponding Laplace tree will be updated via Algorithms~\ref{algo:update_trees} and \ref{algo:update_node}. One leaf will be filled with the reward value. The leaf to update is the $N_{a,t}$-th leaf, which is the first leaf not updated by round $t$.

\begin{algorithm}[H]
\caption{Update the Laplace tree of the played arm (\textsc{Update})}\label{algo:update_trees}
\begin{algorithmic}[1]
\Require Trees $\{\mathcal T_a\}_{a=1}^K$, previous pull counts $\{N_{a,t-1}\}_{a=1}^{K}$, played arm $a$, reward $r \in \{0,1\}$, round $t$

\State $N_{a,t}=N_{a,t-1}+1$ \Comment{Increase pull count by 1}
\State \Call{UpdateNode}{$\mathcal T_a, N_{a,t}, r$}
\State \Return $\mathcal T_a$, $N_{a,t}$
\end{algorithmic}
\end{algorithm}

\begin{algorithm}[H]
\caption{\textsc{UpdateNode}$(u,s,y)$}\label{algo:update_node}
\begin{algorithmic}[1]
\Require Node $u$, local pull index $s$, reward $y \in \{0,1\}$
\If{$s < u.\ell$ \textbf{or} $s > u.r$}
    \State \Return
\EndIf

\State $u.\mathrm{raw} \gets u.\mathrm{raw} + y$ 

\If{$u.\ell = u.r$}
    \State $u.\widetilde{\mathrm{raw}} \gets u.\mathrm{raw} + u.Z$
    \State $u.\mathrm{complete} \gets \mathrm{true}$
    \State \Return
\EndIf

\If{$s \le u.\mathrm{left}.r$}
    \State \Call{UpdateNode}{$u.\mathrm{left}, s, y$}
\Else
    \State \Call{UpdateNode}{$u.\mathrm{right}, s, y$}
\EndIf

\If{$s = u.r$ \textbf{and} $u.\mathrm{complete} = \mathrm{false}$}
    \State $u.\widetilde{\mathrm{raw}} \gets u.\mathrm{raw} + u.Z$
    \State $u.\mathrm{complete} \gets \mathrm{true}$
\EndIf
\end{algorithmic}
\end{algorithm}

\subsection{Compute privatized empirical means}

To update the running private statistics, we compute the private empirical mean each round. The crucial step is to retrieve the private cumulative reward for each arm. This is implemented via Algorithm \ref{algo:prefixsum} and \ref{algo:computemeans}.

\begin{algorithm}[H]
\caption{\textsc{PrivatePrefixSum}$(u,n)$}\label{algo:prefixsum}
\begin{algorithmic}[1]
\Require Node $u$, prefix length $n$

\If{$n < u.\ell$}
    \State \Return $0$
\EndIf
\If{$u.r \le n$}
    \State \Return $u.\widetilde{\mathrm{raw}}$
\EndIf
\State \Return $\Call{PrivatePrefixSum}{u.\mathrm{left}, n} + \Call{PrivatePrefixSum}{u.\mathrm{right}, n}$
\end{algorithmic}
\end{algorithm}

\begin{algorithm}[H]
\caption{Compute private empirical means from the arm-wise trees (\textsc{ComputeMeans})}\label{algo:computemeans}
\begin{algorithmic}[1]
\Require Trees $\{\mathcal T_a\}_{a=1}^K$, round $t$, pull counts $\{N_{a,t}\}_{a=1}^{K}$
\For{$a=1,2,\dots,K$}
    \If{$N_{a,t} = 0$}
        \State $\tilde{S}_{a,t} \gets 0$
        \State $\tilde{S}'_{a,t} \gets 0$
        \State $m_{a,t}=\frac{1}{2}$
    \Else
        \State $\tilde{S}_{a,t} \gets \Call{PrivatePrefixSum}{\mathcal T_a, N_{a,t}}$
        \State $\tilde{S}'_{a,t} \gets [\tilde{S}_{a,t}]_{0}^{N_{a,t}}$
        \State $m_{a,t}=\frac{1+\tilde{S}'_{a,t}}{2+N_{a,t}}$
    \EndIf
\EndFor
\State \Return $\{\tilde{S}_{a,t}\}_{a=1}^{K}, \{\tilde{S}'_{a,t}\}_{a=1}^{K}, \{m_{a,t}\}_{a=1}^{K}$
\end{algorithmic}
\end{algorithm}

%% file: Appendix/privacy_proof.tex
\section{Differential Privacy Analysis of \textsc{AO-Pri-BAI}}\label{apx:DP}
We now state the proof of Theorem \ref{thm:DP}.
\begin{proof}[Proof of Theorem \ref{thm:DP}]
Fix two neighbouring reward matrices
\[
r=\bigl\{r_{a,t}\bigr\}_{a\in[K], t\in[T]}, \,
r'=\bigl\{r'_{a,t}\bigr\}_{a\in[K], t\in[T]}.
\]
so that there exists a unique pair $(a^{*},t^{*})\in [K]\times[T]$ such that
\[
r_{a,t}=r'_{a,t} \, \forall \, (a,t) \neq (a^{*},t^{*}).
\]

For each arm $a\in[K]$, let $\mathcal{T}_a$ denote its Laplace tree. For every node $u \in \mathcal{T}_a$, write $I_{a,u}\subseteq[T]$ for the local pull-count interval associated with u, and define the corresponding raw subtree sum by
\[
S_{a,u}(r):=\sum_{s\in I_{a,u}} r_{a,s}.
\]

\textsc{AO-Pri-BAI} releases the noisy node value
\[
\tilde{S}_{a,u}(r):=S_{a,u}(r)+Z_{a,u},
\]

where the random variables $\{Z_{a,u}\}$ are mutually independent and satisfy
\[
Z_{a,u}\sim \mathrm{Lap}(b),
\qquad
b=\frac{m_T}{\epsilon},
\qquad
m_T:=\bigl\lceil \log_2 T\bigr\rceil+1.
\]

Let $\mathcal{M}(r):=\bigl\{\widetilde S_{a,u}(r)\bigr\}_{a,u}$ denote the full noisy Laplace forest generated from r.
We first show that the mechanism $r\mapsto \mathcal{M}(r)$ is $\epsilon$-DP.

For any realization $y={y_{a,u}}$ of the forest, the joint density of $\mathcal{M}(r)$ is

\[
\mathcal{M}_{r}(y)
=\prod_{a,u}\frac{1}{2b}
\exp\Bigl(-\frac{\bigl|y_{a,u}-S_{a,u}(r)\bigr|}{b}\Bigr),
\]

and similarly

\[
\mathcal{M}_{r'}(y)
=\prod_{a,u}\frac{1}{2b}
\exp\Bigl(-\frac{\bigl|y_{a,u}-S_{a,u}(r')\bigr|}{b}\Bigr).
\]

Therefore, the joint density ratio is given by

\[
\frac{\mathcal{M}_{r}(y)}{\mathcal{M}_{r'}(y)}
=\prod_{a,u}
\frac{
\exp\Bigl(-\frac{\bigl|y_{a,u}-S_{a,u}(r)\bigr|}{b}\Bigr)
}{
\exp\Bigl(-\frac{\bigl|y_{a,u}-S_{a,u}(r')\bigr|}{b}\Bigr)
}
=\prod_{(a,u) \in P^{*}}
\frac{
\exp\Bigl(-\frac{\bigl|y_{a,u}-S_{a,u}(r)\bigr|}{b}\Bigr)
}{
\exp\Bigl(-\frac{\bigl|y_{a,u}-S_{a,u}(r')\bigr|}{b}\Bigr)
}.
\]

Hence,
\[
\log\frac{\mathcal{M}_{r}(y)}{\mathcal{M}_{r'}(y)}
=\frac{1}{b}\sum_{a,u}
\Bigl(
\bigl|y_{a,u}-S_{a,u}(r')\bigr|
-\bigl|y_{a,u}-S_{a,u}(r)\bigr|
\Bigr).
\]

By triangle inequality, we obtain
\[
\log\frac{\mathcal{M}_{r}(y)}{\mathcal{M}_{r'}(y)}
\leq \frac{1}{b}\sum_{a,u}
\bigl|S_{a,u}(r)-S_{a,u}(r')\bigr|.
\]

Now observe that the two reward tables differ at only one entry, namely $(a^{*},t^{*})$. Hence:

\begin{enumerate}
    \item for every arm $a\neq a^{*}$, all node sums in $\mathcal{T}_{a}$ are unchanged;
    \item within the tree $\mathcal{T}_{a^\star}$, only the nodes lying on the unique root-to-leaf path from the leaf corresponding to local pull count $t^{*}$ to the root can change.
\end{enumerate}

The number of such nodes are at most equal to the height of the tree, which is $m_{T}$. Moreover, since rewards are Bernoulli,
\[
\bigl|r_{a^{*},t^{*}}-r'_{a^{*},t^{*}}\bigr|\leq 1.
\]

Consequently
\[
\sum_{a,u}\bigl|S_{a,u}(r)-S_{a,u}(r')\bigr| \leq m_{T}.
\]

It follows that
\[
\log\frac{\mathcal{M}_{r}(y)}{\mathcal{M}_{r'}(y)}
\leq \frac{m_{T}}{b} = \epsilon.
\]

That is
\[
\mathcal{M}_{r}(y) \leq e^{\epsilon}\mathcal{M}_{r'}(y).
\]

Hence the Laplace forest mechanism $r\mapsto \mathcal{M}(r)$ is $\epsilon$-DP.

We now pass from the noisy forest to the full transcript of \textsc{AO-Pri-BAI}. By construction, all subsequent quantities used by the algorithm—private means, the empirical leader, the challenger distribution, the mixed sampling distribution, the sampled actions, and the final recommendation—are computed solely from $\mathcal{M}(r)$ together with internal randomness that is independent of the reward table. Therefore the full output transcript $\pi(r)=\bigl(A_{1},\dots,A_{T},A_{\mathrm{out}}\bigr) \in [K]^{T+1}$ is a randomized post-processing of $\mathcal{M}(r)$.

By the post-processing property of $\epsilon$-DP, once $\mathcal{M}(r)$ is $\epsilon$-DP, we have the transcript machanism $r \mapsto \pi(r)$ is also $\epsilon$-DP.
\end{proof}

%% file: Appendix/upper_bound_proof.tex
\section{Proof of Upper Bound}\label{apx:upperbound}
We now state the formal proof of Theorem \ref{thm:UpperBound}.
\begin{proof}[Proof of Theorem \ref{thm:UpperBound}]
    Let
    \[\delta_{T}:=\frac{1}{T^{3}},\]
    and define
    \[\varepsilon_{5,T}:=2C_{\mu}r_{T}\bigg(\frac{T^{1-\alpha}-1}{2K(1-\alpha)},\delta_{T}\bigg).\]

    By construction of the final recommendation step in Algorithm \ref{algo:FB-DP-BAI},
    \[A_{\mathrm{out}}=\arg\max_{a\in[K]}\theta_{a}\quad\text{when }\theta\sim q_{\mathrm{Final}},\]
    and therefore the concerned error probability is exactly
    \[\mathbb{P}_{\theta\sim q_{\mathrm{Final}}}(A_{\mathrm{out}}\neq a_{\theta^{*}})=q_{\mathrm{Final}}(\mathrm{Alt}(\theta^{*})).\]

    Define
    \[I_{T}(\theta):=\sum_{a=1}^{K}\bar{e}_{a,T}d_{\epsilon}^{\mathrm{Ber}}(\theta_{a},\mu_{a}),\qquad \bar{e}_{a,T}=\frac{N_{a,T}}{T}.\]
    Then, by definition of $G^{*}$,
    \[G^{*}(\bar{e}_{T})=\inf_{\theta\in\mathrm{Alt}(\theta^{*})}I_{T}(\theta).\]

    We first show that Good Events 5 and 6 hold eventually almost surely under the summable schedule $\delta_{T}=T^{-3}$. By Lemma~\ref{lem:GE5},
    \[E_{1,T}(\delta_{T})\cap E_{3,T}(\delta_{T})\subseteq E_{5,T}(\delta_{T}).\]
    By Lemma~\ref{lem:GE4},
    \[E_{1,T}(\delta_{T})\cap E_{3,T}(\delta_{T})\subseteq E_{4,T}(\delta_{T}),\]
    and since $E_{2,T}$ is deterministic, Lemma~\ref{lem:GE6} implies
    \[E_{1,T}(\delta_{T})\cap E_{3,T}(\delta_{T})\subseteq E_{6,T}(\delta_{T}).\]
    Hence, if we define
    \[\mathcal{E}_{T}:=E_{5,T}(\delta_{T})\cap E_{6,T}(\delta_{T}),\]
    then
    \[E_{1,T}(\delta_{T})\cap E_{3,T}(\delta_{T})\subseteq \mathcal{E}_{T},\]
    and therefore, by Lemmas~\ref{lem:GE1} and \ref{lem:GE3},
    \[\mathbb{P}(\mathcal{E}_{T}^{c})\leq \mathbb{P}(E_{1,T}(\delta_{T})^{c})+\mathbb{P}(E_{3,T}(\delta_{T})^{c})\leq 2\delta_{T}=\frac{2}{T^{3}}.\]
    Since \(\sum_{T=1}^{\infty}\frac{2}{T^{3}}<\infty\), Borel-Cantelli lemma shows that there exists an event $\Omega_{0}$ with $\mathbb{P}(\Omega_{0})=1$ such that for every $\omega\in\Omega_{0}$, there exists $T_{0}(\omega)$ with
    \[\mathcal{E}_{T}(\omega)\text{ holding for all }T\geq T_{0}(\omega).\]

    By Lemma~\ref{lem:GE6control}, there exists an event $\Omega_{1}$ with $\mathbb{P}(\Omega_{1})=1$ such that
    \[\varepsilon_{6,T}(\delta_{T})\to 0\qquad \text{on }\Omega_{1}.\]
    Let
    \[\Omega:=\Omega_{0}\cap\Omega_{1}.\]
    Then $\mathbb{P}(\Omega)=1$.

    Also, \(\varepsilon_{5,T}\to 0\) deterministically. Indeed,
    \[r_{T}(x,\delta)=\rho_{T}(x,\delta)+\frac{1}{x+2},\]
    and with \(x=\frac{T^{1-\alpha}-1}{2K(1-\alpha)}\), \(\delta_{T}=T^{-3}\), \(m_{T}\in O(\log T)\), and \(L_{T}\in O(T)\), we have
    \[r_{T}\bigg(\frac{T^{1-\alpha}-1}{2K(1-\alpha)},\delta_{T}\bigg)\in o(1),\]
    hence
    \[\varepsilon_{5,T}\in o(1).\]

    Fix an arbitrary sample path $\omega\in\Omega$. Then for all sufficiently large $T$, both $E_{5,T}(\delta_{T})$ and $E_{6,T}(\delta_{T})$ hold, and also \(\varepsilon_{6,T}(\delta_{T})\to 0\).

    Now fix an arbitrary $\zeta>0$. Since for each arm $a\in[K]$,
    \[d_{\epsilon}^{\mathrm{Ber}}(\mu_{a},\mu_{a})=0,\]
    and \(p\mapsto d_{\epsilon}^{\mathrm{Ber}}(p,\mu_{a})\) is continuous on $(0,1)$, there exists $\rho_{\zeta}>0$ such that
    \[0<\rho_{\zeta}<\min_{a\in[K]}\{\mu_{a},1-\mu_{a}\},\]
    and whenever \(|p-\mu_{a}|\leq \rho_{\zeta}\),
    \[d_{\epsilon}^{\mathrm{Ber}}(p,\mu_{a})\leq \zeta\qquad \forall a\in[K].\]

    Define the box
    \[B_{\zeta}:=\prod_{a=1}^{K}[\mu_{a}-\rho_{\zeta},\mu_{a}+\rho_{\zeta}]\subset\Theta.\]
    Its Lebesgue volume is
    \[|B_{\zeta}|=(2\rho_{\zeta})^{K}>0.\]
    For every \(\theta\in B_{\zeta}\), we have
    \[d_{\epsilon}^{\mathrm{Ber}}(\theta_{a},\mu_{a})\leq \zeta\qquad \forall a\in[K],\]
    and therefore
    \[I_{T}(\theta)=\sum_{a=1}^{K}\bar{e}_{a,T}d_{\epsilon}^{\mathrm{Ber}}(\theta_{a},\mu_{a})\leq \zeta\sum_{a=1}^{K}\bar{e}_{a,T}=\zeta.\]

    Define the density ratio
    \[R_{T}(\theta):=\frac{q_{\mathrm{Final}}(\theta)}{q_{\mathrm{Final}}(\mu)},\qquad \theta\in\Theta.\]
    Since $q_{\mathrm{Final}}(\mu)>0$ and $\int_{\Theta}q_{\mathrm{Final}}(\theta)d\theta=1$, we have
    \[q_{\mathrm{Final}}(\mathrm{Alt}(\theta^{*}))=\frac{\int_{\mathrm{Alt}(\theta^{*})}R_{T}(\theta)d\theta}{\int_{\Theta}R_{T}(\theta)d\theta}.\]

    On $E_{5,T}(\delta_{T})$, for every \(\theta\in\Theta\),
    \[-I_{T}(\theta)-\varepsilon_{5,T}\leq \frac{1}{T}\log R_{T}(\theta)\leq -I_{T}(\theta)+\varepsilon_{5,T}.\tag{8.1}\]

    On $E_{6,T}(\delta_{T})$, we have
    \[0\leq \tau^{*}(\theta^{*})-G^{*}(\bar{e}_{T})\leq \varepsilon_{6,T}(\delta_{T}),\]
    hence
    \[G^{*}(\bar{e}_{T})\geq \tau^{*}(\theta^{*})-\varepsilon_{6,T}(\delta_{T}).\tag{8.2}\]
    Since \(G^{*}(\bar{e}_{T})=\inf_{\theta\in\mathrm{Alt}(\theta^{*})}I_{T}(\theta)\), (8.2) implies that for every \(\theta\in\mathrm{Alt}(\theta^{*})\),
    \[I_{T}(\theta)\geq \tau^{*}(\theta^{*})-\varepsilon_{6,T}(\delta_{T}).\tag{8.3}\]

    We now bound the numerator. For every \(\theta\in\mathrm{Alt}(\theta^{*})\), combining (8.1) and (8.3),
    \[\frac{1}{T}\log R_{T}(\theta)\leq -I_{T}(\theta)+\varepsilon_{5,T}\leq -\tau^{*}(\theta^{*})+\varepsilon_{6,T}(\delta_{T})+\varepsilon_{5,T}.\]
    Therefore
    \[R_{T}(\theta)\leq \exp\Big(-T(\tau^{*}(\theta^{*})-\varepsilon_{6,T}(\delta_{T})-\varepsilon_{5,T})\Big)\qquad \forall \theta\in\mathrm{Alt}(\theta^{*}).\]
    Since \(\mathrm{Alt}(\theta^{*})\subset \Theta=(0,1)^{K}\) has Lebesgue volume at most \(1\),
    \[\int_{\mathrm{Alt}(\theta^{*})}R_{T}(\theta)d\theta\leq \exp\Big(-T(\tau^{*}(\theta^{*})-\varepsilon_{6,T}(\delta_{T})-\varepsilon_{5,T})\Big).\tag{8.4}\]

    We next bound the denominator from below. For every \(\theta\in B_{\zeta}\), combining (8.1) and the bound \(I_{T}(\theta)\leq \zeta\),
    \[\frac{1}{T}\log R_{T}(\theta)\geq -I_{T}(\theta)-\varepsilon_{5,T}\geq -\zeta-\varepsilon_{5,T}.\]
    Hence
    \[R_{T}(\theta)\geq \exp\Big(-T(\zeta+\varepsilon_{5,T})\Big)\qquad \forall \theta\in B_{\zeta},\]
    and therefore
    \[\int_{\Theta}R_{T}(\theta)d\theta\geq \int_{B_{\zeta}}R_{T}(\theta)d\theta\geq |B_{\zeta}|\exp\Big(-T(\zeta+\varepsilon_{5,T})\Big).\tag{8.5}\]

    Combining (8.4) and (8.5), we obtain
    \[q_{\mathrm{Final}}(\mathrm{Alt}(\theta^{*}))\leq |B_{\zeta}|^{-1}\exp\Big(-T(\tau^{*}(\theta^{*})-\varepsilon_{6,T}(\delta_{T})-2\varepsilon_{5,T}-\zeta)\Big).\]
    Taking \(-\frac{1}{T}\log\) on both sides gives
    \[-\frac{1}{T}\log q_{\mathrm{Final}}(\mathrm{Alt}(\theta^{*}))\geq \tau^{*}(\theta^{*})-\varepsilon_{6,T}(\delta_{T})-2\varepsilon_{5,T}-\zeta-\frac{1}{T}\log |B_{\zeta}|^{-1}.\]

    Letting \(T\to\infty\) along the fixed sample path \(\omega\in\Omega\), and using \(\varepsilon_{6,T}(\delta_{T})\to 0\), \(\varepsilon_{5,T}\to 0\), and \(\frac{1}{T}\log |B_{\zeta}|^{-1}\to 0\), we get
    \[\liminf_{T\to\infty}-\frac{1}{T}\log q_{\mathrm{Final}}(\mathrm{Alt}(\theta^{*}))\geq \tau^{*}(\theta^{*})-\zeta.\]
    Since \(\zeta>0\) was arbitrary,
    \[\liminf_{T\to\infty}-\frac{1}{T}\log q_{\mathrm{Final}}(\mathrm{Alt}(\theta^{*}))\geq \tau^{*}(\theta^{*}).\]
    Equivalently,
    \[\liminf_{T\to\infty}-\frac{1}{T}\log\mathbb{P}_{\theta\sim q_{\mathrm{Final}}}(A_{\mathrm{out}}\neq a_{\theta^{*}})\geq \tau^{*}(\theta^{*}).\]
\end{proof}

%% file: Appendix/comparison.tex
\section{Difference between Our work and \citet{chen2024fixed}:  Two Numerical Examples}\label{apx:difference}

\paragraph{Example 1.}Consider $\epsilon=0.5$ and the following pair of two-arm instances:
\[\theta_{1}=(0.95, 0.85) \quad \theta_{2}=(0.55, 0.45).\]

By Theorem 4.2 of \citet{chen2024fixed}, we have
\[H(\theta_{1})=H(\theta_{2})=240.\]

Therefore \citet{chen2024fixed}  characterizes the two instances as being equally hard. This is expected because although $\theta_{1}\neq\theta_{2}$, the optimality gap stays the same as $0.1$.

On the other hand, for a two-arm design $\lambda\in\Delta_{[2]}$, we can parametrize it as $\lambda=(w,1-w)$, where $w\in[0,1]$. By definition, we have
\[\tau^{*}(\theta^{*})=\sup_{w\in[0,1]}\inf_{\substack{p\leq q \\ p,q\in(0,1)}}[wd_{\epsilon}^{\mathrm{Ber}}(p,\mu_{1})+(1-w)d_{\epsilon}^{\mathrm{Ber}}(q,\mu_{2})].\]

Given a fixed $w$, the inner infimum occurs at a tie, therefore
\[\tau^{*}(\theta^{*})=\sup_{w\in[0,1]}\inf_{p\in(0,1)}[wd_{\epsilon}^{\mathrm{Ber}}(p,\mu_{1})+(1-w)d_{\epsilon}^{\mathrm{Ber}}(p,\mu_{2})].\]

Let
\[f_{w}(p)=wd_{\epsilon}^{\mathrm{Ber}}(p,\mu_{1})+(1-w)d_{\epsilon}^{\mathrm{Ber}}(p,\mu_{2}),\]
and
\[c(w)=\inf_{p\in(0,1)}f_{w}(p),\]
then
\[\tau^{*}(\theta^{*})=\sup_{w\in[0,1]}c(w).\]

By the closed-form solution of $d_{\epsilon}^{\mathrm{Ber}}$ discussed in Appendix \ref{apx:proof_lipschitz}, we have that $d_{\epsilon}^{\mathrm{Ber}}(p,\mu_{1})$ and $d_{\epsilon}^{\mathrm{Ber}}(p,\mu_{2})$ are both convex in $p$, thus $f_{w}$ is also convex in $p$, as it is a convex combination of convex functions. Furthermore, $c(w)$ is concave in $w$, as it is the pointwise infimum of a family of affine functions in $w$. Therefore, we can evaluate $\tau^{*}(\theta^{*})$ by a nested one-dimensional convex/concave optimization. We can use any efficient solver to compute a numerical result. The following results are computed using a numerical solver in SciPy.

When $\theta^{*}=\theta_{1}$, we have
\[\tau^{*}(\theta_{1})\approx1.45\times10^{-2}.\]

When $\theta^{*}=\theta_{2}$, we have
\[\tau^{*}(\theta_{2})\approx5.03\times10^{-3}.\]

Therefore, $\tau^{*}(\theta^{*})$ has characterized $\theta_{2}$ as harder than $\theta_{1}$, since smaller $\tau^{*}(\theta^{*})$ implies slower exponential decay rate.

\paragraph{Example 2.}

Consider $\epsilon=\log2$ and the following pair of two-arm instances:
\[\theta_{1}=(\frac{8}{9},\frac{1}{5}) \quad \theta_{2}=(\frac{46}{63}, \frac{1}{20}).\]

We proceed to compute $H(\theta^{*})$ and $\tau^{*}(\theta^{*})$ for $\theta^{*} \in \{\theta_{1},\theta_{2}\}$ similar to how we compute them in the previous example. We abbreviate the computation process and show the results directly. We note that for this example, we have
\[\tau^{*}(\theta_{1})=\tau^{*}(\theta_{2})=\frac{1}{2}\log\frac{3}{2},\]
so $\tau^{*}(\theta^{*})$ characterizes the two instances as being equally hard.

On the other hand, when $\theta^{*}=\theta_{1}$, we have
\[H(\theta_{1})=\frac{4050}{961}+\frac{90}{31\log2}\approx8.403.\]

When $\theta^{*}=\theta_{2}$, we have
\[H(\theta_{2})=\frac{3175200}{734449}+\frac{2520}{857\log2}\approx8.565.\]

Therefore, in this case, $H(\theta^{*})$ has characterized $\theta_{2}$ as harder than $\theta_{1}$.

%% file: Appendix/good_events_proof.tex
\section{Proofs of Good Event Lemmas}\label{apx:GE}

We now state the formal good event lemmas and their proofs.
\begin{lemma}[Good Event 1]\label{lem:GE1}
    Define good event $E_{1,T}(\delta)$ as
    \[\forall a\in[K], \forall t\in[T], \qquad
    \bigl|m_{a,t}-\mu_{a}\bigr| \leq \rho_{T}(N_{a,t},\delta)+\frac{1}{N_{a,t}+2},\]
    where $\displaystyle \rho_{T}(N_{a,t},\delta)=\sqrt{\frac{\log\frac{4KT}{\delta}}{2(N_{a,t}+1)}}+\frac{\frac{m_{T}^{2}}{\epsilon}\log\Bigl(\frac{2K(2L_{T}-1)}{\delta}\Bigr)+1}{N_{a,t}+1}$, then
    \[\mathbb{P}\bigl(E_{1,T}(\delta)\bigr) \geq 1-\delta.\]
\end{lemma}
\begin{proof}[Proof of Lemma \ref{lem:GE1}]
    We first consider when $N_{a,t}=0$. It is trivial to note that LHS $\leq\frac{1}{2}$ and RHS$\geq \frac{1}{2}$.

    We then consider when $N_{a,t}\geq 1$. Note that
    \begin{align*}
        \bigl|\tilde{\mu}'_{a,t}-\mu_{a}\bigr|
        &=\Bigl|\frac{\tilde{S}'_{a,t}}{N_{a,t}+1}-\mu_{a}\Bigr| \\
        &=\Bigl|\frac{\tilde{S}'_{a,t}-S_{a,t}}{N_{a,t}+1}
        +\frac{S_{a,t}-N_{a,t}\mu_{a}}{N_{a,t}+1}
        -\frac{\mu_{a}}{N_{a,t}+1}\Bigr|\\
        &\leq \frac{\bigl|\tilde{S}'_{a,t}-S_{a,t}\bigr|}{N_{a,t}+1}
        +\frac{\bigl|S_{a,t}-N_{a,t}\mu_{a}\bigr|}{N_{a,t}+1}
        +\frac{\mu_{a}}{N_{a,t}+1} \\
        &\leq \frac{\bigl|\tilde{S}'_{a,t}-S_{a,t}\bigr|}{N_{a,t}+1}
        +\frac{\bigl|S_{a,t}-N_{a,t}\mu_{a}\bigr|}{N_{a,t}+1}
        +\frac{1}{N_{a,t}+1}.
    \end{align*}

    For each arm $a\in[K]$, we let $X_{a,n}$ denote the random variable which is the reward of pulling arm $a$ for $n$-th time, where $1\leq n\leq T$, and we let $S_{a}(n)=\sum_{s=1}^{n}X_{a,s}$.

    For fixed $a\in[K]$ and $n \in [T]$, Hoeffding's inequality gives
    \[
    \mathbb{P}\Biggl(
        \bigl|S_{a}(n)-n\mu_{a}\bigr|
        > \sqrt{\frac{n}{2}\log\frac{4KT}{\delta}}
    \Biggr)
    \leq 2\exp\Bigl(-\frac{2}{n}\frac{n}{2}\log\frac{4KT}{\delta}\Bigr)
    =\frac{\delta}{2KT}.
    \]

    By union bound over $a\in[K], t\in[T]$, we have with probability of at least $1-\frac{\delta}{2}$,
    \[\bigl|S_{a}(n)-n\mu_{a}\bigr| \leq \sqrt{\frac{n}{2}\log\frac{4KT}{\delta}} \, \forall a\in[K] \, \forall n\in[T].\]

    Note that $S_{a,t}=S_{a}(N_{a,t})$, where $N_{a,t}\in[T]$. Therefore, we have with probability of at least $1-\frac{\delta}{2}$,

    \[\bigl|S_{a,t}-N_{a,t}\mu_{a}\bigr| \leq \sqrt{\frac{N_{a,t}}{2}\log\frac{4KT}{\delta}} \, \forall a\in[K] \,\forall t\in[T].\]

    Note that
    \begin{align*}
    &\bigl|S_{a,t}-N_{a,t}\mu_{a}\bigr|
    \leq \sqrt{\frac{N_{a,t}}{2}\log\frac{4KT}{\delta}}\\
    &\implies\quad 
    \frac{\bigl|S_{a,t}-N_{a,t}\mu_{a}\bigr|}{N_{a,t}+1}
    \leq \frac{1}{N_{a,t}+1}\sqrt{\frac{N_{a,t}}{2}\log\frac{4KT}{\delta}}
    \leq \sqrt{\frac{\log\frac{4KT}{\delta}}{2(N_{a,t}+1)}},
    \end{align*}
    where the last inequality is by the fact that $\frac{\sqrt{n}}{n+1} \leq \frac{1}{\sqrt{n+1}}$ for $n\geq 1$.

    Note that $z_{a,u} \sim \text{Lap}\bigl(\frac{m_{T}}{\epsilon}\bigr)$, by Laplace tail formula,
    \[
    \mathbb{P}\Bigl(
        |z_{a,u}| > \frac{m_{T}}{\epsilon}\log\frac{2K(2L_{T}-1)}{\delta}
    \Bigr)
    =\exp\Bigl(-\log\frac{2K(2L_{T}-1)}{\delta}\Bigr)
    =\frac{\delta}{2K(2L_{T}-1)}.
    \]

    Note that for each tree $\mathcal{T}_{a}$, there are $2L_{T}-1$ nodes. By union bound over all tree nodes, we have with probability of at least $1-\frac{\delta}{2}$,
    \[|z_{a,u}| \leq \frac{m_{T}}{\epsilon}\log\frac{2K(2L_{T}-1)}{\delta} \, \forall \mathcal{T}_{a} \, \forall u\in \mathcal{T}_{a}.\]

    Note that the occurrence of this event implies
    \[
    \bigl|\tilde{S}_{a,t}-S_{a,t}\bigr|
    \leq \frac{m_{T}^{2}}{\epsilon}\log\frac{2K(2L_{T}-1)}{\delta}
    \implies
    \frac{\bigl|\tilde{S}_{a,t}-S_{a,t}\bigr|}{N_{a,t}+1}
    \leq \frac{\frac{m_{T}^{2}}{\epsilon}\log\Bigl(\frac{2K(2L_{T}-1)}{\delta}\Bigr)}{N_{a,t}+1},
    \]
    because to compute $\tilde{S}_{a,t}$ at most $m_{T}$ nodes from $\mathcal{T}_{a}$ are used, and thus we incur at most $m_{T}$ i.i.d Laplacian noises.

    Note that with probability of at least $1-\delta$, the intersection of the two events occurs, which implies
    \begin{align*}
    &\frac{\bigl|\tilde{S}_{a,t}-S_{a,t}\bigr|}{N_{a,t}+1}
    +\frac{\bigl|S_{a,t}-N_{a,t}\mu_{a}\bigr|}{N_{a,t}+1}
    +\frac{1}{N_{a,t}+1}\\
    &\leq \sqrt{\frac{\log\frac{4KT}{\delta}}{2(N_{a,t}+1)}}
    +\frac{\frac{m_{T}^{2}}{\epsilon}\log\Bigl(\frac{2K(2L_{T}-1)}{\delta}\Bigr)+1}{N_{a,t}+1}
    =\rho_{T}(N_{a,t},\delta).
    \end{align*}

    Further note that
    \[
    \bigl|\tilde{S}'_{a,t}-S_{a,t}\bigr|
    = \Bigl|\bigl[\tilde{S}_{a,t}\bigr]_{0}^{N_{a,t}}
    -\bigl[S_{a,t}\bigr]_{0}^{N_{a,t}}\Bigr|
    \leq \bigl|\tilde{S}_{a,t}-S_{a,t}\bigr|,
    \]
    because clipping is $1$-Lipschitz. Therefore
    \begin{align*}
    &\frac{\bigl|\tilde{S}'_{a,t}-S_{a,t}\bigr|}{N_{a,t}+1}
    +\frac{\bigl|S_{a,t}-N_{a,t}\mu_{a}\bigr|}{N_{a,t}+1}
    +\frac{1}{N_{a,t}+1}\\
    &\leq \sqrt{\frac{\log\frac{4KT}{\delta}}{2(N_{a,t}+1)}}
    +\frac{\frac{m_{T}^{2}}{\epsilon}\log\Bigl(\frac{2K(2L_{T}-1)}{\delta}\Bigr)+1}{N_{a,t}+1}
    =\rho_{T}(N_{a,t},\delta),
    \end{align*}

     which implies
     \[\bigl|\tilde{\mu}'_{a,t}-\mu_{a}\bigr| \leq \rho_{T}(N_{a,t},\delta).\]

     Note that
     \[m_{a,t}=\frac{1+\tilde{S}'_{a,t}}{2+N_{a,t}}=\frac{1+(N_{a,t}+1)\tilde{\mu}'_{a,t}}{2+N_{a,t}}.\]

     Therefore
     \[\bigl|m_{a,t}-\tilde{\mu}'_{a,t}\bigr|=\frac{1-\tilde{\mu}'_{a,t}}{2+N_{a,t}}<\frac{1}{N_{a,t}+2}.\]

     Hence, with probability of at least $1-\delta$,
     \[
     \bigl|m_{a,t}-\mu_{a}\bigr|
     \leq \bigl|m_{a,t}-\tilde{\mu}'_{a,t}\bigr|
     +\bigl|\tilde{\mu}'_{a,t}-\mu_{a}\bigr|
     \leq \rho_{T}(N_{a,t},\delta)+\frac{1}{N_{a,t}+2}.
     \]
\end{proof}

\begin{lemma}[Good Event 2]\label{lem:GE2}
    Define good event $E_{2,T}$ as
    \[
    \bigl\{\forall t\in[T], \hat{\theta}_{t}\in(0,1)^{K}, \theta'_{t}\in(0,1)^{K}\bigr\}
    \cap \bigl\{\hat{\theta}_{T+1}\in(0,1)^{K}\bigr\},
    \]
    then
    \[\mathbb{P}(E_{2,T})=1.\]
\end{lemma}

\begin{proof}[Proof of Lemma \ref{lem:GE2}]
    For any $a\in[A], t\in[T]$, by clipping, $\tilde{S}'_{a,t} \in [0, N_{a,t}]$, thus $1\leq1+\tilde{S}'_{a,t} \leq 1+N_{a,t}<2+N_{a,t}$, therefore,
    \[0<\frac{1}{2+N_{a,t}} \leq m_{a,t} <1.\]

    By construction of algorithm, $\theta'_{t}$ and $\hat{\theta}_{T+1}$ are both sampled from distributions supported on $(0,1)^{K}$.
\end{proof}

\begin{lemma}[Good Event 3]\label{lem:GE3}
    Define good event $E_{3,T}(\delta)$ as
    \[
    \forall a\in[K], \forall t\in[T],
    \Bigl|N_{a,t}-\sum_{s=1}^{t}\hat{\lambda}_{s,a}\Bigr|
    \leq v_{T}(t,\delta),
    \]
    where $\displaystyle v_{T}(t,\delta)=\sqrt{\frac{t}{2}\log\Bigl(\frac{2KT}{\delta}\Bigr)}$, then
    \[\mathbb{P}\bigl(E_{3,T}(\delta)\bigr) \geq 1-\delta.\]

    Furthermore, assume $\alpha < \frac{1}{2}$, on $E_{3,T}(\delta)$, define
    \[
    T_{3}(T,\delta)=
    \max\Biggl\{
        4,
        \Biggl\lceil
        \Bigl(8K^{2}(1-\alpha)^{2}\log\frac{2KT}{\delta}\Bigr)^{\frac{1}{1-2\alpha}}
        \Biggr\rceil
    \Biggr\},
    \]
    then for all $t\geq T_{3}(T,\delta)$ and all $a\in[K]$,
    \[N_{a,t} \geq \frac{t^{1-\alpha}-1}{2K(1-\alpha)}.\]
\end{lemma}

\begin{proof}[Proof of Lemma \ref{lem:GE3}]
    Consider an arbitrary $a\in[K]$, let $(\mathcal{F}_{t})_{t\geq0}$ be the filtration generated by all internal randomness up to round $t$, the actions $A_{1},\dots,A_{t}$ and the observed rewards up to round $t$.

    By construction of algorithm, $\hat{\lambda}_{t}$ is computed from the history up to round $t-1$, and thus $\mathcal{F}_{t-1}$-measurable. Therefore
    \[\mathbb{P}\bigl(A_{t}=a\bigm|\mathcal{F}_{t-1}\bigr)=\hat{\lambda}_{t,a},\]
    Define $X_{a,t}=\mathbf{1}\{A_{t}=a\}-\hat{\lambda}_{t,a}$, then
    $\mathbb{E}\bigl[X_{a,t}\bigm|\mathcal{F}_{t-1}\bigr]=0$. So if we define the partial sums $\{M_{a,t}\}$ by $M_{a,t}=\sum_{s=1}^{t}X_{a,s}=N_{a,t}-\sum_{s=1}^{t}\hat{\lambda}_{s,a}$, then it forms a martingale. Note that by Azuma-Hoeffding's inequality,
    \[
    \mathbb{P}\bigl(|M_{a,t}| \geq v_{T}(t,\delta)\bigr)
    \leq 2\exp\biggl(
        -\frac{2}{t}\frac{t}{2}\log\Bigl(\frac{2KT}{\delta}\Bigr)
    \biggr)
    =\frac{\delta}{KT}.
    \]

    By a union bound over $a\in[K]$ and $t\in[T]$, we have
    \[\mathbb{P}\bigl(E_{3,T}(\delta)\bigr) \geq 1-\delta.\]

    Furthermore, on $E_{3,T}(\delta)$, for all $a\in[K]$ and $t\in[T]$, we have
    \[N_{a,t} \geq \sum_{s=1}^{t} \hat{\lambda}_{s,a}-v_{T}(t,\delta).\]

    Note that $\hat{\lambda}_{s,a}=(1-\gamma_{s})\lambda_{s,a}+\gamma_{s}\lambda'_{a}\geq \gamma_{s}\lambda'_{a}=\frac{s^{-\alpha}}{K}$, therefore
    \[\sum_{s=1}^{t} \hat{\lambda}_{s,a}\geq\frac{1}{K}\sum_{s=1}^{t}s^{-\alpha}.\]

    By the standard integral lower bound $\sum_{s=1}^{t}s^{-\alpha} \geq \int_{1}^{t}x^{-\alpha}dx=\frac{t^{1-\alpha}-1}{1-\alpha}$, we have
    \[\sum_{s=1}^{t} \hat{\lambda}_{s,a} \geq \frac{t^{1-\alpha}-1}{K(1-\alpha)}.\]

    Therefore,
    \[N_{a,t} \geq \frac{t^{1-\alpha}-1}{K(1-\alpha)}-v_{T}(t,\delta).\]

    If $t \geq T_{3}(T,\delta)$,
    \[t \geq 4 \geq 2^{\frac{1}{1-\alpha}} \implies t^{1-\alpha}\geq 2 \implies t^{1-\alpha}-1 \geq \frac{1}{2}t^{1-\alpha},\]
    and in addition
    \[t^{1-2\alpha} \geq 8K^{2}(1-\alpha)^{2}\log\frac{2KT}{\delta}.\]

    Multiply both side by $t$, take square root and divide by $2K(1-\alpha)$, we have
    \[
    \frac{t^{1-\alpha}}{2K(1-\alpha)}
    \geq \sqrt{2t\log\frac{2KT}{\delta}}
    \implies
    \frac{t^{1-\alpha}-1}{K(1-\alpha)}
    \geq \sqrt{2t\log\frac{2KT}{\delta}}
    = 2v_{T}(t,\delta).
    \]

    Therefore
    \[N_{a,t} \geq \frac{t^{1-\alpha}-1}{K(1-\alpha)}-v_{T}(t,\delta) \geq \frac{t^{1-\alpha}-1}{2K(1-\alpha)}.\]
\end{proof}

\begin{lemma}[Good Event 4]\label{lem:GE4}
    Define good event $E_{4,T}(\delta)$ as
    \[\hat{a}_{t}=a_{\theta^{*}}, \forall t \in\bigl\{T_{4}(T,\delta)+1,\dots,T\bigr\},\]
    where we define
    \[
    T_{4}(T,\delta)=
    \max\left\{
        T_{3}(T,\delta),
        \Bigl\lceil
        \bigl(1+2K(1-\alpha)N_{4}(T,\delta)\bigr)^{\frac{1}{1-\alpha}}
        \Bigr\rceil
    \right\}
    \]
    and
    \[
    N_{4}(T,\delta)=
    \max\Biggl\{
        \Bigl\lceil \frac{72}{\Delta_{\min}^{2}}\log\frac{4KT}{\delta} \Bigr\rceil,
        \Biggl\lceil
        \frac{12\Bigl(\frac{m_{T}^{2}}{\epsilon}\log\bigl(\frac{2K(2L_{T}-1)}{\delta}\bigr)+1\Bigr)}{\Delta_{\min}}
        \Biggr\rceil,
        \Bigl\lceil \frac{12}{\Delta_{\min}}\Bigr\rceil
    \Biggr\},
    \]
    \[\Delta_{\min}=\mu_{a_{\theta^{*}}}-\max_{a\neq a_{\theta^{*}}}\mu_{a},\]

    then we have
    \[E_{1,T}(\delta) \cap E_{3,T}(\delta) \subseteq E_{4,T}(\delta).\]
\end{lemma}

\begin{proof}[Proof of Lemma \ref{lem:GE4}]
    First note that for any $a\in[K], t\in[T]$, if $N_{a,t} \geq N_{4}(T,\delta)$, then
    \[
    N_{a,t} \geq \frac{72}{\Delta_{\min}^{2}}\log\frac{4KT}{\delta}
    \implies
    \sqrt{\frac{\log\frac{4KT}{\delta}}{2(N_{a,t}+1)}}\leq \frac{\Delta_{\min}}{12},
    \]
    \[
    N_{a,t}
    \geq \frac{12\Bigl(\frac{m_{T}^{2}}{\epsilon}\log\bigl(\frac{2K(2L_{T}-1)}{\delta}\bigr)+1\Bigr)}{\Delta_{\min}}
    \implies
    \frac{\frac{m_{T}^{2}}{\epsilon}\log\bigl(\frac{2K(2L_{T}-1)}{\delta}\bigr)+1}{N_{a,t}+1}
    \leq \frac{\Delta_{\min}}{12},
    \]
    \[N_{a,t} \geq \frac{12}{\Delta_{\min}} \implies \frac{1}{N_{a,t}+2} \leq \frac{\Delta_{\min}}{12}.\]

    Summing them gives
    \[\rho_{T}(N_{a,t}, \delta) + \frac{1}{N_{a,t}+2} \leq \frac{\Delta_{\min}}{4}.\]

    Take any $t \geq T_{4}(T,\delta)$, since $t\geq T_{4}(T,\delta)\geq T_{3}(T, \delta)$, on $E_{3,T}(\delta)$ it gives for all $a\in[K]$
    \[N_{a,t} \geq \frac{t^{1-\alpha}-1}{2K(1-\alpha)}.\]

    In addition,
    \[
    t>T_{4}(T,\delta)
    \geq \bigl(1+2K(1-\alpha)N_{4}(T,\delta)\bigr)^{\frac{1}{1-\alpha}}
    \implies
    \frac{t^{1-\alpha}-1}{2K(1-\alpha)} \geq N_{4}(T,\delta)
    \implies N_{a,t} \geq N_{4}(T,\delta).
    \]

    Note that if we define $r_{T}(x,\delta)=\rho_{T}(x,\delta)+\frac{1}{x+2}$, then clearly it is a decreasing function in first argument, and therefore on $E_{1,T}(\delta)\cap E_{3,T}(\delta)$,
    \[
    \bigl|m_{a,t}-\mu_{a}\bigr|
    \leq r_{T}(N_{a,t},\delta)
    \leq r_{T}\bigl(N_{4}(T,\delta), \delta\bigr)
    \leq \frac{\Delta_{\min}}{4}.
    \]

    Therefore, after time $T_{4}(T,\delta)$, the estimator of any arm is within $\frac{\Delta_{\min}}{4}$ of the corresponding true arm mean.

    Again fix any round $t \in \bigl\{T_{4}(T,\delta)+1, \dots, T\bigr\}$, then $t-1 \geq T_{4}(T,\delta)$, hence for all arms $a\in[K]$,
    \[\bigl|m_{a,t-1}-\mu_{a}\bigr| \leq \frac{\Delta_{\min}}{4}.\]

    Apply to the true best arm $a_{\theta^{*}}$, we have
    \[m_{a_{\theta^{*}},t-1}\geq \mu_{a_{\theta^{*}}}-\frac{\Delta_{\min}}{4}.\]

    Apply to an arbitrary suboptimal arm $b \neq a_{\theta^{*}}$, then
    \[m_{b,t-1} \leq \mu_{b}+\frac{\Delta_{\min}}{4}.\]

    By the definition of $\Delta_{\min}$,
    \[\mu_{b} \leq \mu_{a_{\theta^{*}}}-\Delta_{\min}.\]

    Therefore,
    \[
    m_{b,t-1} \leq \mu_{b}+\frac{\Delta_{\min}}{4}
    \leq \mu_{a_{\theta^{*}}}-\Delta_{\min}+\frac{\Delta_{\min}}{4}
    =\mu_{a_{\theta^{*}}}-\frac{3}{4}\Delta_{\min},
    \]

    from which we have
    \[
    m_{a_{\theta^{*}},t-1}-m_{b,t-1}
    \geq \Bigl(\mu_{a_{\theta^{*}}}-\frac{\Delta_{\min}}{4}\Bigr)
    -\Bigl(\mu_{a_{\theta^{*}}}-\frac{3}{4}\Delta_{\min}\Bigr)
    =\frac{1}{2}\Delta_{\min}>0.
    \]

    Since this holds for every $b \neq a_{\theta^{*}}$, the maximizer $\hat{a}_{t}$ is unique and equal to $a_{\theta^{*}}$. Since $t > T_{4}(T,\delta)$ is arbitrary, this proves
    \[E_{1,T}(\delta) \cap E_{3,T}(\delta) \subseteq E_{4,T}(\delta).\]
\end{proof}

\begin{lemma}[Good Event 5]\label{lem:GE5}
    Define good event $E_{5,T}(\delta)$ as
    \[
    \sup_{\theta\in\Theta}
    \Biggl|
        \frac{1}{T}\log\frac{q_{\mathrm{Final}}(\theta)}{q_{\mathrm{Final}}(\mu)}
        +\sum_{a=1}^{K}\bar{e}_{a,T}d_{\epsilon}^{\mathrm{Ber}}(\theta_{a},\mu_{a})
    \Biggr|
    \leq 2C_{\mu}r_{T}\Bigl(\frac{T^{1-\alpha}-1}{2K(1-\alpha)},\delta\Bigr),
    \]
    where
    \[\bar{e}_{a,T}=\frac{N_{a,T}}{T},\]
    \[C_{\mu}=\frac{1}{3\xi_{\mu}(1-3\xi_{\mu})},\]
    \[\xi_{\mu}=\frac{1}{4}\min_{a\in[K]}\bigl\{\mu_{a},1-\mu_{a}\bigr\}.\]

    Furthermore, assume $T$ large enough such that
    \[\frac{T^{1-\alpha}-1}{2K(1-\alpha)} \geq N_{5}(T, \delta),\]

    where
    \[
    N_{5}(T, \delta) =
    \max\Biggl\{
        \Bigl\lceil \frac{9}{2\xi_{\mu}^{2}}\log\frac{4KT}{\delta} \Bigr\rceil,
        \Biggl\lceil
        \frac{3 \Bigl(\frac{m_{T}^{2}}{\epsilon}\log\bigl(\frac{2K(2L_{T}-1)}{\delta}\bigr)+1\Bigr)}{\xi_{\mu}}
        \Biggr\rceil,
        \Bigl\lceil \frac{3}{\xi_{\mu}}\Bigr\rceil
    \Biggr\},
    \]

    then we have
    \[E_{1,T}(\delta) \cap E_{3,T}(\delta)\subseteq E_{5,T}(\delta).\]
\end{lemma}

\begin{proof}[Proof of Lemma \ref{lem:GE5}]
    For any $\theta \in \Theta$,
    \[
    \frac{1}{T}\log\frac{q_{\mathrm{Final}}(\theta)}{q_{\mathrm{Final}}(\mu)}
    =-\sum_{a=1}^{K}\bar{e}_{a,T}
    \bigl(d_{\epsilon}^{\mathrm{Ber}}(\theta_{a},m_{a,T})
    -d_{\epsilon}^{\mathrm{Ber}}(\mu_{a},m_{a,T})\bigr).
    \]

    Therefore,
    \[
    \frac{1}{T}\log\frac{q_{\mathrm{Final}}(\theta)}{q_{\mathrm{Final}}(\mu)}
    +\sum_{a=1}^{K}\bar{e}_{a,T}d_{\epsilon}^{\mathrm{Ber}}(\theta_{a},\mu_{a})
    = \sum_{a=1}^{K}\bar{e}_{a,t}
    \bigl(d_{\epsilon}^{\mathrm{Ber}}(\theta_{a},\mu_{a})
    -d_{\epsilon}^{\mathrm{Ber}}(\theta_{a},m_{a,T})
    +d_{\epsilon}^{\mathrm{Ber}}(\mu_{a},m_{a,T})\bigr).
    \]

    On $E_{3,T}(\delta)$, we have $N_{a,T} \geq \frac{T^{1-\alpha}-1}{2K(1-\alpha)}$ for all $a\in[K]$, since $r_{T}(x,\delta)$ is decreasing in $x$, we have (on $E_{1,T}(\delta)\cap E_{3,T}(\delta)$) by $E_{4,T}(\delta)$,
    \[
    \bigl|m_{a,T}-\mu_{a}\bigr|
    \leq r_{T}(N_{a,T}, \delta)
    \leq r_{T}\Bigl(\frac{T^{1-\alpha}-1}{2K(1-\alpha)},\delta\Bigr).
    \]

    By the assumption $\frac{T^{1-\alpha}-1}{2K(1-\alpha)} \geq N_{5}(T, \delta)$, we have
    \[
    r_{T}\Bigl(\frac{T^{1-\alpha}-1}{2K(1-\alpha)},\delta\Bigr)
    \leq r_{T}\bigl(N_{5}(T,\delta), \delta\bigr).
    \]

    Note that
    \[
    N_{5}(T,\delta)+1 \geq N_{5}(T, \delta)
    \geq \frac{9}{2\xi_{\mu}^{2}}\log\frac{4KT}{\delta}
    \implies
    \sqrt{\frac{\log\frac{4KT}{\delta}}{2\bigl(N_{5}(T,\delta)+1\bigr)}}
    \leq \frac{\xi_{\mu}}{3},
    \]

    and
    \[
    N_{5}(T, \delta)+1
    \geq\frac{3 \Bigl(\frac{m_{T}^{2}}{\epsilon}\log\bigl(\frac{2K(2L_{T}-1)}{\delta}\bigr)+1\Bigr)}{\xi_{\mu}}
    \implies
    \frac{\frac{m_{T}^{2}}{\epsilon}\log\bigl(\frac{2K(2L_{T}-1)}{\delta}\bigr)+1}{N_{5}(T,\delta)+1}
    \leq \frac{\xi_{\mu}}{3},
    \]

    and
    \[N_{5}(T, \delta)+2 \geq \frac{3}{\xi_{\mu}} \implies \frac{1}{N_{5}(T, \delta)+2} \leq \frac{\xi_{\mu}}{3},\]

    which gives
    \[
    r_{T}\Bigl(\frac{T^{1-\alpha}-1}{2K(1-\alpha)},\delta\Bigr)
    \leq r_{T}\bigl(N_{5}(T,\delta), \delta\bigr)
    \leq \xi_{\mu}.
    \]

    By the definition of $\xi_{\mu}$, we have $4\xi_{\mu} \leq \mu_{a} \leq 1-4\xi_{\mu}$, which gives
    \[m_{a,T} \in [3\xi_{\mu}, 1-3\xi_{\mu}].\]

    Note that for $p\in[0,1]$ and $q\in(0,1)$,
    \[d_{\epsilon}^{\mathrm{Ber}}(p,q)=\inf_{u\in(0,1)}\bigl\{\epsilon|p-u|+\mathrm{kl}(u,q)\bigr\}.\]

    Therefore, let $p=\theta_{a}$, $q=\mu_{a}$ and $q'=m_{a,T}$, we have on $E_{1,T}(\delta) \cap E_{3,T}(\delta)$, by Theorem \ref{thm:Lipschitz-dber},
    \[
    \bigl|d_{\epsilon}^{\mathrm{Ber}}(\theta_{a},\mu_{a})
    -d_{\epsilon}^{\mathrm{Ber}}(\theta_{a},m_{a,T})\bigr|
    \leq C_{\mu}\bigl|m_{a,T}-\mu_{a}\bigr|
    \leq C_{\mu}r_{T}\Bigl(\frac{T^{1-\alpha}-1}{2K(1-\alpha)},\delta\Bigr),
    \]

    and since $d_{\epsilon}^{\mathrm{Ber}}(\mu_{a},\mu_{a})=0$, we have
    \[
    d_{\epsilon}^{\mathrm{Ber}}(\mu_{a},m_{a,T})
    =\bigl|d_{\epsilon}^{\mathrm{Ber}}(\mu_{a},m_{a,T})
    -d_{\epsilon}^{\mathrm{Ber}}(\mu_{a},\mu_{a})\bigr|
    \leq C_{\mu}\bigl|m_{a,T}-\mu_{a}\bigr|
    \leq C_{\mu}r_{T}\Bigl(\frac{T^{1-\alpha}-1}{2K(1-\alpha)},\delta\Bigr).
    \]

    Hence
    \begin{align*}
     &\Biggl|
     \frac{1}{T}\log\frac{q_{\mathrm{Final}}(\theta)}{q_{\mathrm{Final}}(\mu)}
     +\sum_{a=1}^{K}\bar{e}_{a,T}d_{\epsilon}^{\mathrm{Ber}}(\theta_{a},\mu_{a})
     \Biggr|\\
     =&\Biggl|
     \sum_{a=1}^{K}\bar{e}_{a,t}
     \bigl(d_{\epsilon}^{\mathrm{Ber}}(\theta_{a},\mu_{a})
     -d_{\epsilon}^{\mathrm{Ber}}(\theta_{a},m_{a,T})
     +d_{\epsilon}^{\mathrm{Ber}}(\mu_{a},m_{a,T})\bigr)
     \Biggr|\\
     \leq&\sum_{a=1}^{K}\bar{e}_{a,t}
     \Bigl(
     \bigl|d_{\epsilon}^{\mathrm{Ber}}(\theta_{a},\mu_{a})
     -d_{\epsilon}^{\mathrm{Ber}}(\theta_{a},m_{a,T})\bigr|
     +\bigl|d_{\epsilon}^{\mathrm{Ber}}(\mu_{a},m_{a,T})\bigr|
     \Bigr)\\
     \leq&\sum_{a=1}^{K}\bar{e}_{a,t}
     \biggl(2C_{\mu}r_{T}\Bigl(\frac{T^{1-\alpha}-1}{2K(1-\alpha)},\delta\Bigr)\biggr)\\
     =&2C_{\mu}r_{T}\Bigl(\frac{T^{1-\alpha}-1}{2K(1-\alpha)},\delta\Bigr).
    \end{align*}

    Since $\theta \in \Theta$ is arbitrary, taking supremum over $\theta$ gives the desired conclusion.
\end{proof}

\begin{corollary}\label{col:GE5}
    We define
    \[
    T_{5}(T,\delta)=
    \max\left\{
        T_{3}(T,\delta),
        \Bigl\lceil
        \bigl(1+2K(1-\alpha)N_{5}(T,\delta)\bigr)^{\frac{1}{1-\alpha}}
        \Bigr\rceil
    \right\},
    \]
    Then on $E_{1,T}(\delta)\cap E_{3,T}(\delta)$, for $t>T_{5}(T,\delta)$, we have
    \[m_{a,t} \in [3\xi_{\mu},1-3\xi_{\mu}].\]
\end{corollary}

\begin{proof}[Proof of Corollary \ref{col:GE5}]
    If $t > T_{5}(T, \delta) \geq T_{3}(T, \delta)$, then by $E_{3,T}(\delta)$,
    \[N_{a,t} \geq \frac{t^{1-\alpha}-1}{2K(1-\alpha)}.\]

    If $t > T_{5}(T,\delta) \geq \bigl\lceil \bigl(1+2K(1-\alpha)N_{5}(T,\delta)\bigr)^{\frac{1}{1-\alpha}} \bigr\rceil\}$, then
    \[\frac{t^{1-\alpha}-1}{2K(1-\alpha)} \geq N_{5}(T,\delta).\]

    Then, following an analogous argument on the proof of Lemma \ref{lem:GE5}, we have
    \[m_{a,t} \in [3\xi_{\mu},1-3\xi_{\mu}].\]
\end{proof}

\begin{lemma}[Good Event 6]\label{lem:GE6}
    Define good event $E_{6,T}$ as
    \[0 \leq \tau^{*}(\theta^{*})-G^{*}(\bar{e}_{T}) \leq \varepsilon_{6,T}(\delta),\]
    where
    \begin{align*}
    \varepsilon_{6,T}(\delta)
    &=2\epsilon \frac{T_{6}(T,\delta)-1}{T}
    + \frac{2C_\mu}{T} \sum_{t=t_0(T,\delta)}^T \Delta_t
    + R_T^{\max} + R_T^{\min}
    + \frac{2\epsilon}{T} \sum_{t=1}^T \gamma_t
    + \epsilon K \sqrt{\frac{\log(2KT/\delta)}{2T}},\\
    T_{6}(T,\delta)
    &=\max\bigl\{T_{4}(T,\delta), T_{5}(T, \delta)\bigr\}+1,\\
    \Delta_{t}
    &=\max_{a\in[K]}\bigl|m_{a,t-1}-\mu_{a}\bigr|,\\
    G^{*}(\lambda)
    &=\inf_{\theta\in \mathrm{Alt}(\theta^{*})}\sum_{a=1}^{K}\lambda_{a}d_{\epsilon}^{\mathrm{Ber}}(\theta_{a},\mu_{a}),\\
    R_{T}^{\max}
    &=\sup_{\lambda\in\Delta_{[K]}}\frac{1}{T}\sum_{t=1}^{T}\sum_{a=1}^{K}\lambda_{a}d_{\epsilon}^{\mathrm{Ber}}(\theta_{t,a}, m_{a,t-1})
    -\frac{1}{T}\sum_{t=1}^{T}\sum_{a=1}^{K}\lambda_{t,a}d_{\epsilon}^{\mathrm{Ber}}(\theta_{t,a},m_{a,t-1}),\\
    R_{T}^{\min}
    &=\frac{1}{T}\sum_{t=1}^{T}\sum_{a=1}^{K}\lambda_{t,a}d_{\epsilon}^{\mathrm{Ber}}(\theta_{t,a},m_{a,t-1})
    - \frac{1}{T}\sum_{t=1}^{T}\tilde{G}_{t}(\lambda_{t}),\\
    F_{t}(\theta)
    &=\sum_{a=1}^{K}\lambda_{t,a}d_{\epsilon}^{\mathrm{Ber}}(\theta_{a},m_{a,t-1}),\\
    \tilde{G}_{t}(\lambda_{t})
    &=\inf_{\theta\in\mathrm{Alt}(\hat{\theta}_{t})}F_{t}(\theta),
    \end{align*}
then we have
\[E_{1,T}(\delta) \cap E_{2,T}(\delta) \cap E_{3,T}(\delta) \cap E_{4,T}(\delta) \subseteq E_{6,T}(\delta).\]

\end{lemma}

\begin{proof}[Proof of Lemma \ref{lem:GE6}]
    We first define several terms as follows:
    \[H_{1,T}=\tau^{*}(\theta^{*})-\sup_{\lambda\in\Delta_{[K]}}\frac{1}{T}\sum_{t=1}^{T}G_{t}(\lambda),\]
    where $G_{t}(\lambda)=\inf_{\theta\in\mathrm{Alt}(\theta^{*})}\sum_{a=1}^{K}\lambda_{a}d_{\epsilon}^{\mathrm{Ber}}(\theta_{a},m_{a,t-1})$,
    \begin{align*}
     H_{2,T}&=\sup_{\lambda\in\Delta_{[K]}}\frac{1}{T}\sum_{t=1}^{T}G_{t}(\lambda)
     -\sup_{\lambda\in\Delta_{[K]}}\frac{1}{T}\sum_{t=1}^{T}\tilde{G}_{t}(\lambda),\\
     H_{3,T}&=\sup_{\lambda\in\Delta_{[K]}}\frac{1}{T}\sum_{t=1}^{T}\tilde{G}_{t}(\lambda)
     -\frac{1}{T}\sum_{t=1}^{T}\tilde{G}_{t}(\lambda_{t}),\\
     H_{4,T}&=\frac{1}{T}\sum_{t=1}^{T}\tilde{G}_{t}(\lambda_{t})
     -\frac{1}{T}\sum_{t=1}^{T}\tilde{G}_{t}(\hat{\lambda}_{t}),\\
     H_{5,T}&=\frac{1}{T}\sum_{t=1}^{T}\tilde{G}_{t}(\hat{\lambda}_{t})
     -G^{*}(\bar{\hat{\lambda}}_{T}),
    \end{align*}
    where $\bar{\hat{\lambda}}_{T}=\frac{1}{T}\sum_{t=1}^{T}\hat{\lambda}_{t}$ and
    \[H_{6,T}=G^{*}(\bar{\hat{\lambda}}_{T})-G^{*}(\bar{e}_{T}).\]

    By telescoping, we have $\tau^{*}(\theta^{*})-G^{*}(\bar{e}_{T})=\sum_{i=1}^{6}H_{i,T}.$

    We proceed to control each $H_{i,T}$.

($H_{1,T}$): For every $\lambda\in\Delta_{[K]}$ and every $t\geq T_{6}(T,\delta)$,
\begin{align*}
    &\bigl|G_{t}(\lambda)-G^{*}(\lambda)\bigr|\\
    &=\Biggl|
    \inf_{\theta\in\mathrm{Alt}(\theta^{*})}\sum_{a=1}^{K}\lambda_{a}d_{\epsilon}^{\mathrm{Ber}}(\theta_{a},m_{a,t-1})
    -\inf_{\theta\in \mathrm{Alt}(\theta^{*})}\sum_{a=1}^{K}\lambda_{a}d_{\epsilon}^{\mathrm{Ber}}(\theta_{a},\mu_{a})
    \Biggr|\\
    &\leq\sup_{\theta\in\mathrm{Alt}(\theta^{*})}
    \Biggl|
    \sum_{a=1}^{K}\lambda_{a}d_{\epsilon}^{\mathrm{Ber}}(\theta_{a},m_{a,t-1})
    -\sum_{a=1}^{K}\lambda_{a}d_{\epsilon}^{\mathrm{Ber}}(\theta_{a},\mu_{a})
    \Biggr|\\
    &\leq\sup_{\theta\in\mathrm{Alt}(\theta^{*})}\sum_{a=1}^{K}\lambda_{a}
    \bigl|d_{\epsilon}^{\mathrm{Ber}}(\theta_{a},m_{a,t-1})
    -d_{\epsilon}^{\mathrm{Ber}}(\theta_{a},\mu_{a})\bigr|\\
    &\leq\sup_{\theta\in\mathrm{Alt}(\theta^{*})}\sum_{a=1}^{K}\lambda_{a}C_{\mu}\bigl|m_{a,t-1}-\mu_{a}\bigr|\\
    &\leq\sup_{\theta\in\mathrm{Alt}(\theta^{*})}\sum_{a=1}^{K}\lambda_{a}C_{\mu}\Delta_{t}\\
    &=C_{\mu}\Delta_{t}.
\end{align*}

For $t<T_{6}(T,\delta)$, we use a trivial bound
\[\bigl|G_{t}(\lambda)-G^{*}(\lambda)\bigr| \leq \epsilon.\]

Therefore,
\begin{align*}
    H_{1,T}&=\tau^{*}(\theta^{*})-\sup_{\lambda\in\Delta_{[K]}}\frac{1}{T}\sum_{t=1}^{T}G_{t}(\lambda)\\
    &=\sup_{\lambda\in\Delta_{[K]}}G^{*}(\lambda)-\sup_{\lambda\in\Delta_{[K]}}\frac{1}{T}\sum_{t=1}^{T}G_{t}(\lambda)\\
    &\leq\Biggl|
    \sup_{\lambda\in\Delta_{[K]}}G^{*}(\lambda)
    -\sup_{\lambda\in\Delta_{[K]}}\frac{1}{T}\sum_{t=1}^{T}G_{t}(\lambda)
    \Biggr|\\
    &\leq \sup_{\lambda\in\Delta_{[K]}}
    \Biggl|G^{*}(\lambda)-\frac{1}{T}\sum_{t=1}^{T}G_{t}(\lambda)\Biggr|\\
    &=\sup_{\lambda\in\Delta_{[K]}}
    \Biggl|\frac{1}{T}\biggl[
    \sum_{t=1}^{T_{6}(T,\delta)-1}\bigl(G^{*}(\lambda)-G_{t}(\lambda)\bigr)
    +\sum_{t=T_{6}(T,\delta)}^{T}\bigl(G^{*}(\lambda)-G_{t}(\lambda)\bigr)
    \biggr]\Biggr|\\
    &\leq\sup_{\lambda\in\Delta_{[K]}}\frac{1}{T}
    \biggl[\epsilon\bigl(T_{6}(T,\delta)-1\bigr)
    +C_{\mu}\sum_{t=T_{6}(T,\delta)}^{T}\Delta_{t}\biggr]\\
    &=\epsilon\frac{T_{6}(T,\delta)-1}{T}
    +\frac{C_{\mu}}{T}\sum_{t=T_{6}(T,\delta)}^{T}\Delta_{t}.
\end{align*}

($H_{2,T}$) For $t\geq T_{6}(T,\delta)$, $E_{4,T}(\delta)$ implies
\[\hat{a}_{t}=a_{\theta^{*}},\]
which gives
\[\tilde{G}_{t}(\lambda)=G_{t}(\lambda) \, \forall \lambda \in \Delta_{[K]}.\]

For $t < T_{6}(T,\delta)$, again by a trivial bound
\[\bigl|G_{t}(\lambda)-\tilde{G}_{t}(\lambda)\bigr| \leq \epsilon,\]

which gives us
\begin{align*}
    H_{2,T}&=\sup_{\lambda\in\Delta_{[K]}}\frac{1}{T}\sum_{t=1}^{T}G_{t}(\lambda)
    -\sup_{\lambda\in\Delta_{[K]}}\frac{1}{T}\sum_{a=1}^{K}\tilde{G}_{t}(\lambda)\\
    &\leq\sup_{\lambda\in\Delta_{[K]}}\frac{1}{T}\sum_{t=1}^{T}
    \bigl|G_{t}(\lambda)-\tilde{G}_{t}(\lambda)\bigr|\\
    &=\sup_{\lambda\in\Delta_{[K]}}\frac{1}{T}\sum_{t=1}^{T_{6}(T,\delta)-1}\epsilon\\
    &=\epsilon\frac{T_{6}(T,\delta)-1}{T}.
\end{align*}

($H_{3,T}$) For $t\in[T]$ and $\lambda \in \Delta_{[K]}$, since $\theta_{t}\in\mathrm{Alt}(\hat{\theta}_{t})$
\[\tilde{G}_{t}(\lambda) \leq \sum_{a=1}^{K}\lambda_{a}d_{\epsilon}^{\mathrm{Ber}}(\theta_{t,a}, m_{a,t-1}).\]

Therefore,
\[
\sup_{\lambda\in\Delta_{[K]}}\frac{1}{T}\sum_{t=1}^{T}\tilde{G}_{t}(\lambda)
\leq \sup_{\lambda\in\Delta_{[K]}} \frac{1}{T}\sum_{t=1}^{T}\sum_{a=1}^{K}\lambda_{a}d_{\epsilon}^{\mathrm{Ber}}(\theta_{t,a}, m_{a,t-1}).
\]

Subtract $\frac{1}{T}\sum_{t=1}^{T}\tilde{G}_{t}(\lambda_{t})$ on both sides, we have
\begin{align*}
    H_{3,T}&=\sup_{\lambda\in\Delta_{[K]}}\frac{1}{T}\sum_{t=1}^{T}\tilde{G}_{t}(\lambda)
    -\frac{1}{T}\sum_{t=1}^{T}\tilde{G}_{t}(\lambda_{t})\\
    &\leq \sup_{\lambda\in\Delta_{[K]}} \frac{1}{T}\sum_{t=1}^{T}\sum_{a=1}^{K}\lambda_{a}d_{\epsilon}^{\mathrm{Ber}}(\theta_{t,a}, m_{a,t-1})
    -\frac{1}{T}\sum_{t=1}^{T}\tilde{G}_{t}(\lambda_{t})\\
    &=\Biggl(
    \sup_{\lambda\in\Delta_{[K]}}\frac{1}{T}\sum_{t=1}^{T}\sum_{a=1}^{K}\lambda_{a}d_{\epsilon}^{\mathrm{Ber}}(\theta_{t,a}, m_{a,t-1})
    -\frac{1}{T}\sum_{t=1}^{T}\sum_{a=1}^{K}\lambda_{t,a}d_{\epsilon}^{\mathrm{Ber}}(\theta_{t,a}, m_{a,t-1})
    \Biggr)\\
    &\quad + \Biggl(
    \frac{1}{T}\sum_{t=1}^{T}\sum_{a=1}^{K}\lambda_{t,a}d_{\epsilon}^{\mathrm{Ber}}(\theta_{t,a}, m_{a,t-1})
    - \frac{1}{T}\sum_{t=1}^{T}\tilde{G}_{t}(\lambda_{t})
    \Biggr)\\
    &\leq R_{T}^{\max} + R_{T}^{\min}.
\end{align*}

($H_{4,T}$)Since $\tilde{G_{t}}(\cdot)$ is the infimum of linear functionals with coefficients in $[0,\epsilon]$, it is $\epsilon$-Lipschitz in $l_{1}$:
\[
\bigl|\tilde{G_{t}}(\lambda_{1})-\tilde{G_{t}}(\lambda_{2})\bigr|
\leq \epsilon \bigl\|\lambda_{1}-\lambda_{2}\bigr\|_{1}.
\]

Therefore,
\[
\bigl\|\lambda_{t}-\hat{\lambda}_{t}\bigr\|_{1}
=\gamma_{t}\bigl\|\lambda_{t}-\lambda'\bigr\|_{1}
\leq 2\gamma_{t},
\]

which gives
\[
H_{4,T}=\frac{1}{T}\sum_{t=1}^{T}\tilde{G}_{t}(\lambda_{t})
-\frac{1}{T}\sum_{t=1}^{T}\tilde{G}_{t}(\hat{\lambda}_{t})
\leq \frac{2\epsilon}{T}\sum_{t=1}^{T}\gamma_{t}.
\]

($H_{5,T}$)For $t \geq T_{6}(T,\delta)$,
\[\tilde{G_{t}}(\hat{\lambda}_{t}) = G_{t}(\hat{\lambda_{t}}),\]
and
\[G_{t}(\hat{\lambda}_{t}) \leq G^{*}(\hat{\lambda}_{t})+C_{\mu}\Delta_{t}.\]
For $t<T_{6}(T,\delta)$, use the trivial bound
\[\tilde{G}_{t}(\hat{\lambda}_{t}) \leq \epsilon.\]
therefore
\begin{align*}
    H_{5,T}&=\frac{1}{T}\sum_{t=1}^{T}\tilde{G}_{t}(\hat{\lambda}_{t})-G^{*}(\bar{\hat{\lambda}}_{T})\\
    &\leq \epsilon\frac{T_{6}(T,\delta)-1}{T}
    +\frac{1}{T}\sum_{t=T_{6}(T,\delta)}^{T}G^{*}(\hat{\lambda}_{t})
    +\frac{C_{\mu}}{T}\sum_{t=T_{6}(T,\delta)}^{T}\Delta_{t}
    -G^{*}(\bar{\hat{\lambda}}_{T}).
\end{align*}

By concavity of $G^{*}(\cdot)$, we have
\[
\frac{1}{T}\sum_{t=T_{6}(T,\delta)}^{T}G^{*}(\hat{\lambda}_{t})
\leq \frac{1}{T}\sum_{t=1}^{T}G^{*}(\hat{\lambda}_{t})
\leq G^{*}\Biggl(\frac{1}{T}\sum_{t=1}^{t}\bar{\lambda_{t}}\Biggr)
=G^{*}(\bar{\hat{\lambda}}_{T}).
\]

Combining the inequalities, we have
\[H_{5,T} \leq \epsilon\frac{T_{6}(T,\delta)-1}{T}+\frac{C_{\mu}}{T}\sum_{t=T_{6}(T,\delta)}^{T}\Delta_{t}.\]

($H_{6,T}$)Since $G^{*}(\cdot)$ in $\epsilon$-Lipschitz,
\[
\bigl|G^{*}(\bar{\hat{\lambda}}_{T})-G^{*}(\bar{e}_{T})\bigr|
\leq \epsilon\bigl\|\bar{\hat{\lambda}}_{T}-\bar{e}_{T}\bigr\|_{1}.
\]

By $E_{3,T}(\delta)$, for any arm $a\in[K]$,
\[
\Bigl|N_{a,T}-\sum_{t=1}^{t}\hat{\lambda}_{t,a}\Bigr|
\leq v_{T}(T,\delta)=\sqrt{\frac{T}{2}\log\frac{2KT}{\delta}}.
\]

Divide by $T$ and sum over $a\in[K]$, we have
\[\bigl\|\bar{e}_{T}-\bar{\hat{\lambda}}_{T}\bigr\|_{1} \leq K\sqrt{\frac{1}{2T}\log\frac{2KT}{\delta}}.\]

Therefore,
\[H_{6,T}=G^{*}(\bar{\hat{\lambda}}_{T})-G^{*}(\bar{e}_{T}) \leq \epsilon K\sqrt{\frac{1}{2T}\log\frac{2KT}{\delta}}.\]

Summing $H_{i,T}$ over $i=1,2,3,4,5,6$, we have proven the desired inequality.
\end{proof}

\begin{lemma}[Control of terms appearing in Good Event 6]\label{lem:GE6control}
    We have with probability 1,
    \[\varepsilon_{6,T}(\delta_{T}) \in o(1),\]
    where $\delta_{T}=\frac{1}{T^{3}}$.
\end{lemma}

\begin{proof}[Proof of Lemma \ref{lem:GE6control}]
    In the analysis of $\varepsilon_{6,T}$, we fix $\alpha \in \bigl(0,\frac{1}{2}\bigr)$
    and we take $\beta_{t}$ to be a positive sequence such that $\beta_{t}\rightarrow\infty$.
    The desired results follow by Lemma \ref{sublem:GE6control1}, \ref{sublem:GE6control2}, \ref{sublem:GE6control3}, \ref{sublem:GE6control4}, \ref{sublem:GE6control5}, \ref{sublem:GE6control6}.
\end{proof}

\begin{sublemma}\label{sublem:GE6control1}
    We have
    \[R_{T}^{\max} \in o(1).\]
\end{sublemma}

\begin{proof}[Proof of Lemma \ref{sublem:GE6control1}]
    The AdaHedge update is run on normalized loss
    \[l_{t,a}=1-\frac{d_{\epsilon}^{\mathrm{Ber}}(\theta_{t,a},m_{a,t-1})}{\epsilon} \in [0,1].\]
    Let the standard AdaHedge regret on the loss sequence be
    \[R_{T}^{\mathrm{ah}}=\sum_{t=1}^{T}\sum_{a=1}^{K}\lambda_{t,a}l_{t,a}-\min_{i\in[K]}\sum_{t=1}^{t}l_{t,i}.\]
    Since the losses are in $[0,1]$, by Corollary 9 of \citep{rooij2014adahedge}, we have
    \[R_{T}^{\mathrm{ah}} \leq \sqrt{T\ln K}+\frac{4}{3}\ln K+2.\]

    Hence
    \[\frac{R_{T}^{\mathrm{ah}}}{T}\leq \sqrt{\frac{\ln K}{T}}+\frac{4 \ln K}{3T}+\frac{2}{T}\in o(1).\]

    Note that $d_{\epsilon}^{\mathrm{Ber}}(\theta_{t,a}, m_{a,t-1})=\epsilon(1-l_{t,a})$, we have
\begin{align*}
    R_{T}^{\max}
    &=\sup_{\lambda\in\Delta_{[K]}}\frac{1}{T}\sum_{t=1}^{T}\sum_{a=1}^{K}\lambda_{a}d_{\epsilon}^{\mathrm{Ber}}(\theta_{t,a},m_{a,t-1})
    -\frac{1}{T}\sum_{t=1}^{T}\sum_{a=1}^{K}\lambda_{t,a}d_{\epsilon}^{\mathrm{Ber}}(\theta_{t,a},m_{a,t-1})\\
    &=\epsilon\Biggl(
    \frac{1}{T}\sum_{t=1}^{T}\sum_{a=1}^{K}\lambda_{t,a}l_{t,a}
    -\inf_{\lambda \in \Delta_{[K]}}\frac{1}{T}\sum_{t=1}^{T}\sum_{a=1}^{K}\lambda_{a}l_{t,a}
    \Biggr).
\end{align*}
    Note that the above is a linear programming over a simplex, we attain infimum at an extreme point:
    \[\inf_{\lambda \in \Delta_{[K]}}\frac{1}{T}\sum_{t=1}^{T}\sum_{a=1}^{K}\lambda_{a}l_{t,a}=\min_{a\in[K]}\frac{1}{T}\sum_{t=1}^{T}l_{t,a}.\]

    Therefore,
    \[R_{T}^{\max}=\frac{\epsilon}{T}R_{T}^{\mathrm{ah}} \in o(1).\]
\end{proof}

\begin{sublemma}\label{sublem:GE6control2}
    We have with probability 1,
    \[R_T^{\min} \in o(1).\]
\end{sublemma}

\begin{proof}[Proof of Lemma \ref{sublem:GE6control2}]
    First note that $F_{t}(\cdot)$ is $\epsilon$-Lipschitz in $l_{1}$.
    We let $\overline{\mathrm{Alt}(\hat{\theta}_{t})}=\bigl\{\theta\in[0,1]^{K}: \exists b \neq \hat{a}_{t}, \theta_{b}\geq\theta_{\hat{a}_{t}}\bigr\}$ denote the closure of $\mathrm{Alt}(\hat{\theta}_{t})$. Because $\overline{\mathrm{Alt}(\hat{\theta}_{t})}$ is compact and $F_{t}$ is continuous, there exists a minimizer $\theta^{*}_{t}$ such that
    \[F_{t}(\theta^{*}_{t})=\tilde{G}_{t}(\lambda_{t})=\inf_{\theta\in\mathrm{Alt}(\hat{\theta}_{t})}F_{t}(\theta).\]

    Let $h=\hat{a}_{t}$, since $\theta^{*}_{t}\in\overline{\mathrm{Alt}(\hat{\theta}_{t})}$, there exists $b\in[K]$ such that $b\neq h$ and $\theta^{*}_{t,b} \geq \theta^{*}_{t,h}$.

    Fix $r\in\bigl(0,\frac{1}{2}\bigr]$. For each coordinate $j$, define
    $J_{j}=\Bigl[\max\bigl\{0,\theta^{*}_{t,j}-r\bigr\}, \min\bigl\{1, \theta^{*}_{t,j}+r\bigr\}\Bigr]$
    then each $J_{j}$ has length at least $r$.

    For $j\not\in\{b,h\}$, choose any subinterval $I_{j}\subset J_{j}$ of length $\frac{r}{2}$.

    For $j\in\{b,h\}$, we first write
    \[J_{b}=[p_{b},q_{b}], J_{h}=[p_{h},q_{h}]\]

    Since $\theta^{*}_{t,b} \geq \theta^{*}_{t,h}$, we have $p_{b} \geq p_{h}$ and $q_{b} \geq q_{h}$. We choose $I_{h}$ to be the lower quarter of $J_{h}$, so that $|I_{h}| \geq \frac{r}{4}$. We choose $I_{b}$ to be the upper quarter of $J_{b}$, so that $|I_{b}| \geq \frac{r}{4}$.

    Then, we have
    \[\max I_{h}=\frac{3p_{h}+q_{h}}{4}, \, \min I_{b}=\frac{p_{b}+3q_{b}}{4},\]
    which gives
    \[
    \min I_{b} -\max I_{h}
    = \frac{p_{b}+3q_{b}-(3p_{h}+q_{h})}{4}
    =\frac{(p_{b}-p_{h})+3(q_{b}-q_{h})+2(q_{h}-p_{h})}{4}
    \geq 0.
    \]

    Therefore every $\theta$ in the rectangle $B_{t}(r)=\prod_{j=1}^{K}I_{j}$ satisfies $\theta_{b} \geq \theta_{h}$, which implies
    \[B_{t}(r) \subset \overline{\mathrm{Alt}(\hat{\theta}_{t})}.\]

    Also, $B_{t}(r) \subset \Bigl\{\theta\in[0,1]^{K}: \bigl\|\theta-\theta^{*}_{t}\bigr\|_{\infty} \leq r\Bigr\}$, and its Lebesgue volume is lower bounded by
    \[
    \bigl|B_{t}(r)\bigr|
    \geq \Bigl(\frac{r}{4}\Bigr)^{2}\Bigl(\frac{r}{2}\Bigr)^{K-2}
    =\frac{r^{K}}{2^{K+2}}.
    \]

    Now define $\mu_{t}$ as the uniform probability measure on $\overline{\mathrm{Alt}(\hat{\theta}_{t})}$. Since $\overline{\mathrm{Alt}(\hat{\theta}_{t})} \subset [0,1]^{K}$, its volume is at most $1$, hence
    \[\mu_{t}\bigl(B_{t}(r)\bigr) \geq \bigl|B_{t}(r)\bigr|\geq c_{K}r^{K}, c_{K}=2^{-(K+2)}.\]

    For $\theta\in B_{t}(r)$, we have $\bigl\|\theta-\theta^{*}_{t}\bigr\|_{1} \leq Kr$, so by $\epsilon$-Lipschitz,
    \[F_{t}(\theta) \leq F_{t}(\theta^{*}_{t}) + \epsilon Kr = \tilde{G}_{t}(\lambda_{t})+\epsilon Kr.\]

    Thus the set $A_{t}(r)=\bigl\{\theta\in\overline{\mathrm{Alt}(\hat{\theta}_{t})}:F_{t}(\theta) \leq \tilde{G}_{t}(\lambda_{t})+\epsilon Kr\bigr\}$ satisfies
    \[\mu_{t}\bigl(A_{t}(r)\bigr) \geq \mu_{t}\bigl(B_{t}(r)\bigr) \geq c_{K}r^{K}.\]

    Let $Z_{t}$ be the normalization constant of $q_{t}$, therefore for $\theta_{t}$,
    \[
    q_{t}(\theta)=\frac{\exp\bigl(-\beta_{t}F_{t}(\theta)\bigr)}{Z_{t}}, \,
    Z_{t}=\int_{\overline{\mathrm{Alt}(\hat{\theta}_{t})}}
    \exp\bigl(-\beta_{t}F_{t}(\theta)\bigr)d\mu_{t}(\theta).
    \]

    Because $A_{t}(r) \subseteq \overline{\mathrm{Alt}(\hat{\theta}_{t})}$ and on $A_{t}(r)$,
    \[F_{t}(\theta) \leq \tilde{G}_{t}(\lambda_{t})+\epsilon Kr,\]
    we have
\begin{align*}
    Z_{t}&\geq\int_{A_{t}(r)}\exp\bigl(-\beta_{t}F_{t}(\theta)\bigr)d\mu_{t}(\theta)\\
    &\geq \int_{A_{t}(r)}
    \exp\Bigl(-\beta_{t}\bigl(\tilde{G}_{t}(\lambda_{t})+\epsilon Kr\bigr)\Bigr)d\mu_{t}(\theta)\\
    &=\mu_{t}\bigl(A_{t}(r)\bigr)
    \exp\Bigl(-\beta_{t}\bigl(\tilde{G}_{t}(\lambda_{t})+\epsilon Kr\bigr)\Bigr)\\
    &\geq c_{K}r^{K}
    \exp\Bigl(-\beta_{t}\bigl(\tilde{G}_{t}(\lambda_{t})+\epsilon Kr\bigr)\Bigr).
\end{align*}

Note that
\[\log q_{t}(\theta)=-\beta_{t}F_{t}(\theta)-\log Z_{t}.\]

Take expectation over $q_{t}$:
\[
\int q_{t}(\theta)\log q_{t}(\theta) d\mu_{t}(\theta)
=-\beta_{t}\mathbb{E}\bigl[F_{t}(\theta_{t})\bigm|\mathcal{F}_{t-1}\bigr]-\log Z_{t}.
\]

LHS is $KL\bigl(q_{t}\bigm\|\mu_{t}\bigr) \geq 0$, therefore
\[
\mathbb{E}\bigl[F_{t}(\theta_{t})\bigm|\mathcal{F}_{t-1}\bigr]
\leq -\frac{1}{\beta_{t}}\log Z_{t}
\leq \tilde{G}_{t}(\lambda_{t})+\epsilon Kr
+\frac{K\log \frac{1}{r}-\log c_{K}}{\beta_{t}}.
\]

Thus
\[
\mathbb{E}\bigl[F_{t}(\theta_{t})\bigm|\mathcal{F}_{t-1}\bigr]-\tilde{G}_{t}(\lambda_{t})
\leq\epsilon Kr+\frac{K\log \frac{1}{r}-\log c_{K}}{\beta_{t}}.
\]

By our specification that $\beta_{t} \rightarrow\infty$, we fix $\tilde{t}$ such that $\beta_{t} > 2$ for $t\geq \tilde{t}$. For $t\geq \tilde{t}$, choose $r=\frac{1}{\beta_{t}}$, then
\[
\mathbb{E}\bigl[F_{t}(\theta_{t})\bigm|\mathcal{F}_{t-1}\bigr]-\tilde{G}_{t}(\lambda_{t})
\leq\frac{\epsilon K}{\beta_{t}}+\frac{K\log \beta_{t}-\log c_{K}}{\beta_{t}}.
\]

Since $c_{K}=2^{-(K+2)}$, we have
\[-\log c_{K}=(K+2)\log 2,\]

thus,
\[
\mathbb{E}\bigl[F_{t}(\theta_{t})\bigm|\mathcal{F}_{t-1}\bigr]-\tilde{G}_{t}(\lambda_{t})
\leq\frac{\epsilon K+K\log \beta_{t}+(K+2)\log 2}{\beta_{t}},
\]

and hence,
\[
\mathbb{E}\bigl[F_{t}(\theta_{t})\bigm|\mathcal{F}_{t-1}\bigr]-\tilde{G}_{t}(\lambda_{t})
\in O\Bigl(\frac{\log \beta_{t}}{\beta_{t}}\Bigr).
\]

Now, define the martingale difference
\[Y_{t}=F_{t}(\theta_{t})-\mathbb{E}\bigl[F_{t}(\theta_{t})\bigm|\mathcal{F}_{t-1}\bigr],\]

since $F_{t}(\theta_{t}) \in [0,\epsilon]$, we have $|Y_{t}| \leq \epsilon$.

Note that
\[R_{T}^{\min}=\frac{1}{T}\sum_{t=1}^{T}\bigl(F_{t}(\theta_{t})-\tilde{G}_{t}(\lambda_{t})\bigr).\]

When $t<\tilde{t}$, we apply the trivial bound
\[F_{t}(\theta_{t})-\tilde{G}_{t}(\lambda_{t}) \in [0,\epsilon],\]

thus
\begin{align*}
    R_{T}^{\min}&=\frac{1}{T}\sum_{t=1}^{T}\bigl(F_{t}(\theta_{t})-\tilde{G}_{t}(\lambda_{t})\bigr)\\
    &\leq\frac{\tilde{t}\epsilon}{T}
    +\frac{1}{T}\sum_{t=\tilde{t}}^{T}\bigl(F_{t}(\theta_{t})-\tilde{G}_{t}(\lambda_{t})\bigr)\\
    &=\frac{\tilde{t}\epsilon}{T}+\frac{1}{T}\sum_{t=\tilde{t}}^{T}Y_{t}
    +\frac{1}{T}\sum_{t=\tilde{t}}^{T}
    \Bigl(\mathbb{E}\bigl[F_{t}(\theta_{t})\bigm|\mathcal{F}_{t-1}\bigr]
    -\tilde{G}_{t}(\lambda_{t})\Bigr).
\end{align*}

By Azuma-Hoeffding, for $b>0$,
\[
\mathbb{P}\Biggl(\Bigl|\sum_{t=\tilde{t}}^{T}Y_{t}\Bigr| >b T\Biggr)
\leq 2\exp\Bigl(-\frac{b^{2}T^{2}}{2\epsilon^{2}(T-\tilde{t}+1)}\Bigr)
\leq 2\exp\Bigl(-\frac{b^{2}T}{2\epsilon^{2}}\Bigr), \, \forall \,T\geq\tilde{t}.
\]

Therefore,
\[
\sum_{T=\tilde{t}}^{\infty}
\mathbb{P}\Biggl(\biggl|\frac{\sum_{t=\tilde{t}}^{T}Y_{t}}{T}\biggr| >b\Biggr)
<\infty.
\]

By Borel-Cantelli,
\[
\mathbb{P}\Biggl(
\biggl|\frac{\sum_{t=\tilde{t}}^{T}Y_{t}}{T}\biggr|
>b \text{ for infinitely many }T
\Biggr)=0.
\]

We consider the countable sequence $b_{n}=\frac{1}{n}, n\in\mathbb{N}$, therefore we have for $n\in\mathbb{N}$,
\[
\mathbb{P}\Biggl(
\biggl|\frac{\sum_{t=\tilde{t}}^{T}Y_{t}}{T}\biggr|
>\frac{1}{n} \text{ for infinitely many }T
\Biggr)=0.
\]

Taking complement of this event and taking intersection over $n$, we have
\[
\mathbb{P}\Biggl(
\bigcap_{n\in\mathbb{N}}
\biggl\{\Bigl|\frac{\sum_{t=\tilde{t}}^{T}Y_{t}}{T}\Bigr|
<\frac{1}{n} \text{ for sufficiently large }T\biggr\}
\Biggr)=1,
\]

which implies
\[\frac{\sum_{t=\tilde{t}}^{T}Y_{t}}{T} \rightarrow 0 \text{ a.s.}.\]

Furthermore, $\beta_{t}\rightarrow \infty$ is a sufficient condition for $\frac{\log\beta_{t}}{\beta_{t}} \rightarrow 0$. If we let
\[
\frac{1}{T}\sum_{t=\tilde{t}}^{T}
\Bigl(\mathbb{E}\bigl[F_{t}(\theta_{t})\bigm|\mathcal{F}_{t-1}\bigr]
-\tilde{G}_{t}(\lambda_{t})\Bigr)
\leq\frac{1}{T}\sum_{t=\tilde{t}}^{T} s_{t},
\]
where
\[
s_{t}=\frac{\epsilon K+K\log\beta_{t}+(K+2)\log2}{\beta_{t}}
\in O\Bigl(\frac{\log \beta_{t}}{\beta_{t}}\Bigr)\in o(1),
\]

then, by Ces\`aro Mean Theorem,
\[
\frac{1}{T}\sum_{t=\tilde{t}}^{T}
\Bigl(\mathbb{E}\bigl[F_{t}(\theta_{t})\bigm|\mathcal{F}_{t-1}\bigr]
-\tilde{G}_{t}(\lambda_{t})\Bigr)
\in o(1).
\]

Combining above, we have
\[R_{T}^{\min} \in o(1) \text{ a.s.}.\]
\end{proof}

\begin{sublemma}\label{sublem:GE6control3}
    We have
    \[\frac{1}{T}\sum_{t=T_{6}(T,\delta_{T})}^{T}\Delta_{t} \in o(1).\]
\end{sublemma}

\begin{proof}[Proof of Lemma \ref{sublem:GE6control3}]
    On $E_{1,T}(\delta_{T})$, we have
    \[\bigl|m_{a,t}-\mu_{a}\bigr| \leq \rho_{T}(N_{a,t}, \delta_{T})+\frac{1}{N_{a,t}+2}.\]

    On $E_{3,T}(\delta_{T})$, we have
    \[N_{a,t} \in \Omega\bigl(t^{1-\alpha}\bigr).\]

    Since $\delta_{T}=\frac{1}{T^{3}}$, we have
    \[\log\frac{4KT}{\delta_{T}} \in O(\log T),\]

    and since $m_{T}\in O(\log T)$ and $L_{T} \in O(T)$, we have
    \[
    \frac{m_{T}^{2}}{\epsilon}\log\Bigl(\frac{2K(2L_{T}-1)}{\delta_{T}}\Bigr)+1
    \in O\bigl((\log T)^{3}\bigr).
    \]

    Therefore, for all $a\in[K]$,
    \[
    \bigl|m_{a,t}-\mu_{a}\bigr|
    \lesssim \sqrt{\frac{\log T}{t^{1-\alpha}}}
    +\frac{(\log T)^{3}}{t^{1-\alpha}} + \frac{1}{t^{1-\alpha}}.
    \]

    Note that
    \[
    \frac{1}{T}\sum_{t=1}^{T}\sqrt{\frac{\log T}{t^{1-\alpha}}}
    = \sqrt{\log T}\frac{1}{T}\sum_{t=1}^{T}t^{-\frac{1-\alpha}{2}}
    \in O\Bigl(T^{-\frac{1-\alpha}{2}}\sqrt{\log T}\Bigr)
    \in o(1),
    \]

    \[
    (\log T)^{3}\frac{1}{T}\sum_{t=1}^{T}\frac{1}{t^{1-\alpha}}
    \in O\bigl((\log T)^{3}T^{-(1-\alpha)}\bigr)
    \in o(1),
    \]

    \[
    \frac{1}{T}\sum_{t=1}^{T}\frac{1}{t^{1-\alpha}}
    \in O\bigl(T^{-(1-\alpha)}\bigr)
    \in o(1).
    \]

    Combination of above proves the desired conclusion.
\end{proof}

\begin{sublemma}\label{sublem:GE6control4}
    We have
    \[\frac{1}{T}\sum_{t=1}^{T}\gamma_{t} \in o(1).\]
\end{sublemma}

\begin{proof}[Proof of Lemma \ref{sublem:GE6control4}]
\[
\frac{1}{T}\sum_{t=1}^{T}\gamma_{t}
=\frac{1}{T}\sum_{t=1}^{T}t^{-\alpha}
\leq \frac{1}{T}\Bigl(1+\int_{1}^{T}t^{-\alpha}dt\Bigr)
=\frac{1}{T}+\frac{T^{1-\alpha}-1}{T(1-\alpha)}
=\frac{1}{T}+\frac{1}{T^{\alpha}(1-\alpha)}-\frac{1}{T(1-\alpha)}
\in o(1).
\]
\end{proof}

\begin{sublemma}\label{sublem:GE6control5}
    We have
    \[\frac{T_{6}(T,\delta_{T})}{T} \in o(1).\]
\end{sublemma}

\begin{proof}[Proof of Lemma \ref{sublem:GE6control5}]
    By definition of $T_{6}(T,\delta)$, it is sufficient to prove both $T_{4}(T,\delta_{T})$ and $T_{5}(T,\delta_{T})$ are $o(T)$.

    Recall definition of $T_{4}(T,\delta)$:
    \[
    T_{4}(T,\delta)=
    \max\Biggl\{
        T_{3}(T,\delta),
        \Biggl\lceil
        \bigl(1+2K(1-\alpha)N_{4}(T,\delta)\bigr)^{\frac{1}{1-\alpha}}
        \Biggr\rceil
    \Biggr\}.
    \]

    Recall definition of $T_{3}(T,\delta)$:
    \[
    T_{3}(T,\delta)=
    \max\Biggl\{
        4,
        \Biggl\lceil
        \Bigl(8K^{2}(1-\alpha)^{2}\log\frac{2KT}{\delta}\Bigr)^{\frac{1}{1-2\alpha}}
        \Biggr\rceil
    \Biggr\}.
    \]
    $T_{3}(T,\delta_{T})$ is the maximum over a term in $O(1)$ and another term in $O\bigl((\log T)^{\frac{1}{1-2\alpha}}\bigr)$, which implies that
    \[T_{3}(T,\delta_{T}) \in O\Bigl((\log T)^{\frac{1}{1-2\alpha}}\Bigr).\]

    Recall definition of $N_{4}(T, \delta)$:
    \[
    N_{4}(T,\delta)=
    \max\Biggl\{
        \Bigl\lceil \frac{72}{\Delta_{\min}^{2}}\log\frac{4KT}{\delta} \Bigr\rceil,
        \Biggl\lceil
        \frac{12\Bigl(\frac{m_{T}^{2}}{\epsilon}\log\bigl(\frac{2K(2L_{T}-1)}{\delta}\bigr)+1 \Bigr)}{\Delta_{\min}}
        \Biggr\rceil,
        \Bigl\lceil \frac{12}{\Delta_{\min}}\Bigr\rceil
    \Biggr\}.
    \]

    $N_{4}(T,\delta_{T})$ is the maximum over a term in $O(\log T)$, a term in $O\bigl((\log T)^{3}\bigr)$, and another term in $O(1)$. Therefore,
    \[T_{4}(T,\delta_{T})\in \text{{polylog}}(T) \in o(T).\]

    On the other hand, recall the definition of $T_{5}(T,\delta)$:
    \[
    T_{5}(T,\delta)=
    \max\Biggl\{
        T_{3}(T,\delta),
        \Biggl\lceil
        \bigl(1+2K(1-\alpha)N_{5}(T,\delta)\bigr)^{\frac{1}{1-\alpha}}
        \Biggr\rceil
    \Biggr\},
    \]
    where
    \[
    N_{5}(T, \delta) =
    \max\Biggl\{
        \Bigl\lceil \frac{9}{2\xi_{\mu}^{2}}\log\frac{4KT}{\delta} \Bigr\rceil,
        \Biggl\lceil
        \frac{3  \biggl(\frac{m_{T}^{2}}{\epsilon}\log\Bigl(\frac{2K(2L_{T}-1)}{\delta}\Bigr)+1\biggr) }{\xi_{\mu}}
        \Biggr\rceil,
        \Bigl\lceil \frac{3}{\xi_{\mu}}\Bigr\rceil
    \Biggr\}.
    \]

    By an analogous argument, we have
    \[T_{5}(T,\delta) \in \text{polylog}(T) \in o(T).\]
\end{proof}

\begin{sublemma}\label{sublem:GE6control6}
    We have
    \[\epsilon K\sqrt{\frac{1}{2T}\log\frac{2KT}{\delta_{T}}} \in o(1).\]
\end{sublemma}

\begin{proof}[Proof of Lemma \ref{sublem:GE6control6}]
    This is straightforward after substitution of $\delta_{T}=\frac{1}{T^{3}}$.
\end{proof}

%% file: Appendix/discussion_sampling.tex
\section{Discussion of sampling oracle}

There are two stages of the algorithm that require sampling from a probability distribution. The min-learner samples from
\[
q_t(\theta)
\propto
\mathbf{1}_{\theta \in \mathrm{Alt}(\hat{\theta}_t)}
\exp\!\left(
-\beta_t \sum_{a=1}^{K} \lambda_{t,a}
d_{\epsilon}^{\mathrm{Ber}}(\theta_a,m_{a,t-1})
\right),
\]
while the final recommendation samples from
\[
q_{\mathrm{Final}}(\theta)
\propto
\exp\!\left(
-\sum_{a=1}^{K} N_{a,T}
d_{\epsilon}^{\mathrm{Ber}}(\theta_a,m_{a,T})
\right).
\]

To study the computational tractability of these sampling procedures, consider a general probability distribution $p$ supported on $(0,1)^K$ with density
\[
p(\theta)
\propto
\exp\!\left(
-\sum_{a=1}^{K} C_a
d_{\epsilon}^{\mathrm{Ber}}(\theta_a,s_a)
\right),
\]
where each $C_a>0$ and $s_a\in(0,1)$ is fixed. Both $q_t$ and $q_{\mathrm{Final}}$ are of this form, up to the additional truncation to $\mathrm{Alt}(\hat{\theta}_t)$ in the first stage. Hence, it suffices to analyze the efficiency of sampling from $p$.

Observe that
\begin{align*}
p(\theta)
&\propto
\exp\!\left(
-\sum_{a=1}^{K} C_a
d_{\epsilon}^{\mathrm{Ber}}(\theta_a,s_a)
\right) =
\prod_{a=1}^{K}
\exp\!\left(
-C_a d_{\epsilon}^{\mathrm{Ber}}(\theta_a,s_a)
\right),
\end{align*}
which shows that $p$ factorizes into $K$ one-dimensional distributions.

Define
\[
p_a(\theta_a)
=
\frac{
\exp\!\left(
-C_a d_{\epsilon}^{\mathrm{Ber}}(\theta_a,s_a)
\right)
}{Z_a},
\]
where $Z_a$ denotes the normalizing constant. Then
\[
p(\theta)
=
\prod_{a=1}^{K} p_a(\theta_a).
\]
Therefore, it is sufficient to analyze a single one-dimensional density $p_a$.

Using the closed-form expression of $d_{\epsilon}^{\mathrm{Ber}}$ derived in Appendix~\ref{apx:proof_lipschitz}, we may write
\[
p_a(\theta_a)
=
\frac{1}{Z_a} w_a(\theta_a),
\]
where
\[
w_a(\theta_a)
=
\begin{cases}
(1-s_a+s_a e^{-\epsilon})^{C_a}
e^{C_a\epsilon \theta_a},
& \theta_a \in (0,g_{\epsilon}^{+}(s_a)),
\\[0.5em]
e^{-C_a kl(\theta_a,s_a)},
& \theta_a \in [g_{\epsilon}^{+}(s_a),g_{\epsilon}^{-}(s_a)],
\\[0.5em]
(1-s_a+s_a e^{\epsilon})^{C_a}
e^{-C_a\epsilon \theta_a},
& \theta_a \in (g_{\epsilon}^{-}(s_a),1).
\end{cases}
\]

The normalizing constant is given by
\begin{align*}
Z_a
&=
\int_{0}^{g_{\epsilon}^{+}(s_a)}
(1-s_a+s_a e^{-\epsilon})^{C_a}
e^{C_a\epsilon \theta_a}
\, d\theta_a
\\
&\quad+
\int_{g_{\epsilon}^{+}(s_a)}^{g_{\epsilon}^{-}(s_a)}
e^{-C_a \kl(\theta_a,s_a)}
\, d\theta_a
\\
&\quad+
\int_{g_{\epsilon}^{-}(s_a)}^{1}
(1-s_a+s_a e^{\epsilon})^{C_a}
e^{-C_a\epsilon \theta_a}
\, d\theta_a.
\end{align*}

In general, the second integral cannot be expressed in elementary closed form, and hence neither $Z_a$ nor the density $p_a$ admits a closed-form expression. Nevertheless, the key observation is that $\log w_a$ is linear in $\theta_a$ on the first and third intervals, while on the middle interval,
\[
\log w_a(\theta_a)
=
-C_a \kl(\theta_a,s_a),
\]
which is concave in $\theta_a$. Consequently, $\log w_a$ is concave over the entire domain, implying that both $w_a$ and $p_a$ are log-concave.

This log-concavity property enables the use of efficient exact sampling methods such as Adaptive Rejection Sampling (ARS) \citep{arspy_docs,gilks1992ars,unuran_manual} and Transformed Density Rejection (TDR) \citep{unuran_manual}. These methods sample exactly from log-concave distributions without requiring knowledge of the normalizing constant, are readily available in standard Python packages, and have demonstrated strong empirical efficiency in practice.

%% file: Appendix/additional_experiments.tex
\section{Additional Experiments and Implementation Details}\label{apx:more_experiments}
We implement all algorithms in Python version $3.10.11$ on a $12$-core $64$-bit Apple M3 Pro CPU. The code for implementation and experiment execution is included as supplementary material.

We include three more instances in this section: The first instance has $20$ arms. The best arm has a mean reward of $0.85$. There are $2$ arms with a mean reward of $0.825$, and another $17$ arms with a mean reward of $0.25$. The second instance has $100$ arms. The best arm has a mean reward of $0.90$. There are $4$ arms with a mean reward of $0.875$. There are $49$ arms with a mean reward of $0.30$, and another $46$ arms with a mean reward of $0.20$. The experimental results of these two additional instances are reported in Figure \ref{fig:20arms} and \ref{fig:100arms} respectively. They have confirmed our experimental finding from Section \ref{sec:experiments}.

\begin{figure}[H]
    \centering
    \includegraphics[width=0.75\linewidth]{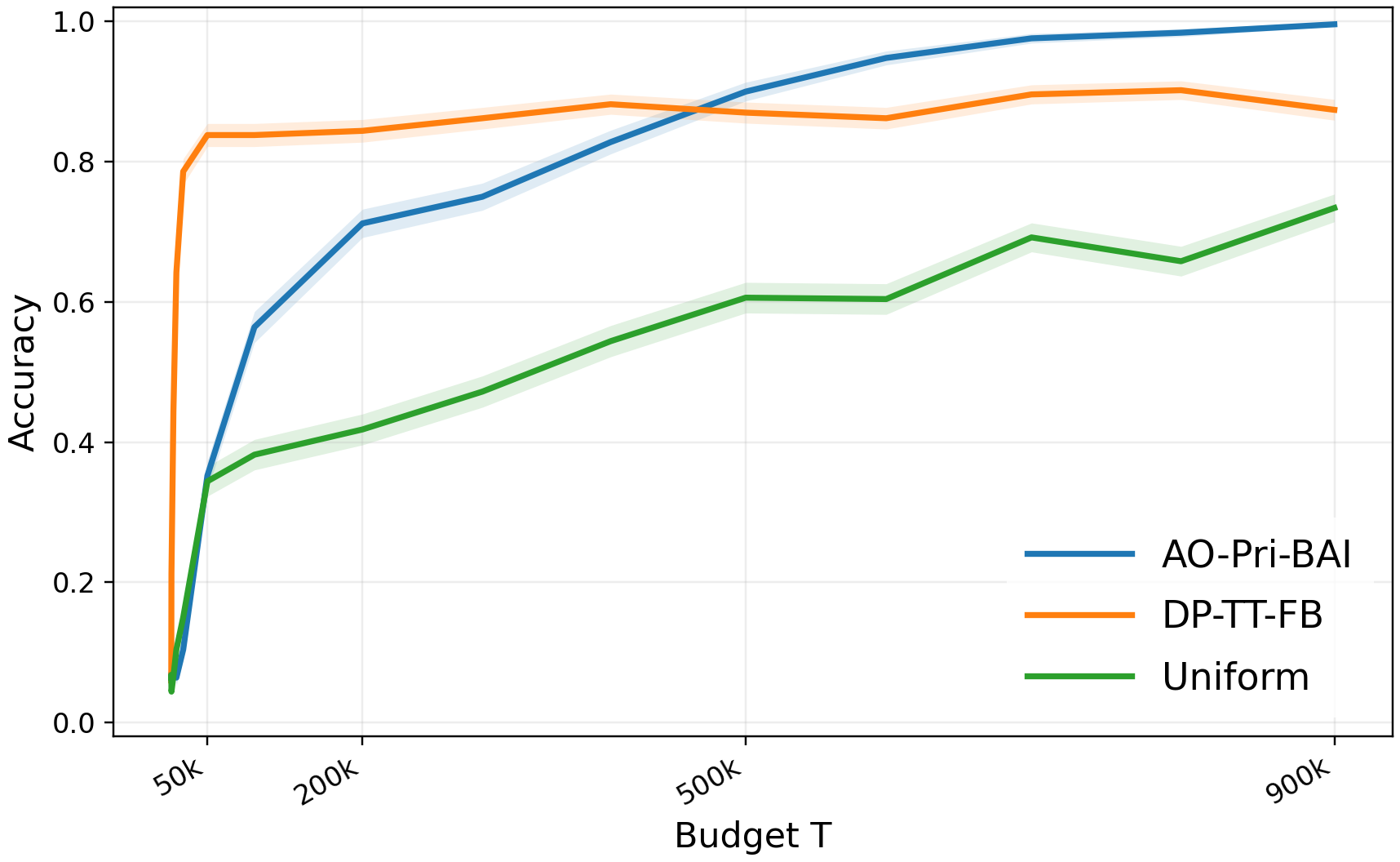}
    \caption{Comparison of \textsc{AO-Pri-BAI} with \textsc{Uniform} and \textsc{DP-TT-FB} on the 20-arm instance.}
    \label{fig:20arms}
\end{figure}

\begin{figure}[H]
    \centering
    \includegraphics[width=0.75\linewidth]{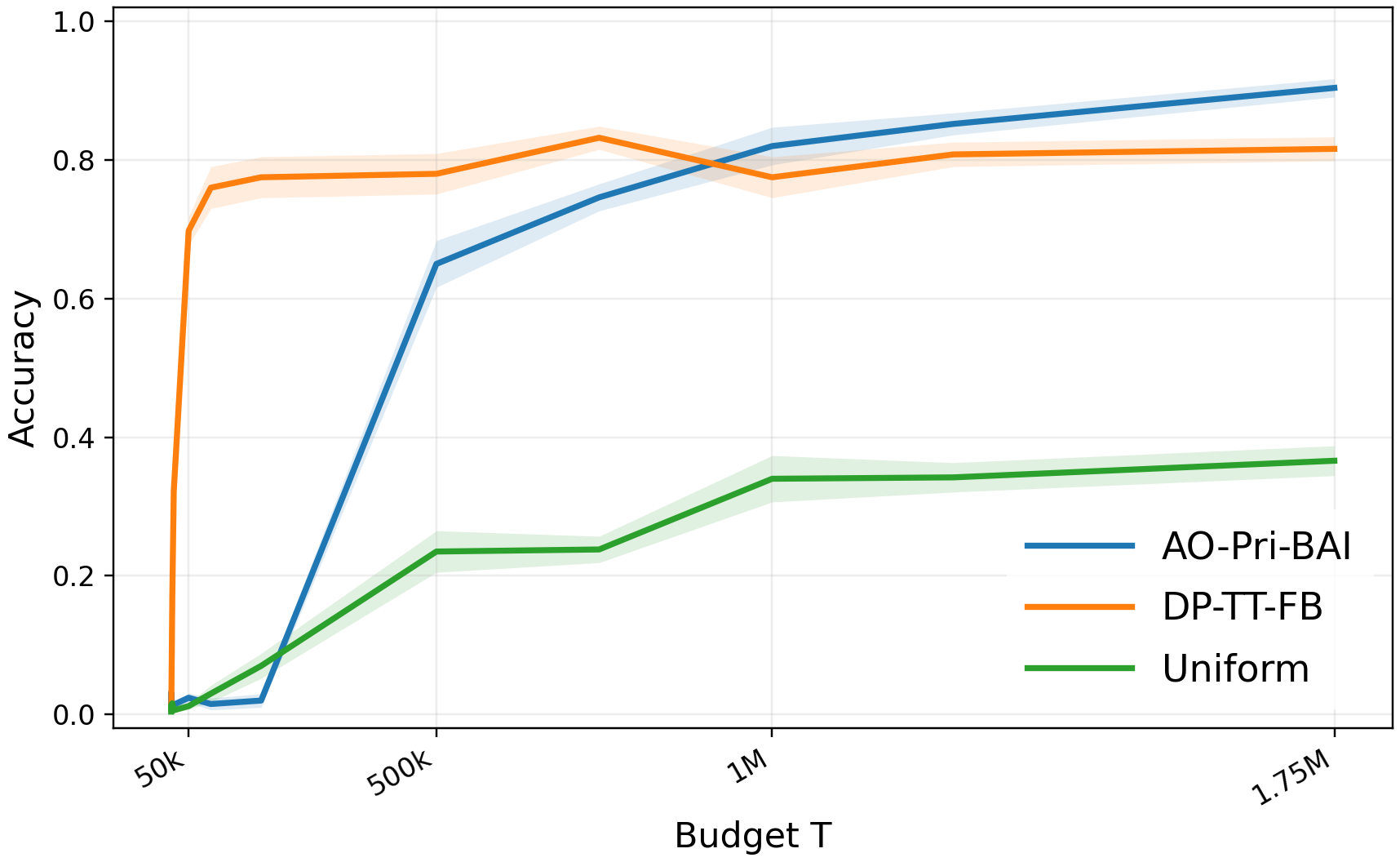}
    \caption{Comparison of \textsc{AO-Pri-BAI} with \textsc{Uniform} and \textsc{DP-TT-FB} on the 100-arm instance.}
    \label{fig:100arms}
\end{figure}

\begin{figure}[H]
    \centering
    \includegraphics[width=0.95\linewidth]{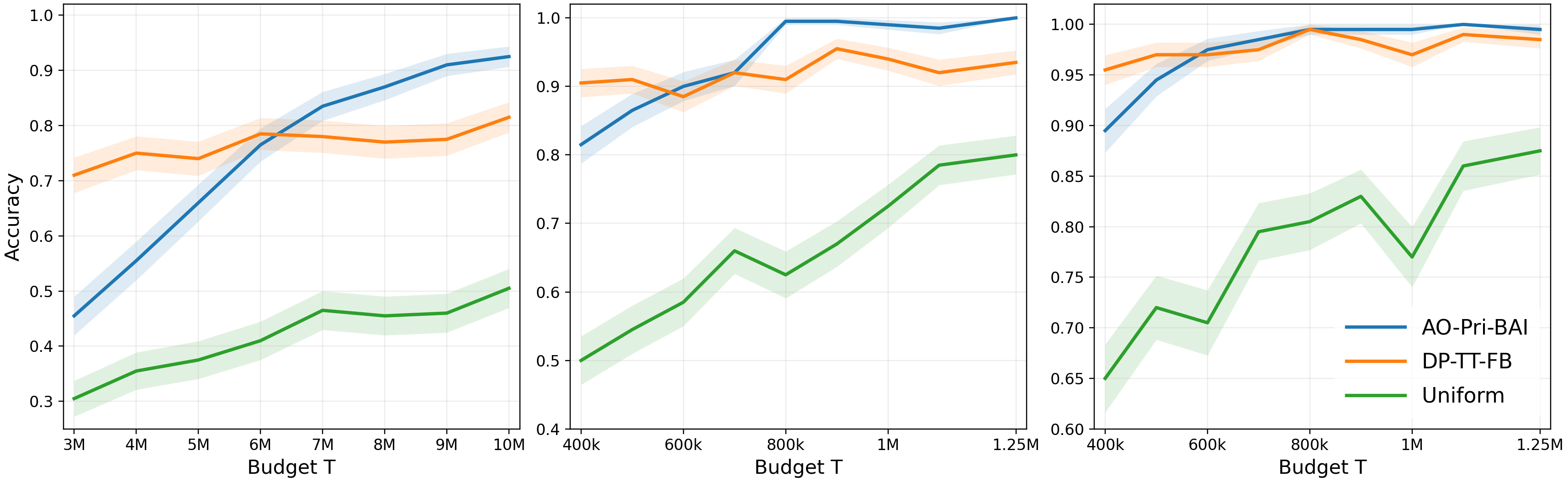}
    \caption{Comparison of \textsc{AO-Pri-BAI} with \textsc{Uniform} and \textsc{DP-TT-FB} on different values of $\epsilon$. From left to right: $\epsilon=0.05, \epsilon=1, \epsilon=5$.}
    \label{fig:vary}
\end{figure}

The third experiment is a series of experiments run on the same instance with varying $\epsilon$ values: The instance has $K=20$ arms. The unique best arm has arm mean $0.85$. There are 4 sub-optimal arms with arm mean $0.845$, and another $15$ sub-optimal arms with arm mean $0.25$. We test the same instance on $\epsilon=0.05, 1, 5$. The problem is getting easier with increasing values of $\epsilon$, because increasing $\epsilon$ means relaxing the privacy constraint. Figure \ref{fig:vary} shows the performance of all algorithms on that instance across different values of $\epsilon$. The plots show that, across various privacy levels, AO-Pri-BAI outperforms DP-TT-FB and Uniform in the long run (when $T$ is large), which is consistent with our asymptotic theoretical analyses. AO-Pri-BAI may be outperformed by its competitiors at the beginning when $T$ is small. This is justified by the fact that our theoretical guarantee is asymptotic, and the good events which allow us to control the error probability take a "burn-in period" (See the proof sketch in Section \ref{sec:upper_bound} for details). However, when $T$ gets larger, AO-Pri-BAI starts to show advantage over DP-TT-FB and Uniform across all $\epsilon$ values tested, empirically supporting our theory of asymptotic optimality.

%% file: Appendix/future_work.tex
\section{Future Work}\label{apx:future_work}

An interesting future work is to generalize our analysis beyond Bernoulli bandits. Here, we provide an insightful discussion that we could generalize to any distribution, possibly non-parametric, supported on $[0,1]$ and subsequently to any distribution with bounded support on $[a,b]$ by scaling our noise term by $(b-a)$. The practical obstacle is the computation of the privacy-aware transportation cost $d_{\epsilon}(\theta_{1},\theta_{2},\lambda)$. We focus on Bernoulli bandits for this work as it admits a closed form solution (Appendix \ref{apx:proof_lipschitz}), which is essential for keeping both the theoretical analysis and the implementation tractable. We lose this benefit beyond Bernoulli bandits. To overcome this obstacle, one possible approach is inspired by \citet{pmlr-v117-riou20a}: We approximate each reward distribution by a categorical distribution supported on $\{0,\frac{1}{M},\frac{2}{M},\dots,1\}$ using the randomized quantization technique provided in Lemma 6 of \citet{pmlr-v117-riou20a}. Specifically, if an observed reward satisfies $r\in[\frac{m}{M},\frac{m+1}{M})$, we replace it by the randomized quantized reward $\tilde{r}=\frac{m+B}{M}$, where $B\sim\mathrm{Ber}(Mr-m)$. By this construction, $\mathbb{E}[\tilde{r}|r]=r$, so the quantization preserves the conditional mean of every observation. The resulting discretized problem can be computed explicitly for both the total variation distance and the $\KL$ divergence between categorical distributions, allowing us to approximate $d_{\epsilon}(\theta_{1},\theta_{2},\lambda)$. As $M$ increases, the quantized model provides an increasingly accurate approximation to the original $[0,1]$-bounded reward model. Hence, it suggests replacing the exact computation of transportation cost in Algorithm \ref{algo:FB-DP-BAI} by its quantized approximation as a natural extension of \textsc{AO-Pri-BAI} to bounded nonparametric reward distributions. As for our lower bound analysis, it is not fundamentally tied to Bernoulli bandits. The proof of Theorem \ref{thm:LowerBound} is built upon a privacy-aware change-of-measure argument expressed directly in terms of total variation and KL divergence between reward distributions. Consequently, we expect it to extend, with only modest technical modifications, to arbitrary reward distributions supported on $[0,1]$.

Several others directions also remain open. First, the distributions used by \textsc{AO-Pri-BAI} for sampling instances are Gibbs distributions defined through the privacy-aware divergence, rather than the exact Bayesian posterior induced by the private observation process. Deriving a closed-form or tractable exact posterior for the composition of Bernoulli rewards, Laplace perturbations, clipping, and adaptive sampling appears challenging, and we leave a more fully Bayesian treatment for future work. Second, our current analysis focuses on the asymptotic error exponent. A sharper non-asymptotic finite-budget theory would be useful for understanding the practical thresholds at which the asymptotic behavior becomes visible. 

%% file: checklist.tex
\section*{NeurIPS Paper Checklist}

\begin{enumerate}

\item {\bf Claims}
    \item[] Question: Do the main claims made in the abstract and introduction accurately reflect the paper's contributions and scope?
    \item[] Answer: \answerYes{} 
    \item[] Justification: The claims made in the abstract and introduction accurately reflect the paper's contributions and scope. They are elaborate in detail in Sections \ref{sec:lower_bound}, \ref{sec:methodology}, \ref{sec:upper_bound}, \ref{sec:experiments}. For all theoretical results stated, they are proved formally in Appendix \ref{apx:proof_lipschitz}, \ref{apx:lowerbound}, \ref{apx:DP}, \ref{apx:upperbound}, \ref{apx:GE}.

\item {\bf Limitations}
    \item[] Question: Does the paper discuss the limitations of the work performed by the authors?
    \item[] Answer: \answerYes{} 
    \item[] Justification: The current limitations and directions that are worth further exploration are stated and discussed in Appendix \ref{apx:future_work}.

\item {\bf Theory assumptions and proofs}
    \item[] Question: For each theoretical result, does the paper provide the full set of assumptions and a complete (and correct) proof?
    \item[] Answer: \answerYes{} 
    \item[] Justification: Theorem \ref{thm:Lipschitz-dber} and \ref{thm:Lipschitz-dep} are proven in Appendix \ref{apx:proof_lipschitz}. Theorem \ref{thm:LowerBound} is proven in Appendix \ref{apx:lowerbound}. Theorem \ref{thm:DP} is proven in Appendix \ref{apx:DP}. Theorem \ref{thm:UpperBound} is proven in Appendix \ref{apx:upperbound}. The multiple Good Events that Theorem \ref{thm:UpperBound} relies on are all proven in Appendix \ref{apx:GE}.

    \item {\bf Experimental result reproducibility}
    \item[] Question: Does the paper fully disclose all the information needed to reproduce the main experimental results of the paper to the extent that it affects the main claims and/or conclusions of the paper (regardless of whether the code and data are provided or not)?
    \item[] Answer: \answerYes{} 
    \item[] Justification: \textsc{AO-Pri-BAI} has pseudocode stated in Algorithm \ref{algo:FB-DP-BAI}. The implementation details, the experiment settings and the choice of hyperparameters are all stated explicitly in Section \ref{sec:experiments} and Appendix \ref{apx:more_experiments}. We provide our implementation code as supplementary material. The information available is sufficient for one to reproduce the experiments.

\item {\bf Open access to data and code}
    \item[] Question: Does the paper provide open access to the data and code, with sufficient instructions to faithfully reproduce the main experimental results, as described in supplemental material?
    \item[] Answer: \answerYes{} 
    \item[] Justification: We provide our implementation code as supplementary material. The reproducibility of our experimental results is as discussed as above. We rely on synthetic experiments whose setting is explicitly described in Section \ref{sec:experiments} and Appendix \ref{apx:more_experiments}, so there is no dataset to be referenced. 

\item {\bf Experimental setting/details}
    \item[] Question: Does the paper specify all the training and test details (e.g., data splits, hyperparameters, how they were chosen, type of optimizer) necessary to understand the results?
    \item[] Answer: \answerYes{} 
    \item[] Justification: We provide our implementation code as supplementary material. We explain all experiment settings, including default choice of hyperparameters, in Section \ref{sec:experiments} and Appendix \ref{apx:more_experiments}.

\item {\bf Experiment statistical significance}
    \item[] Question: Does the paper report error bars suitably and correctly defined or other appropriate information about the statistical significance of the experiments?
    \item[] Answer: \answerYes{} 
    \item[] Justification: The figures of Section \ref{sec:experiments} and Appendix \ref{apx:more_experiments} plot empirical accuracy of 500 independent trials for each algorithm, with 1 standard error bars based on the standard deviations on empirical accuracy over those trials.

\item {\bf Experiments compute resources}
    \item[] Question: For each experiment, does the paper provide sufficient information on the computer resources (type of compute workers, memory, time of execution) needed to reproduce the experiments?
    \item[] Answer: \answerYes{} 
    \item[] Justification: The device used for all implementation and experiment is a personal laptop, described in Appendix \ref{apx:more_experiments}.
    
\item {\bf Code of ethics}
    \item[] Question: Does the research conducted in the paper conform, in every respect, with the NeurIPS Code of Ethics \url{https://neurips.cc/public/EthicsGuidelines}?
    \item[] Answer: \answerYes{} 
    \item[] Justification: We check and ensure that the research conform in every respect with the NeurIPS Code of Ethics.

\item {\bf Broader impacts}
    \item[] Question: Does the paper discuss both potential positive societal impacts and negative societal impacts of the work performed?
    \item[] Answer: \answerYes{} 
    \item[] Justification: Our work discusses $\epsilon$-differentially private fixed-budget best-arm identification. This work naturally fits into domains where data sensitivity is of paramount importance, such as clinical trials or user studies. Hence, our work can carry a significant positive societal impact.
    
\item {\bf Safeguards}
    \item[] Question: Does the paper describe safeguards that have been put in place for responsible release of data or models that have a high risk for misuse (e.g., pre-trained language models, image generators, or scraped datasets)?
    \item[] Answer: \answerNA{} 
    \item[] Justification: We only rely on synthetic data. To the best of our knowledge, we believe no such safeguard is necessary.

\item {\bf Licenses for existing assets}
    \item[] Question: Are the creators or original owners of assets (e.g., code, data, models), used in the paper, properly credited and are the license and terms of use explicitly mentioned and properly respected?
    \item[] Answer: \answerNA{} 
    \item[] Justification: Our work does not use existing assets. Our data is synthetic data and our implementation is done from scratch.

\item {\bf New assets}
    \item[] Question: Are new assets introduced in the paper well documented and is the documentation provided alongside the assets?
    \item[] Answer: \answerNA{} 
    \item[] Justification: Our work does not release new assets. The implementation code is provided as a supplementary material.

\item {\bf Crowdsourcing and research with human subjects}
    \item[] Question: For crowdsourcing experiments and research with human subjects, does the paper include the full text of instructions given to participants and screenshots, if applicable, as well as details about compensation (if any)? 
    \item[] Answer: \answerNA{} 
    \item[] Justification: Our work does not involve crowdsourcing nor research with human subjects.

\item {\bf Institutional review board (IRB) approvals or equivalent for research with human subjects}
    \item[] Question: Does the paper describe potential risks incurred by study participants, whether such risks were disclosed to the subjects, and whether Institutional Review Board (IRB) approvals (or an equivalent approval/review based on the requirements of your country or institution) were obtained?
    \item[] Answer: \answerNA{} 
    \item[] Justification: Our work does not involve crowdsourcing nor research with human subjects.

\item {\bf Declaration of LLM usage}
    \item[] Question: Does the paper describe the usage of LLMs if it is an important, original, or non-standard component of the core methods in this research? Note that if the LLM is used only for writing, editing, or formatting purposes and does \emph{not} impact the core methodology, scientific rigor, or originality of the research, declaration is not required.
    \item[] Answer: \answerYes{}{} 
    \item[] Justification: The authors used OpenAI Codex (GPT-6 Astra, GPT-5.6 Sol) to formulate and refine intermediate lemmas, check algebraic and probabilistic calculations, and improve the exposition. The authors independently verified all mathematic statements, proofs, and references in the final manuscript and take full responsibility for its contents.

\end{enumerate}